\documentclass[lettersize,journal]{IEEEtran}
\usepackage{array}
\usepackage[caption=false,font=normalsize,labelfont=sf,textfont=sf]{subfig}
\usepackage{textcomp}
\usepackage{stfloats}
\usepackage{verbatim}
\usepackage{graphicx}
\usepackage{amsmath,amssymb}
\usepackage{siunitx}
\usepackage{upgreek}
\usepackage[noend]{algpseudocode}
\usepackage{algorithm}
\usepackage{multirow}
\usepackage{url}
\usepackage{balance}

\usepackage{enumitem}
\usepackage{xcolor}
\usepackage{soul}
\usepackage[hidelinks]{hyperref}
\usepackage{caption}
\usepackage{subcaption}
\usepackage{tabularx,booktabs}
\usepackage{makecell}
\usepackage[colorinlistoftodos,prependcaption,textsize=small]{todonotes}
\usepackage{longtable}
\usepackage{lscape}
\usepackage{hvfloat}
\usepackage[table]{xcolor}
\usepackage{lipsum}

\usepackage[
    backend=biber,
    bibstyle=ieee,
    citestyle=ieee-comp,
    natbib=true,
    sorting=none
]{biblatex}
\usepackage{tabularx}
\usepackage{placeins}

\usepackage{acro}
\usepackage{enumitem}
\usepackage{amsmath, amssymb, amsfonts, amsthm}
\newtheorem{theorem}{Theorem}

\newtheorem{lemma}{Lemma}
\theoremstyle{definition}
\newtheorem{definition}{Definition}
\newtheorem{proposition}{Proposition}

\DeclareAcronym{cfg}{
  short       = CFG ,
  long        = Context-Free Grammar ,
  short-plural = s ,
  long-plural  = s
}

\DeclareAcronym{nn}{
  short       = NN ,
  long        = Neural Network ,
  short-plural = s ,
  long-plural  = s
} 

\DeclareAcronym{ode}{
  short       = ODE ,
  long        = ordinary differential equation ,
  short-plural = s ,
  long-plural  = s
}

\DeclareAcronym{mse}{
  short       = MSE ,
  long        = mean square error ,
  short-plural = s ,
  long-plural  = s
}

\DeclareAcronym{sr}{
  short       = SR ,
  long        = symbolic regression ,
  short-plural =  ,
  long-plural  = 
}

\begin{document}
\bstctlcite{IEEEexample:BSTcontrol}

\title{Exhaustive Canonical Skeletons Generation:\\ A Formal Dataset for Decomposable Symbolic Regression}

\author{\IEEEauthorblockN{Giorgio Morales and Anikó Ekárt} \\
\IEEEauthorblockA{Aston Centre for Artificial Intelligence Research and Application, Aston University, Birmingham, UK}\vspace{-3.2ex}

\thanks{Manuscript received Oct 1; revised Oct 30; accepted Nov 8, 2026.}}

\markboth{IEEE Computational Intelligence Letters}%
{Morales \MakeLowercase{\textit{et al.}}: A Formal Dataset for Decomposable Symbolic Regression}


\maketitle

\begin{abstract}
Neural symbolic regression (SR) models require diverse datasets of mathematical expressions for robust training.
However, existing datasets rely predominantly on stochastic sampling methods, which suffer from high structural redundancy and incomplete functional coverage. 
To address this, we introduce an exhaustive dataset of univariate canonical skeletons generated via a deterministic constrained context-free grammar. 
A canonical skeleton acts as a generic structural representation that subsumes more specific functional forms, eliminating redundant expressions and avoiding repeated computational efforts. 
Although our dataset consists of univariate skeletons, it is designed to support the decomposable SR paradigm to solve multivariate SR problems without assuming linear separability. 
Furthermore, we pair this dataset with a sampling strategy that generates clean numerical data without triggering numerical instabilities or relying on artificial protected operators. 
Canonicalizing two existing neural SR datasets reveals that our dataset is substantially more compact, eliminates structural repetitions, and achieves a uniform operator distribution free of simplicity bias. 
Finally, we establish a benchmark of 150 SR problems across three tiers, 
on which state-of-the-art SR methods are shown to struggle significantly.\footnote{Dataset and code available at \href{https://github.com/ACAIRA-Aston/SRSkeletons}{github.com/ACAIRA-Aston/SRSkeletons}}
\end{abstract}


\vspace{-1ex}

\section{Introduction}

\Ac{sr} aims to uncover explicit, human-interpretable mathematical expressions that describe underlying patterns in numerical data. 
In recent years, neural \ac{sr} models (i.e., deep generative models and Transformers)~\cite{SRthatscales,SymbolicGPT,SRtransformer,SeTGAP} have made significant advances by framing equation discovery as a sequence-to-sequence translation task. 
However, the generalization capability of neural SR models is tied to the quality, diversity, and structural validity of their training datasets. 

To construct these datasets, existing approaches rely on stochastic tree sampling, templates, or heuristics~\cite{Lample2020Deep,SRthatscales,SRtransformer,SymbolicGPT,MultiSetSR,Yu_2025} (see Supplementary Materials (\textbf{S.M.})~\ref{app:related} for a review). 
While these methods allow the generation of large datasets, they remain probabilistic explorations of an exponentially growing search space. 
Thus, they suffer from severe structural redundancy:
distinct derivation paths frequently collapse into functionally identical expressions after simplification (e.g., $x \cdot 1$ vs. $x+0$), over-exposing models to a narrow subset of simple forms. 
Even when constrained by templates~\cite{SymbolicGPT}, sampling weights~\cite{MultiSetSR}, or probabilistic grammars~\cite{Yu_2025}, these pipelines rely on arbitrary heuristics and fail to guarantee exhaustive and non-redundant coverage of valid functional structures.

To overcome these limitations, our \textbf{first contribution} is an Exhaustive dataset of Canonical Univariate Skeletons (ExCaUS), coupled with a principled sampling strategy. 
A symbolic skeleton replaces specific numeric parameters with constant placeholders. 
We extend this concept by defining a \textit{canonical skeleton} as a generic structural representation that subsumes more specific instantiated representations. 
By formalizing explicit closure rules, canonical skeletons eliminate functional collisions and prevent repeated generation efforts.

While our dataset consists of univariate skeletons, its utility is not restricted to univariate SR problems. 
Instead, it is designed to serve as the foundation for training SR models under a decomposable SR paradigm~\cite{SeTGAP}. 
As detailed in Sec.~\ref{sec:datasetdef}, complex multivariate systems can be solved by first identifying the isolated functional relationship between individual variables and the system's response in the form of univariate skeletons. 
These skeletons are then merged into multivariate expressions through progressive structural composition.

To construct our dataset, we establish a formal recursive structural decomposition mapped directly to a \ac{cfg}. 
Unlike stochastic generation methods, our grammar-based approach provides a deterministic framework for exhaustive enumeration. 
Furthermore, we pair this dataset with a robust data sampling strategy designed to generate clean numerical data without triggering numerical instabilities or relying on protected operators. 
To validate our approach, we compared our canonical dataset against two existing datasets used by previous neural \ac{sr} methods, which we canonicalized to ensure a fair evaluation. 
The comparison reveals that the baseline approaches present severe skewness and redundancy, while our dataset is substantially more compact and entirely free of structural duplicates. 
Additionally, it displays a more uniform distribution of unary operators, successfully avoiding the common bias toward overly simplistic functional forms.

Our \textbf{second contribution} is a rigorous \ac{sr} benchmark, ExCaUS-Bench.
We generated 150 problems, split into three tiers, to evaluate a suite of \ac{sr} methods. 
Exhaustive tuning and testing show that most \ac{sr} methods struggle significantly to recover the underlying functions in these problems. 
Unlike standard benchmarks that evaluate models over narrow intervals, our benchmark utilizes extended domain ranges, substantially increasing the difficulty of the regression.

\section{Dataset Definition}  \label{sec:datasetdef}
 
In this section, we define the principles for constructing the dataset of symbolic skeletons.
Let $\mathcal{U}$ and $\mathcal{B}$ denote the finite set of unary (e.g., $\texttt{sin}$ and $\texttt{log}$) and binary operators (e.g., $+, -, \times, \div$), respectively.
Thus, a mathematical expression $f(\mathbf{x})$ is represented as a rooted tree with nodes $\in \{\mathcal{U}, \mathcal{B}\}$, where $\mathbf{x} \!=\! \{x_1, \dots, x_n\}$ is a multivariate input vector.

A symbolic skeleton expression is an abstract representation of a function $f(\mathbf{x})$ that captures the functional form between a variable, or set of variables, and the system's response using constant placeholders $c_i$ that may represent numeric constants or functions of other variables.
It is obtained using the skeleton function $\kappa(\cdot)$.
In general, we employ it as $\kappa(f(\mathbf{x}); \mathbf{x}_s)$, where $\mathbf{x}_s$ is a subset of variables and $S \subseteq \{1, \dots, n\}$ is an index set. 
Thus, $\mathbf{x}_s$ represents the set of variables with respect to which the skeleton is expressed. 
In particular, $\kappa(f(\mathbf{x}); \mathbf{x}) \!=\! \kappa(f(\mathbf{x}))$.
For example, $f(\mathbf{x}) = 2 \sin( \frac{x_1}{10 x_2} + \frac{\pi}{2})^2 + {10}$; then $\kappa(f(\mathbf{x}), x_1) = c_1 \sin(c_2 x_1 + c_3)^2 + c_4$, $\kappa(f(\mathbf{x}), x_2) = c_5 \sin(\frac{c_6}{x_2} + c_7)^2 + c_8$, and $\kappa(f(\mathbf{x})) = c_{9} \sin(c_{10} \frac{x_1}{x_2} + c_{11})^2 + c_{12}$.

From now on, we focus on univariate symbolic skeletons.
\textbf{Although our dataset consists of univariate skeletons, its intended use is not restricted to univariate SR, nor does it assume linear separability.} 
Following the decomposable SR paradigm of \citet{SeTGAP}, univariate skeletons are viewed as structural components that capture the functional relationship between a variable and the system response. 

Suppose that $f(x) \!=\! \frac{x_1}{x_2} \tanh(\sqrt{2x_1^2 + x_2^2})$ and that the individual variable-response relationship analysis yields 
$\mathbf{e}(x_1) = c_1 x_1 \tanh(c_2 |x_1| + c_3)$ and $\mathbf{e}(x_2)= \frac{c_4}{x_2} \tanh(c_5 |x_2| + c_6)$.
These skeletons do not describe the complete system independently; rather, they identify structural patterns that are consistent with the underlying system. 
The arguments $c_2|x_1|+c_3$ and $c_5|x_2|+c_6$ can be viewed as special cases of $\sqrt{c_7x_1^2+c_8}$ and $\sqrt{c_9x_2^2+c_{10}}$, respectively. 
The subsequent task is thus to discover a multivariate composition that preserves the identified univariate structures while jointly fitting the entire dataset. 
As shown in \cite{SeTGAP}, this can be achieved through a progressive merging process in which univariate skeletons are combined into increasingly expressive multivariate skeletons until a suitable representation of the whole system is obtained.

Let $\mathcal{F}$ be the set of all admissible univariate expressions $f(x)$ constructed from $\mathcal{U}$ and $\mathcal{B}$.
Then, $\mathcal{S}_{skel}$ denotes the set of all skeleton expressions (i.e., expression trees whose leaves are ``$x$'' or ``$c_i$'') s.t. $\mathcal{S}_{skel} = \{ \kappa(f(x)) : f \in \mathcal{F} \}$.
An instantiation of a skeleton $s \in \mathcal{S}_{skel}$ is defined by a vector $\boldsymbol{\alpha} = [\alpha_1, \dots, \alpha_k] \in \mathbb{R}^k$, where $k$ is the number of placeholders in $s$. 
The instantiated expression is denoted $s[\boldsymbol{\alpha}]$.

\begin{definition}[Subsumption]
For $s_1, s_2 \in \mathcal{S}_{skel}$, we say $s_1$ subsumes $s_2$, $s_1 \succeq s_2$, if, for any instantiation $\boldsymbol{\beta}$, $\exists \boldsymbol{\alpha}$ s.t. $s_1[\boldsymbol{\alpha}] \equiv s_2[\boldsymbol{\beta}]$. Here, $\equiv$ represents equivalence between identically equal functions on an open interval containing 0.
\end{definition}
For instance, consider the skeletons $s_1 \!=\! c_1 x^2 \!+\! c_2 x \!+\! c_3$ and $s_2 \!=\! c_4 x^2$. 
Then, $s_1 \succeq s_2$, as $s_1\left[[\beta, 0, 0]\right] \equiv s_2\left[[\beta]\right]$.
For the sake of brevity, from now on, skeletons may be expressed using unnumbered placeholders $c$ (e.g., $s_1 \!=\! c x^2 \!+\! c x \!+\! c$).
However, each $c$ is considered as an independent placeholder. 

\begin{definition}[Canonical atoms]
We define the canonical atoms $C \!:=\! c$ (constant), $L \!:=\! cx+c$ (linear), 
and $P \!:=\! P^m(x)$ (polynomial). 
Here, $m$ is a placeholder index, and every polynomial of finite degree belongs to the family $P$.
\end{definition}

\begin{definition}[Canonical equivalence]
\label{def:equiv}
We define the following equivalence and closure rules between canonical atoms:
\begin{itemize}[leftmargin=1em, noitemsep,topsep=0pt]
    \item Constant: $C+c \!\sim\! C, \, C\cdot c \!\sim\! C, \, \nu(C)\!\sim\! C \text{ and } \nu\in \mathcal{U}$.
    \item Linear: $x\!+\!C \!\sim\! L, \, x\!\cdot\! C \!\sim\! L, L\!+\!C \!\sim\! L,\,  L\!\cdot\! C \!\sim\! L, L\!+\!x \!\sim\! L$.
    \item Polynomial: $L\cdot x \!\sim\! P, \,P+C \!\sim\! P, \,P \cdot C \!\sim\! P, \, P+x \!\sim\! P,$ \\\hspace*{4.7em} $P \cdot x \!\sim\! P, L \cdot L \!\sim\! P$, $P + L \!\sim\! P$, $P \cdot L \!\sim\! P$.
\end{itemize}
\end{definition}

\begin{proposition}[Recursive Structural Decomposition]
\label{sec:prop1}
Let $f(x)$ be defined by a finite composition of unary and binary ($+$ and $\times$) operators. 
Then, $f(x)$ obeys: 
$ f(x) = v_0 + \sum_i v_i \prod_j \nu_{ij}(v_{ij}(x))$,
where $v_0, v_i \in \mathbb{R}$, $\nu_{ij} \in \mathcal{U} \cup \{I\}$, and $v_{ij}(x)$ are subexpressions that admit the same decomposition recursively until reaching variables or constants.
\end{proposition}

Proposition~\ref{sec:prop1} was proved by \citet{SeTGAP}.
Although the structure is defined using $\mathcal{B} = \{ +, \times\}$, expressions involving other binary operations, such as subtraction or division, are subsumed by this decomposition including appropriate unary operators, such as negation and inversion. 
In \textbf{S.M.}~\ref{app:extension}, we show that certain functions 
that may not mirror this form, such as piecewise-defined functions, can be represented using operations acting upon functional blocks that are individually expressible through this structural decomposition.

In addition, the power operator is treated as a unary operator, with the exponent regarded as a parameter external to the skeleton. 
That is, $\texttt{Pow}(f(x))$ denotes the family $\{f(x)^h \!:\! h \!\in\! \mathbb{Q}\}$, while the specific instantiated form is obtained in Sec.~\ref{sec:sampling} by sampling $h$.
Note that exponents outside $\mathbb{Q}$ can be represented through compositions of $\exp$ and $\log$. 
This choice constrains the expression space by excluding exponents that depend on $x$, thereby avoiding a broader class of expressions that is not considered in this study. 
This restriction is consistent with the standard power-law parametrization used in scientific modeling settings, where the exponent is typically constant~\cite{feynman}. 
However, variable-exponent formulations do exist in more general mathematical settings outside our scope.

\begin{definition}[Affine expansion]
For any subexpression or leaf $u$, the affine expansion operator is defined as $\mathcal{A}[u] = c_i \cdot u + c_j$, where $c_i, c_j$ are newly instantiated placeholders.
\end{definition}

\begin{definition}[Normal form]
Let $\mathcal{N}(s)$ denote the normal form of skeleton $s$, obtained by recursively applying the affine expansion operator $\mathcal{A}[\cdot]$ to every node and leaf of the $s$ tree, and the equivalence rules $\sim$, until no further transformation is possible.
For example, given $s(x) = \sin(x) + \frac{x^2}{c}$; then, $\mathcal{N}(s) = c \, \sin(cx+c) + cx^2+c = c \, \sin(L) + P$.
\end{definition}

\begin{lemma}[Recursive Canonical Form]
\label{lemma1}
Every skeleton $s(x) \in S_{\text{skel}}$ is subsumed by a unique finite rooted tree $\mathcal{N}(s)$ with leaves $\in \{C, L, P\}$ and internal nodes that satisfy the recursive structural decomposition defined in Proposition~\ref{sec:prop1}.
That is, $s(x) = c_0 + \sum c_i \prod \nu_{ij}(T_{ij}(x))$, where $c_0, c_i$ are constant placeholders and $T_{ij}$ skeleton subtrees.
\end{lemma}

\begin{theorem}
\label{theorem1}
    Let $\mathcal{S} \!=\! \{ \mathcal{N}(s)\!\! :\! s \!\in\! \mathcal{S}_{\text{skel}} \}$ be the set of canonical univariate skeletons obtained from skeletons $s$ with a total unary operator count bounded by $N_U$ and a maximum unary nesting depth bounded by $N_D$. 
    Then, the set $\mathcal{S}$ is finite.
\end{theorem}

Lemma~\ref{lemma1} and
Theorem~\ref{theorem1} are proved in \textbf{S.M.}~\ref{app:lemma1} and~\ref{app:theorem1}.

\section{Grammar-based Skeleton Generation}

Theorem~\ref{theorem1} establishes that the \textbf{ex}haustive \textbf{ca}nonical \textbf{u}nivariate \textbf{s}keletons (\textbf{ExCaUS}) set $\mathcal{S}$ is finite under the constraints $N_U$ and $N_D$. 
A procedure to construct $\mathcal{S}$ must ensure that each element is generated exactly once, avoiding the redundant evaluation of non-canonical or functionally equivalent forms.
Below, we design a CFG to produce expressions that are, by construction, in their normal form $\mathcal{N}(s)$.

\vspace{-1.5ex}
\subsection{Grammar Definition}

We define a grammar $G = (V, \Sigma, R, \sigma)$ where $V$, $\Sigma$, and $R$ are the set of non-terminals, terminal symbols, and production rules, respectively, and $\sigma$ is the start symbol. 
To align with the Proposition~\ref{sec:prop1}, 
the non-terminals represent the hierarchical levels of the sum-of-products form.
$R$ is defined as:
\[
\begin{aligned}
\sigma &\to \tau, &
\tau &\to C + \Omega \mid C + \Omega + \tau, &
\Omega &\to C \cdot \Phi,\\
\Phi &\to \Psi \mid \Psi \cdot \Phi, &
\Psi &\to \nu(\tau) \mid \mathrm{Atom}, &
\mathrm{Atom} &\to L \mid P.
\end{aligned}
\]

Here, 
$\Omega$ represents an individual component of the sum, $\Phi$ represents the product of unary blocks, and $\Psi$ is either an atom or a unary operator acting on a sub-skeleton.
Note that the atom $C$ is excluded from the terminal production rule, as its inclusion would introduce functional redundancies.
Specifically, any occurrences of $C$ within a functional branch are subsumed by the existing additive and multiplicative placeholders at higher levels of the structural hierarchy.
$G$ serves as the blueprint for generating expressions that adhere to the recursive canonical form (Lemma~\ref{lemma1}).
While the recursive nature of $R$ theoretically allows for infinite functional topologies, the constraints described in the following subsections bound the search space and derive a finite and exhaustive set $\mathcal{S}$.

Recall $\nu \in \mathcal{U} \cup \{I\}$.  
The set $\mathcal{U}$ is selected to include the most common functional forms found across real-world engineering and physics problems~\cite{feynman,SRthatscales}.
We employ the operators: $\sin$, $\tan$, $\arcsin$, $\arctan$, $\sinh$, $\cosh$, $\tanh$, $\text{arsinh}$, $\text{arcosh}$, $\text{artanh}$, $\exp$, $\log$, $\sqrt{\cdot}$, $\texttt{abs}$, $\texttt{square}$, $\texttt{cube}$, $\texttt{Pow}$, and $\texttt{inv}$.

\vspace{-1.5ex}
\subsection{Production Constraints for Canonicality}

To ensure that the grammar generates only irreducible skeletons in their normal form, we impose three logical constraints. 
They enforce the structural closure rules in Def.~\ref{def:equiv} 
These rules serve only as computational efficiency constraints that prune redundant trees early in the generation pipeline, ultimately yielding the same set $\mathcal{S}$ through systematic filtering.

\begin{itemize}[leftmargin=1em, noitemsep,topsep=0pt]
    \item Additive Independence: In the expansion of $\tau$, the terms $\{\Omega_i\}$ must be functionally linearly independent. Furthermore, a sum is prohibited if the resulting expression is subsumed by an existing atomic category. 
    For example, the combination $C \cdot P + C \cdot L$ is disallowed because it reduces to the polynomial atom $P$ via the closure rules.
    \item Multiplicative Irreducibility: In the expansion of $\Phi$, we enforce that the product $\Psi_1 \dots \Psi_k$ must be unique and non-atomic. 
    Since power operations are already defined as unary operators in $\mathcal{U}$, repeated factors are unnecessary. 
    In addition, products of atoms that simplify to a single atom (e.g., $L \cdot L$ and $L \cdot P$, which resolve to $P$) are also excluded.
    \item Functional Irreducibility: A branch $\Psi = \nu(\sigma)$ is permissible if it produces a functional form distinct from the atoms $L$ and $P$. 
    If a composition results in an expression equivalent to a linear or polynomial form (e.g., $\ln(e^x) \sim x$), the production is pruned. 
    This also prevents vacuous operations on constants, enforcing the rule $\nu(C) \sim C$ by requiring $\sigma$ to contain at least one instance of the variable $x$.
\end{itemize}

\vspace{-1.5ex}
\subsection{Complexity Constraints}

To ensure that $\mathcal{S}$ is a finite and exhaustive dataset, our derivation process is governed by two constraints:
\begin{itemize}[leftmargin=1em, noitemsep,topsep=0pt]
    \item Global Unary Bound: We consider a unary count, $n_u$, which is incremented globally with every application of the rule $\Psi \to \nu(\sigma)$. 
    For any generated skeleton $s$, $n_u$ must not exceed $N_U$, restricting the total functional complexity.
    \item Compositional Nesting Bound: A local nesting depth $n_d$ tracks the recursive depth of the functional composition.
    Thus, the production $\Psi \to \nu(\sigma)$ increments $n_d$. 
    The rule is only available if the current depth $n_d < N_D$.
\end{itemize}




To mitigate computational overhead and prevent the pipeline from producing unrealistic skeletons or expressions that are functionally equivalent to others, 
an additional set of simplification and filtering rules is given in \textbf{S.M.}~\ref{app:rules}.

\vspace{-1.5ex}
\subsection{Symmetry Breaking and Uniqueness}

To ensure that the resulting dataset $\mathcal{S}$ contains unique functional representatives, we must address the structural symmetries inherent in the commutativity of addition and multiplication. 
Without such constraints, the grammar would generate redundant entries; e.g., $C\sin(L) + C\exp(L)$ and $C\exp(L) + C\sin(L)$.
Thus, we enforce a strict lexicographical ordering during the expansion of the recursive rules.
When generating a sum of terms ($\Omega$) or a product of factors ($\Psi$), the components are sorted according to a fixed priority:

\begin{enumerate}[leftmargin=1em, noitemsep,topsep=0pt]
    \item Complexity: Components are ordered by their total number of unary operators ($n_u$) and then by their nesting depth ($n_d$).
    \item Order: For elements with equal complexity, ties are broken by the alphabetical order of the string representation of their operators (e.g., $\exp(\cdot) \prec \ln(\cdot) \prec \sin(\cdot)$) and atoms ($L \prec P$).
\end{enumerate}


\vspace{-1.5ex}
\subsection{Implementation Algorithm}

The implementation of the skeleton generator is structured as a state-space expansion problem, where $G$ serves as the structural foundation and the symbolic library SymPy (\href{www.sympy.org/}{\text{sympy.org}}) is used as an engine for algebraic normalization.
By adopting a breadth-first traversal strategy, the skeletons are enumerated hierarchically according to their structural complexity.
The generator stores every unique intermediate derivation string encountered in an efficient hash table.

\setlength{\textfloatsep}{7pt}
\begin{algorithm}[t]
\fontsize{6.5}{7.1}\selectfont
\caption{Dataset Generation of Canonical Skeletons}
\label{alg:enumeration}
\renewcommand{\algorithmicrequire}{\textbf{Inputs:}}  
\renewcommand{\algorithmicensure}{\textbf{Output:}}
\begin{algorithmic}[1]
\Require $G$, $N_U=2$, $N_D=2$. \hspace{6em} \textbf{Output:} Set of unique canonical skeletons $\mathcal{S}$
\State $\mathcal{S} \gets \emptyset$
\State $Q \gets \text{Queue}([\sigma])$ \Comment{Initialize with the start symbol}
\State $\text{SeenPaths} \gets \text{Set}([\sigma])$ \Comment{Track partial derivations}

\While{$Q$ is not empty}
    \State $\text{expr} \gets Q.\text{pop}()$
    
    \If{$\text{expr}$ is fully resolved} \label{alg:FT1}
        \State $\text{s}_{\text{sorted}} \gets \texttt{CanonicalSort}(\texttt{simplify}(\text{expr}))$
        \State $\mathcal{S}.\texttt{add}(\text{s}_{\text{sorted}})$
        \State \textbf{continue} \label{alg:FT2}
    \EndIf
    
    \State $NT \gets \texttt{findLeftmostNT}(\text{expr})$
    \ForAll{productions $NT \to \beta$ in $G$}
        \State $\text{path} \gets \texttt{applyRule}(\text{expr}, NT \to \beta)$ \label{alg:applyRule}
        \State $\text{s} \gets \texttt{getsubTree}(\text{path})$
            \State $\text{s}' \gets \texttt{simplify}(\text{s})$
            
            \If{$\texttt{cmplx}(\text{s}') < \texttt{cmplx}(\text{s})$} 
                \State \textbf{reject} $\text{path}$ \Comment{Functional Irreducibility} \label{alg:irreduc}
                \State \textbf{continue}
            \EndIf
            \State $\text{s}_{\text{can}} \gets \texttt{CanonicalSort}(\text{s}')$ \Comment{Sort locally} \label{alg:localsort}
            \State $\text{path} \gets \texttt{updateSubtree}(\text{path}, \text{s}_{\text{can}})$ 
        
        \Statex \hspace*{3em} \# Global Constraint Validation
        \If{$\texttt{countNu}(\text{path}) \!\le\! N_U$ and $\texttt{maxNd}(\text{path}) \!\le\! N_D$ and $\text{path} \!\notin\! \text{SeenPaths}$}\label{alg:addQ}
            \State $Q.\texttt{push}(\text{path})$ 
            \State $\text{SeenPaths}.\texttt{add}(\text{path})$
        \EndIf
    \EndFor
\EndWhile
\end{algorithmic}
\end{algorithm}

As shown in Algorithm~\ref{alg:enumeration}, the main loop extracts a partial expression from the queue $Q$ to identify its leftmost non-terminal using \texttt{findLeftmostNT}. 
For each admissible production rule from $G$ applied to the current expression, we generate a new candidate \texttt{path} (Line~\ref{alg:applyRule}).
Function \texttt{getSubtree} does a recursive look-back starting from the site of the most recent substitution.
It identifies the most immediate parent node in the derivation tree that has transitioned into a fully resolved state, meaning the entire branch rooted at that node now consists exclusively of terminal symbols.

Once this branch is isolated, it is passed to SymPy's \texttt{simplify} for algebraic simplification.
If the result is found to be functionally equivalent to an atom or if its expression complexity, measured using function \texttt{cmplx} as the number of nodes required by the expression tree, is lower than that of the unsimplified version, the entire derivation path is rejected  (Line~\ref{alg:irreduc}).
This ensures that any functional form that simplifies to a lower-complexity form is discarded, as the deterministic nature of the breadth-first traversal guarantees that the canonical representative of that form was already generated earlier.

If the extended path is not rejected, \texttt{CanonicalSort} (Line~\ref{alg:localsort}) organizes the components of the simplified subtree to enforce a unique representation. 
This function applies a multi-level priority sort, ordering terms or factors first by their total unary operator count and nesting depth, and then by the alphabetical string representation of their root operators. 
This locally canonicalized branch is then substituted back into the main derivation tree using function \texttt{updateSubtree}. 
Before a path is returned to $Q$, it must satisfy the global complexity constraints for $N_U$ and $N_D$ (Line~\ref{alg:addQ}). 
If a derivation is fully resolved, it undergoes a final algebraic reduction and sorting before being added to the hash-based set $\mathcal{S}$ (Lines~\ref{alg:FT1}--\ref{alg:FT2}).

\vspace{-1ex}
\section{Sampling Strategy} \label{sec:sampling}

Below, we outline the steps for instantiating mathematical expressions and sampling corresponding numerical datasets. 


\vspace{-1.5ex}
\subsection{Sampling Expression Constants and Support}

Given a skeleton $\mathbf{e}(x) \in \mathcal{S}$, we construct a parameterized function $f(x)$.
Atoms $C$ are replaced with constant placeholders $c$, and atoms $L$ are expanded to $c\,x + c$. 
Atoms $P = P^m(x)$ are replaced by polynomials $\sum_{i=0}^m c_i x^i$, where the degree $m$ is sampled from $\mathcal{U}(2, 6)$, or constrained to $\mathcal{U}(2, 3)$ when nested within a unary operator to prevent explosive growth.

Power operators $\texttt{Pow}(\cdot)$, treated as unary in Sec.~\ref{sec:datasetdef}, are instantiated by assigning a rational exponent $h \in (0, 7)$. 
To avoid overriding primitive operators, values $h \in \{0.5, 1, 2, 3\}$ are excluded. 
The exponent $h$ is sampled with a $70\%$ probability as an integer and a $30\%$ probability as a fraction $p/q$ with denominator $q \in \{2, \dots, 6\}$. 
Once the symbolic structure is complete, each placeholder $c$ is sampled independently from $\mathcal{U}(-50, 50)$, yielding a fully parameterized function $f(x)$.

Finally, $f(x)$ is evaluated to generate input--response pairs $(X, Y)$ with $n = 10,000$ points. 
The support vector $X \in \mathbb{R}^n$ is drawn uniformly from $\mathcal{U}(-10, 10)$. 
Although this domain range is held constant to maintain evaluation uniformity, the randomized sampling of the internal constants scales the functions across diverse numerical ranges naturally. 
The response vector is subsequently evaluated as $Y = f(X)$.

\vspace{-1.8ex}
\subsection{Sampling Data} \label{sec:samplingdata}

To construct a robust dataset for training and evaluating SR models, $(X, Y)$ pairs must be sampled without triggering numerical instabilities or relying on artificial protected operators. 
Thus, our data generation builds upon the safeguard mechanisms introduced by \citet{MultiSetSR}.

Our data generation process parses expressions hierarchically to handle domain constraints across three classes of ``special functions'': single-bounded (e.g., $\log$ and $\sqrt{\cdot}$), double-bounded (e.g., $\arcsin$ and $\arccos(x)$), and operators with vertical singularities (e.g., $\texttt{inv}$ and $\tan$). 
While single- and double-bounded operations are managed via dynamic input shifting and linear scaling (see \textbf{S.M.}~\ref{app:invalidops} for full tree-traversal details), our handling of vertical asymptotes represents a major point of divergence and improvement over prior work.

Existing approaches to managing vertical asymptotes safely rely on post-hoc truncation or crude grid heuristics: 
\begin{itemize}[leftmargin=1em, noitemsep,topsep=0pt]
    \item \citet{SRthatscales} employed post-hoc truncation, dropping individual input--response pairs that evaluate to undefined values (NaN). A critical drawback is that expressions such as $\ln(x-8)$ yield undefined outputs across major regions of $x \in [-10, 8]$, severely shrinking the target dataset and creating unrepresentative input spaces of varying lengths.
    \item \citet{MultiSetSR} attempted to identify singularities by tracking zero-crossings on a uniform grid (e.g., 1000 points) and enforcing a fixed spatial exclusion window ($\epsilon = 0.05$). This approach suffers from two flaws: (i) a uniform spatial threshold ignores local scaling; e.g., for $f(x) = \frac{1}{x}$, $\epsilon = 0.05$ restricts outputs to a safe $f(0.05) = 20$, whereas for a steeper case like $f(x) = \frac{1}{x^5}$, the same distance yields $f(0.05) = 3.2 \times 10^6$, instantly triggering numerical overflow. (ii) Coarse grid discretization misses narrow singularities in high-frequency terms such as $\sin(20x^2)$.
\end{itemize}

\noindent To overcome these limitations, we propose a derivative-driven singularity avoidance mechanism.
Hence, the exclusion boundary is determined based on the geometric slope angle of the function. 
We define the unsafe asymptotic regime as the region where the tangent to the curve approaches a vertical orientation, indicating a structural explosion. 
Formally, we project the first derivative $f'(x)$ onto a bounded $[0, 1)$ metric representing the sine of the local slope angle. 
We exclude any input point $x$ that satisfies the condition on the sine of the slope angle, $\frac{|f'(x)|}{\sqrt{1 + (f'(x))^2}} \geq \alpha$, where $\alpha \in [0, 1)$ represents a hyperparameter threshold (e.g., $\alpha = 0.9995$, which structurally restricts the local slope angle to less than $\approx 88.2^\circ$). 
By using this geometric condition, sharp asymptotes receive a tightly compressed exclusion window, whereas wide asymptotic walls prompt the window to expand, as depicted in \textbf{S.M.}~\ref{app:invalidops}. 
Hence, this analysis yields a continuous exclusion window around each singularity point, which is stored to delineate forbidden subdomains, ensuring that the final equidistant support points are sampled exclusively within the remaining safe regions.


\section{Experimental Results}

\subsection{Dataset Comparison}

We compare 
our ExCaUS dataset to those of two neural SR approaches, SeTGAP~\cite{SeTGAP} and NeSymRes~\cite{SRthatscales}, with publicly available datasets.
The SeTGAP dataset contains $1.65$ million univariate skeletons, whereas the NeSymRes dataset comprises $10$ million skeletons spanning up to three input variables.
For fairness, we canonicalize both baseline datasets using our closure rules (Def.~\ref{def:equiv}).
To ensure consistency, we re-skeletonize NeSymRes's expressions with respect to their 
primary variable, $x_1$, which is present in all instances.
Hence, given a multivariate skeleton $s(x_1, x_2, x_3)$, a canonical univariate skeleton is obtained as $\mathcal{N}(\kappa(s(x_1, x_2, x_3) ; x_1))$.

\begin{table}[t]
\centering
\caption{Quantitative comparison of SR datasets}
\label{tab:dataset_comparison}
\vspace{-1ex}
\small
\resizebox{0.5\columnwidth}{!}{%
\begin{tabular}{l r r r}
\toprule
\textbf{Metric} & \textbf{ExCaUS} & \textbf{SeTGAP} & \textbf{NeSymRes} \\
\midrule
Total Expressions & 71,027 & 1,650,000 & 10,000,000 \\
Unique Skeletons & 71,027 & 664 & 128,514 \\
Duplication Rate (\%) & 0.00 & 99.96 & 98.71 \\
Dictionary Size (Unique Ops) & 16 & 12 & 11 \\
Mean Unary Nesting Depth & 1.58 & 1.29 & 1.12 \\
Max Unary Nesting Depth & 2 & 3 & 6 \\
Mean Operators per Expression & 2.32 & 1.36 & 1.31 \\
Max Operators per Expression & 4 & 3 & 10 \\
\bottomrule
\end{tabular}
}
\vspace{-2.7ex}
\end{table}

\begin{figure}[!t]
    \centering
    \includegraphics[width=.92\columnwidth]{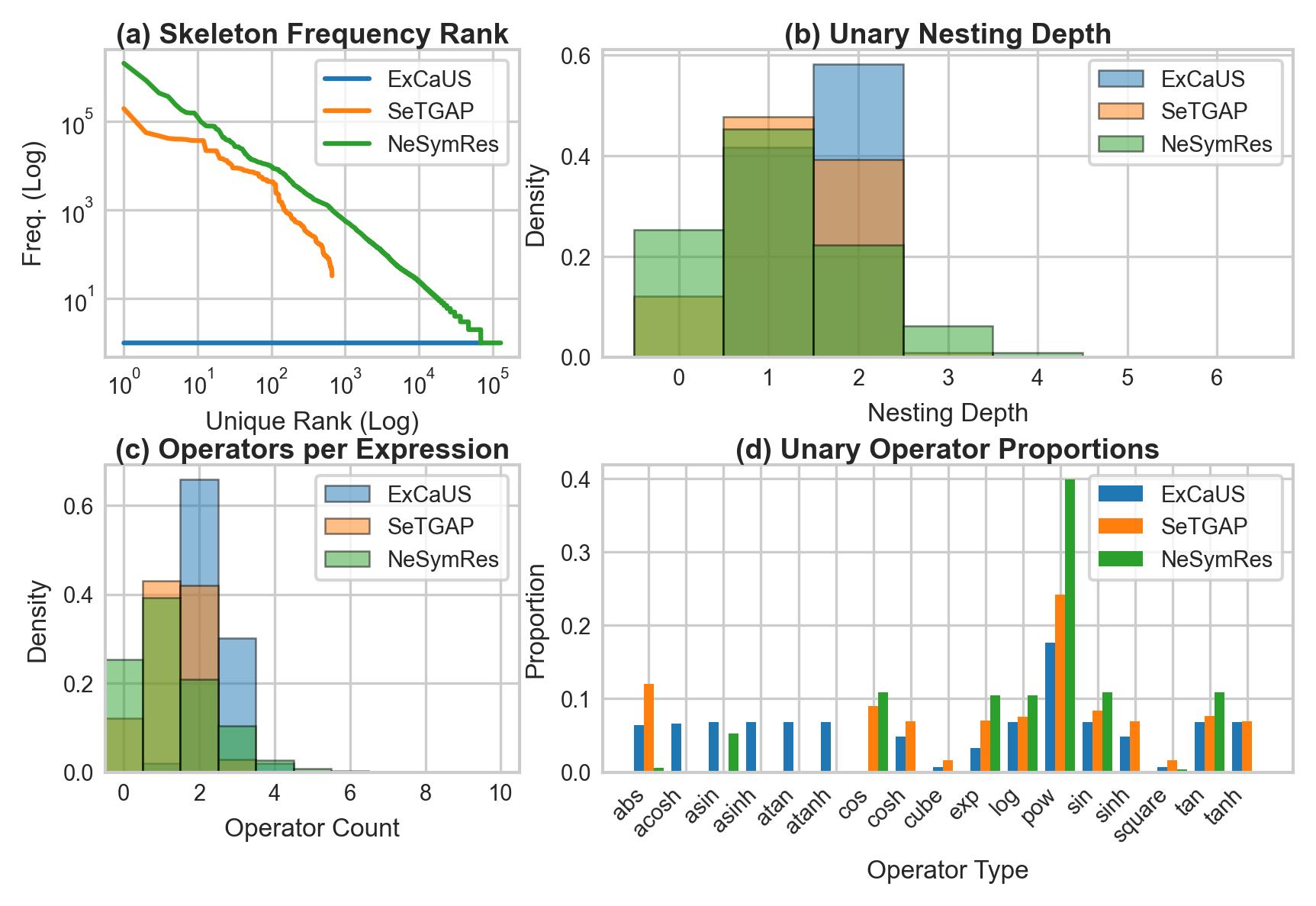}
    \vspace{-2ex}
    \caption{SR datasets comparison}
    \label{fig:datasets}
    \vspace{-1.5ex}
\end{figure}

To evaluate structural scale, redundancy, and complexity across datasets, 
in Table~\ref{tab:dataset_comparison}, we report the total raw expression count $N$, the number of unique canonical skeletons $U$, and the resulting duplication rate, $1\!-\!U/N$.
To evaluate functional complexity, we report the mean and maximum values for both operator count per expression and unary nesting depth.

Fig.~\ref{fig:datasets} depicts the datasets' distributional characteristics:

\begin{itemize}[leftmargin=1em, noitemsep,topsep=0pt]
\item \textbf{Frequency Rank (Fig.~\ref{fig:datasets}a):} Displays the unique skeletons' frequency sorted by rank on a log-log scale. 
This is used because scalar duplication rates fail to capture frequency concentration, revealing whether dataset expansion yields structural variety or merely repeats high-frequency heads.
\item \textbf{Unary Nesting Depth (Fig.~\ref{fig:datasets}b):} Shows the empirical probability density function of functional nesting depth. 
\item \textbf{Operators per Expression (Fig.~\ref{fig:datasets}c):} Probability density function of unary operator counts per expression.
\item \textbf{Unary Operator Proportions (Fig.~\ref{fig:datasets}d):} Illustrates the relative frequency distribution across individual operator types to evaluate operator usage balance across datasets.
\end{itemize}

\vspace{-1.5ex}
\subsection{Benchmarking}  \label{sec:resbench}

The dataset $\mathcal{S}$ also serves to construct a challenging benchmark, which we refer to as \textbf{ExCaUS-Bench}. 
Existing benchmarks often feature problems where evaluations readily achieve near-optimal performance. 
For instance, the Feynman dataset~\cite{feynman} mostly comprises additively separable equations, and the limited variable domains 
prevent the accurate identification of functional forms. 
For example, the narrow domains applied to equation \texttt{I.9.18} ($\frac{G\,m_1\, m_2}{(x_2 - x_1)^2 + (y_2 - y_1)^2 + (z_2 - z_1)^2}$) (specifically $[3, 4]$ for $x_1, y_1, z_1$ and $[1, 2]$ for $x_2, y_2, z_2$) restrict a model's capacity to infer the correct asymptotic behaviour.

Thus, we stratified $\mathcal{S}$ into three tiers of difficulty (Easy, Medium, and Hard) measured by the length of the skeleton string representation. 
We randomly sampled 50 skeletons per tier. 
Following the methodology outlined in Sec.~\ref{sec:sampling}, each skeleton was used to generate a single parametric equation and a corresponding dataset of 10,000 points. 
The resulting equations (denoted E1--E50, M1--M50, and H1--H50) and their generated curves are detailed in Tables~\ref{tab:eq_e}--\ref{tab:eq_h} and Figs.~\ref{fig:eq_e1}--\ref{fig:eq_h3} in \textbf{S.M.}~\ref{app:benchmark}. 
To evaluate generalization, each dataset was divided into an interpolation set (the middle 90\% of the data, bounded between the 5th and 95th percentiles) and an extrapolation set (comprising the bottom 5\% and top 5\% tails).

Using SRToolkit~\citep{SRToolkit} as a unified infrastructure, we evaluated 13 SR methods: 
AIFeynman~\citep{feynman}, Bayesian SR (BSR)~\citep{Jin2019_BSR}, Deep SR (DSR)~\citep{petersen2021deep}, GPlearn (\href{https://gplearn.readthedocs.io}{\text{gplearn.readthedocs.io}}), GP-GOMEA~\citep{Virgolin2021_GPGOMEA}, ITEA~\citep{deFranca2021_ITEA}, Operon~\citep{Kommenda2019_ParamID}, PySR~\citep{pysr2}, NeSymReS~\citep{SRthatscales}, E2E~\citep{SRtransformer}, TaylorGP~\citep{TaylorGP}, Transformer-based Planning for SR (TPSR)~\citep{shojaee_transformer-based_2023}, and unified DSR (uDSR)~\citep{uDSR}.
For hyperparameter tuning, we performed an exhaustive grid search using one representative equation per tier (E0, M9, and H9). 
For fairness, each candidate configuration was executed five times per problem, selecting the hyperparameter set yielding the 
average \ac{mse} across valid (non-$\mathrm{NaN}$) runs.
The optimal configurations are detailed in \textbf{S.M.}~\ref{app:hyperparameter}.

All methods were evaluated on ExCAUS-Bench with these parameters across ten independent runs per equation, training on the interpolation datasets. 
Figs.~\ref{fig:mse} and~\ref{fig:r2} show extrapolation performance in MSE and coefficient of determination, $R^2$, respectively. 
For each problem, we computed the average metric across valid runs. 
Each subplot summarizes the distribution of these per-problem averages across all equations in that tier, where box boundaries represent the interquartile range (IQR), solid lines indicate medians, and individual points denote outliers.
Tables~F.3--F.8 of \textbf{S.M.}~F report the per-problem performance 
for interpolation and extrapolation MSE and $R^2$.

\vspace{-1.0ex}
\section{Discussion}

Table~\ref{tab:dataset_comparison} and Fig.~\ref{fig:datasets} show the efficiency contrast between our deterministic framework and stochastic sampling pipelines. 
While the NeSymRes and SeTGAP datasets require generating millions of expressions, 
post-hoc canonicalization reveals substantial structural redundancy, with duplication rates over $98.7\%$ and $99.9\%$. 
As illustrated 
in Fig.~\ref{fig:datasets}(a), stochastic samplers 
over-expose models to a small subset of trivial functions. 
Specifically, expressions reducing to atoms $C$, $L$, and $P$ total $2927$, $2089502$, and $443383$ in NeSymRes (25.3\% of all duplicates), and $0$, $2269$, and $197479$ in SeTGAP (11.97\% of all duplicates), evidencing a strong simplicity bias.
In contrast, ExCaUS achieves a flat frequency-rank curve with a $0\%$ duplication rate, containing $71,027$ unique canonical skeletons. 
This reduces training memory and computation requirements while providing complete coverage of the functional space.

\begin{figure}[!t]
    \centering
    \includegraphics[width=.9\columnwidth]{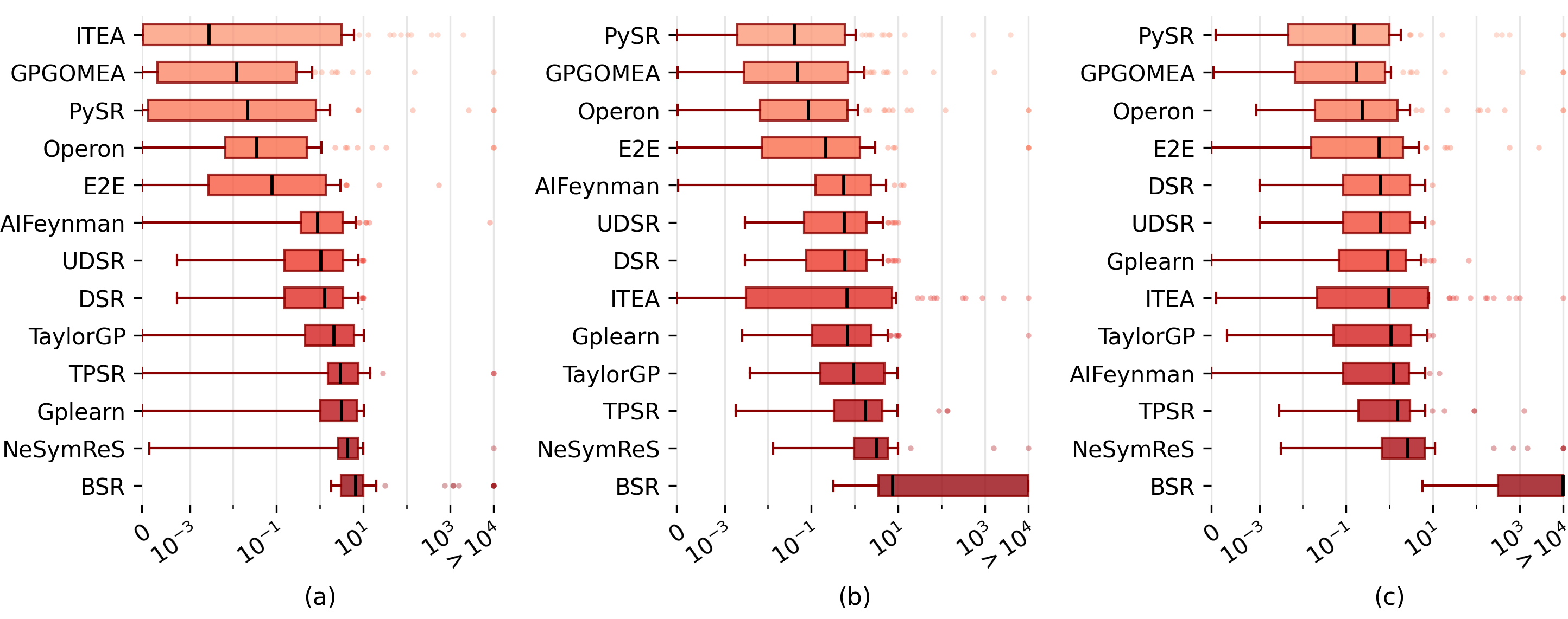}
    \vspace{-1.5ex}
    \caption{Extrapolation MSE on the \textbf{(a)} E, \textbf{(b)} M, and \textbf{(c)} H tiers}
    \label{fig:mse}
    \vspace{-3.5ex}
\end{figure}

\begin{figure}[!t]
    \centering
    \includegraphics[width=.9\columnwidth]{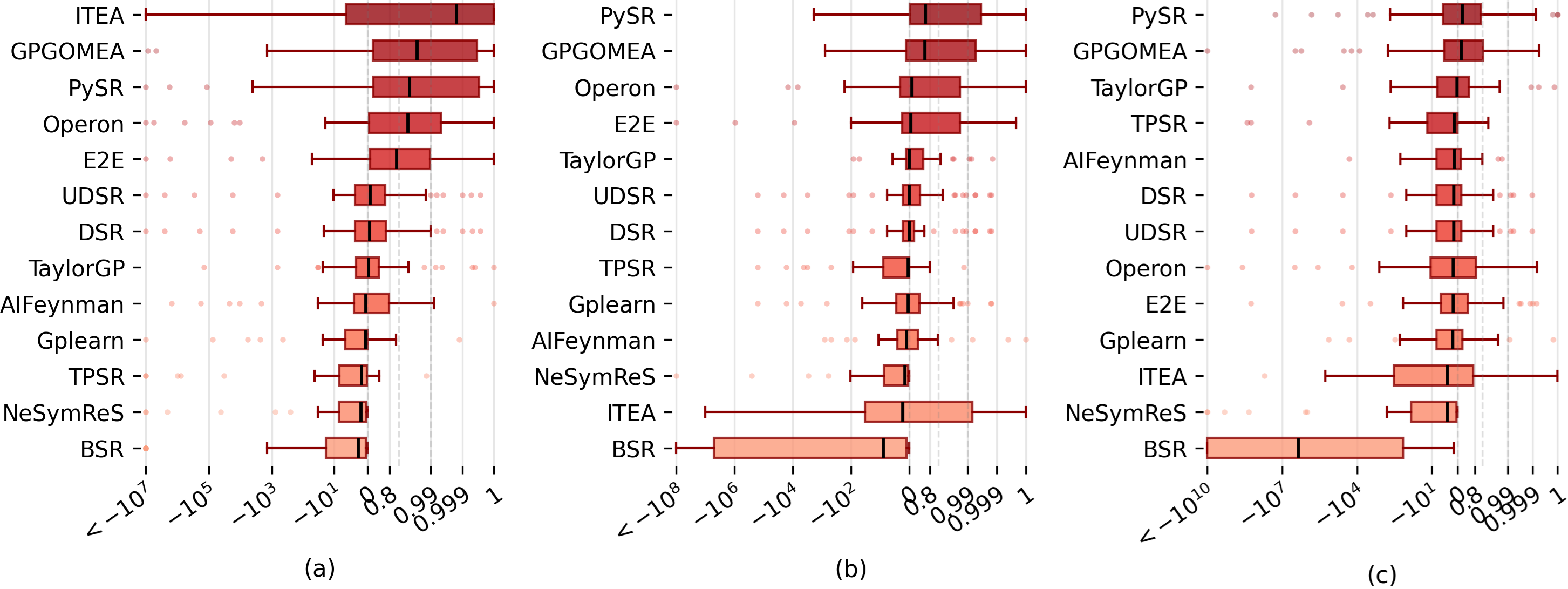}
    \vspace{-1.5ex}
    \caption{Extrapolation $R^2$ on the \textbf{(a)} E, \textbf{(b)} M, and \textbf{(c)} H tiers}
    \label{fig:r2}
    \vspace{-1ex}
\end{figure}

Furthermore, Figs.~\ref{fig:datasets}(b--d) highlight how ExCaUS effectively eliminates the simplicity and operator biases inherent to probabilistic sampling. 
As shown in Figs.~\ref{fig:datasets}(b) and (c), stochastic methods concentrate density heavily at low complexity levels (nesting depth $0$--$1$ and $\le 1$ operators per expression), leaving deeper structural topologies under-represented despite maintaining long tails (e.g., maximum operator counts up to $10$). 
In contrast, ExCaUS distributes expressions in a more balanced manner, reaching an average of $2.32$ operators per expression and a mean unary nesting depth of $1.58$. 
Crucially, Fig.~\ref{fig:datasets}(d) reveals extreme operator imbalance in the compared datasets, where \texttt{pow} dominates up to $40\%$ of all unary instances in NeSymRes. 
Conversely, ExCaUS spans a larger vocabulary ($16$ unique operators) with a notable uniform proportion across all functional families.
Although perfect operator balance is not strictly achieved due to the additional structural filtering rules introduced in \textbf{S.M.}~\ref{app:rules}, this well-apportioned coverage effectively prevents operator-level bias. 
This ensures that when neural SR models are pre-trained on ExCaUS in future work, they will learn unbiased representations across diverse mathematical domains.

Evaluation across ExCaUS-Bench (Figs.~\ref{fig:mse} and~\ref{fig:r2}) 
determines
that state-of-the-art SR methods struggle as structural complexity increases. 
On the Hard tier, the leading methods in extrapolation $R^2$ are PySR, GPGOMEA, and TaylorGP, whereas the top three in extrapolation MSE are PySR, GPGOMEA, and Operon. 
$R^2$ assesses an expression's capacity to capture the functional shape and variance, allowing methods like TaylorGP to achieve high $R^2$ despite small vertical shifts or scale mismatches that severely penalize MSE.
Conversely, algorithms like Operon specialize at fine-tuning numerical constants via nonlinear parameter optimization, minimizing residuals even if subtle trend mismatches slightly lower $R^2$.

Predictive capability degrades on the Medium and Hard tiers. 
Here, the extrapolation MSE distributions widen, producing a spread of extreme outliers. 
In addition, the median $R^2$ values shift heavily toward zero or achieve extreme negative values, indicating a complete failure to capture the underlying data variance outside the interpolation range. 
This drop in performance highlights the challenge introduced by ExCaUS-Bench: although modern algorithms can locally approximate numerical data within training domains, they systematically fail to deduce the underlying functional forms required for true mathematical extrapolation, even in univariate settings.

\vspace{-0.5ex}
\section{Conclusions}

Neural SR models are limited by the structural redundancy and simplicity bias of traditional stochastic sampling pipelines. 
To address this, we presented ExCaUS, an exhaustive dataset of univariate canonical skeletons generated via a deterministic constrained CFG, alongside ExCaUS-Bench, a 150-problem SR benchmark. 
ExCaUS represents a compact, operator-balanced dataset that eliminates functional redundancy and simplicity bias. 
Furthermore, by challenging algorithms that otherwise perform well on standard benchmarks, ExCaUS-Bench serves as a rigorous, non-redundant testbed to push future improvements in SR.
Building on these foundations, in the future we will focus on training neural SR models on datasets derived from ExCaUS and our proposed sampling strategy to achieve robust symbolic expression learning.



\vspace{-1ex}
\printbibliography


 




\counterwithin{figure}{section}
\counterwithin{table}{section}
\counterwithin{algorithm}{section}
\counterwithin{equation}{section}

\clearpage
\appendices

\counterwithin{lemma}{section}

\section{Related Work} \label{app:related}

The performance of neural SR models is tied to the quality and diversity of their training datasets.
Thus far, the generation of these datasets has relied on stochastic tree-sampling methods, where mathematical expressions are constructed by randomly branching from a set of operators and terminals.
\citet{Lample2020Deep} pioneered this approach by generating millions of equations to train Transformer models for symbolic integration and \ac{ode} solving.
It uses a one-pass stochastic procedure designed to sample mathematical expression trees uniformly, so that every valid structure of a specific size has a similar probability of being selected.
While this allows for the generation of millions of samples, it remains a probabilistic exploration.
As the complexity of the expressions increases, the number of possible configurations grows exponentially, and uniform sampling can only capture a sparse fraction of the functional universe.

Building on this, \citet{SRthatscales} utilized a similar sampling method that incorporated a dataset of 10 million equations with a strict limit of five non-leaf nodes.
Their approach uses a post-processing step to simplify the resulting expression strings into skeletons.
However, since the initial sampling is stochastic, the process frequently generates trees that are structurally distinct but functionally redundant
(e.g., $x \cdot 1$ vs. $x+0$).
This calls for substantial data cleansing to remove invalid or trivial expressions, a challenge also faced by \citet{SRtransformer}.
To maximize variety, they apply random affine transformations ($ax+b$) to variables and operators.
While this prevents the model from seeing identical numerical realizations, it does not address the underlying collision problem.
That is, different stochastic sampling paths often converge on the same functional identity, leading to a significant simplicity bias where the model is over-exposed to over-represented forms.

\citet{SymbolicGPT} introduced a distinct approach.
Instead of sampling trees by the number of nodes, they use a template-based strategy in which equations are generated from a balanced binary tree of fixed depth.
This depth controls the number of functional nestings and, hence, expression complexity.
To allow for the representation of shallower or unbalanced forms, they use terminal nodes that truncate subtrees.
Constants are integrated via a probabilistic insertion ratio.
Nevertheless, the approach remains a stochastic decoration of fixed templates, which does not inherently prevent the generation of functionally redundant expressions.

Furthermore, \citet{MultiSetSR} introduced a more controlled stochastic framework for training a Multi-Set Transformer, using a larger set of operators than \citet{SRthatscales} and a probability-weighted sampling strategy to favor sound expressions over overly complex ones.
For example, they implemented a system to prevent algebraic redundancies and numerical instabilities.
They also utilized both additive and multiplicative placeholders at every expression subtree without restricting the total number of placeholders to generate more general skeletons. 
While these heuristics reduce the occurrence of junk expressions and improve structural diversity compared to prior random samplers, the underlying process remains probabilistic and relies on arbitrary sampling weights.

Grammar-based and constrained generation emerge as potential alternative methods.
\citet{Yu_2025} proposed a framework to generate libraries of 
50,000 skeleton expressions for \ac{ode} discovery.
The complexity of produced expressions is strictly limited by a maximum expression length, ranging from 40 to 65 elements.
For some benchmarks, they employ probabilistic \acp{cfg}, where lower probabilities are assigned to complex rules to bias the generation toward simpler functional forms.
While their approach ensures that generated skeletons conform to the structural requirements of differential equations, it still represents a sampling strategy designed for exploration.

\section{Lemma and Theorem Proofs} \label{app:proofs}

\subsection{Lemma 1: Recursive Canonical Form} \label{app:lemma1}

{\renewcommand{\thelemma}{}
\begin{lemma}
Every skeleton $s(x) \in S_{\text{skel}}$ is subsumed by a unique finite rooted tree $\mathcal{N}(s)$ with leaves $\in \{C, L, P\}$ and internal nodes that satisfy the recursive structural decomposition defined in Proposition~\ref{sec:prop1}.
That is, $s(x) = c_0 + \sum c_i \prod \nu_{ij}(T_{ij}(x))$, where $c_0, c_i$ are constant placeholders and $T_{ij}$ skeleton subtrees.
\end{lemma}
}

\begin{proof}
We prove the lemma by structural induction on the height $h$ of the arbitrary skeleton tree $s$.
For the base case, the leaf nodes, we have either the variable $x$ or placeholders $c$. 
Applying the normalization operator: $\mathcal{N}(x) = \mathcal{A}[x] = c_1 x + c_2$; i.e., the atom $L$.
For a placeholder, $\mathcal{A}[c]$ reduces to the atom $C$.
In both cases, the leaves of $\mathcal{N}(s)$ belong to the set $\{C, L, P\}$, satisfying the case for a tree of height $h=0$.

For the inductive case, assume that for all trees $s'$ of height $\le h$, $\mathcal{N}(s')$ exists, is finite, and follows the recursive template structure.
We now consider a tree $s$ of height $h+1$:

\begin{itemize}[leftmargin=1em, noitemsep,topsep=0pt]
    \item Unary Root ($s = \nu(s')$): By the inductive hypothesis, $s'$ has a normal form $\mathcal{N}(s')$. 
    Applying the normalization to the root $\nu$: $\mathcal{N}(s) = \mathcal{A}[\nu(\mathcal{N}(s'))] = c_i \nu(\mathcal{N}(s')) + c_j$. 
    Since $\mathcal{N}(s')$ is already a valid sum-of-products from our hypothesis, the resulting expression $s$ remains a finite tree that matches the template structure for a single term.
    \item Binary Root ($s = b(s_1, s_2)$): We consider $b \in \{+, \times\}$. By the inductive hypothesis, $s_1$ and $s_2$ have canonical forms $\mathcal{N}(s_1)$ and $\mathcal{N}(s_2)$ that follow the expected template:
    \begin{itemize}
        \item Addition: $\mathcal{N}(s_1 + s_2) = \mathcal{A}[\mathcal{N}(s_1) + \mathcal{N}(s_2)]$. Since $\mathcal{N}(s_1)$ and $\mathcal{N}(s_2)$ are already sums of products, their sum is a larger sum introducing new placeholders; e.g., $c_1 \mathcal{N}(s_1) + c_2 + c_3 \mathcal{N}(s_2) + c_4$. 
        The closure rules ensure the redundant placeholders collapse into a single one.
        \item Multiplication: $\mathcal{N}(s_1 \times s_2) = \mathcal{A}[\mathcal{N}(s_1) \times \mathcal{N}(s_2)]$. 
        Multiplying two sum-of-products results in a new sum-of-products, which aligns with the expected template.
        If $\mathcal{N}(s_1)$ and $\mathcal{N}(s_2)$ are canonical atomic skeletons (e.g., $L$ and $P$), the closure rules (e.g., $P \cdot L \sim P$) are invoked to ensure that the product is projected onto the atomic leaf $P$ rather than remaining a complex tree.
    \end{itemize}
\end{itemize}

\noindent In all cases, the normalization operator $\mathcal{N}$ takes an arbitrary subtree in the tree and projects it into a sum-of-products form wrapped in an affine expansion. 
Because this process is applied bottom-up, and the equivalence rules ensure that the atoms $\{C, L, P\}$ are stable under these operations, the final tree $\mathcal{N}(s)$ is guaranteed to be a finite representative following the recursive structural decomposition.

Finally, normalization is idempotent: $\mathcal{N}(\mathcal{N}(s)) = \mathcal{N}(s)$. 
Because $\mathcal{A}[\mathcal{A}[u]]$ simplifies to $\mathcal{A}[u]$ via the merging of arbitrary placeholders, normalizing again produces no further changes. 
By construction, $\mathcal{N}(s)$ \textit{is the maximal version of the functional topology of} $s$. Since every leaf and node has been expanded to its most general affine form before being reduced by $\sim$, any instantiation of $s$ is necessarily a subset of the functions reachable by $\mathcal{N}(s)$, meaning $\mathcal{N}(s) \succeq s$.

\end{proof}

\subsection{Theorem 1} \label{app:theorem1}

{\renewcommand{\thetheorem}{}
\begin{theorem}
    Let $\mathcal{S} \!=\! \{ \mathcal{N}(s)\!\! :\! s \!\in\! \mathcal{S}_{\text{skel}} \}$ be the set of canonical univariate skeletons obtained from skeletons $s$ with a total unary operator count bounded by $N_U$ and a maximum unary nesting depth bounded by $N_D$. 
    Then, the set $\mathcal{S}$ is finite.
\end{theorem}
}

\begin{proof}
    The proof follows from the combinatorial properties of the normal form $\mathcal{N}(s)$ and the constraints placed on the structural complexity of the skeletons in $\mathcal{S}_{\text{skel}}$.

    According to Lemma~\ref{lemma1}, the normalization process $\mathcal{N}$ maps any skeleton $s$ to a unique tree structure where every node is labeled from a strictly finite set. 
    Since these sets are finite by assumption, any possible canonical skeleton in $\mathcal{S}$ is composed of a finite alphabet of functional blocks.

    Furthermore, the constraints $N_U$ and $N_D$ define the space of the functional topologies. 
    Because $\mathcal{N}(s)$ follows the recursive structural decomposition, the number of $\nu_{ij}$ operators is bounded by $N_U$, the depth of the recursive composition is bounded by $N_D$, and the binary operators $(+, \times)$ serve only to link these unary terms in a sum-of-products form.    
    Crucially, $\mathcal{N}$ ensures that adding more binary operations or placeholders does not generate a new skeleton. 
    For instance, the affine and polynomial closure rules ensure that repeated additions or multiplications are absorbed into existing atomic skeletons. 
    Thus, two skeletons that differ only by the number of additive or multiplicative constants are projected onto the same unique representative in $\mathcal{S}$.

    Finally, since the number of nodes is bounded and the set of possible labels for each node is finite, the total number of distinct rooted trees that can be constructed is combinatorially finite. 
    Because $\mathcal{N}$ acts as a canonical map that projects the infinite space of raw skeletons into this bounded set of maximally expanded functional shapes, it follows that $\mathcal{S}$ must contain a finite number of elements.
\end{proof}

\section{Expressive Scope and Structural Extensions} \label{app:extension}

In this appendix, we justify the generalizability of the recursive structural decomposition (Proposition~\ref{sec:prop1}).
Specifically, we demonstrate how this formulation naturally accommodates piecewise-defined functions, complex-valued expressions, and differential equations.

\subsubsection{Piecewise-Defined Functions}

Piecewise functions are often perceived as problematic for continuous SR due to their localized rules. 
However, within our framework, a piecewise function can be seamlessly integrated through analytical operators.
Discontinuous or boundary-switching behaviour can be internalized into the unary operator set $\mathcal{U}$ using indicator functions, or approximated smoothly via continuous operators already present in $\mathcal{U}$, such as the hyperbolic tangent.

Consider a piecewise function split across a threshold $\tau$:
\[f(x) = \begin{cases} g(x) & \text{if } x < \tau \\ h(x) & \text{if } x \geq \tau \end{cases}
\]
Using the Heaviside step function $H(x) \in \mathcal{U}$ (where $H(x) = 1$ if $x \geq 0$ and $0$ otherwise), this expression is rewritten as an algebraic combination:
\[f(x) = g(x) (1 - H(x - \tau)) + h(x) H(x - \tau)\]
Expanding this product yields: 
\[
f(x) = g(x) - g(x)H(x - \tau) + h(x)H(x - \tau).
\]

If $g(x)$ and $h(x)$ are admissible expressions satisfying Proposition~\ref{sec:prop1}, their product with the unary transformations $H(x-\tau)$ maps directly onto the sum-of-products decomposition.
Therefore, piecewise boundaries can be treated as non-linear unary modulations of functional blocks.

\subsubsection{Complex-Valued Variables and Fields}

While Proposition~\ref{sec:prop1} is introduced using real-valued functions, its algebraic foundation extends naturally to the complex domain $\mathbb{C}$.
Since $\mathbb{C}$ constitutes an algebraic field (i.e., it satisfies all foundational axioms for addition and multiplication), the considered set of binary operators $\mathcal{B} = \{+, \times\}$ preserves the same associative, commutative, and distributive properties required by the recursive decomposition.

Let $f(x)$ be a complex-valued expression of a real variable $x$, mapped as $f: \mathbb{R} \to \mathbb{C}$, such that $f(x) = \text{RE}(x) + i\,\text{IM}(x)$.
We show that both $\text{RE}(x)$ and $\text{IM}(x)$ can be individually expressed using the recursive structural decomposition defined in Proposition~\ref{sec:prop1}. 
By definition, $f(x)$ is built from a finite composition of operators. 
We proceed inductively on the structure of the expression:

\begin{itemize}[leftmargin=1em, noitemsep,topsep=0pt]
    \item Base Case: Let $f(x) = a + ib$ for $a,b \in \mathbb{R}$. 
    Here, $\text{RE}(x) = a$ and $\text{IM}(x) = b$, both of which are real constants satisfying Proposition~\ref{sec:prop1} with $v_0 = a$ or $b$ and all other terms vanishing. 
    If $f(x) = x$, then $\text{RE}(x) = x$ and $\text{IM}(x) = 0$, which also trivially satisfies the proposition.

    \item Unary Operators: Let $f(x) = \nu(g(x))$, where $\nu$ is a complex unary operator acting on a subexpression $g(x)\in \mathbb{R}$. 
    Any such operator can be decomposed into its real and imaginary parts, $\nu(\cdot) = \nu_{\text{re}}(\cdot) + i\,\nu_{\text{im}}(\cdot)$, where $\nu_{\text{re}}, \nu_{\text{im}} \in \mathcal{U}$. 
    Thus, $\text{RE}(x) = \nu_{\text{re}}(g(x))$ and $\text{IM}(x) = \nu_{\text{im}}(g(x))$, preserving the unary-transformation mapping layer.

    \item Binary Operators:
    \begin{itemize}
        \item Addition: Let $f(x) = g(x) + h(x)$. Since complex addition is component-wise:
        \[\text{RE}(x) = \text{RE}_g(x) + \text{RE}_h(x),\]
        \[\text{IM}(x) = \text{IM}_g(x) + \text{IM}_h(x).\]
        If the real and imaginary components of $g(x)$ and $h(x)$ satisfy Proposition~\ref{sec:prop1}, their sums do as well, as addition appends terms to the outer linear summation $\sum_i v_i$.
        
        \item Multiplication: Let $f(x) = g(x) \times h(x)$. Exploiting the field axioms of $\mathbb{C}$, multiplication expands as:
        \[f(x) = (\text{RE}_g(x) + i\,\text{IM}_g(x)) \times (\text{RE}_h(x) + i\,\text{IM}_h(x)),\]
        \[\text{RE}(x) = \text{RE}_g(x)\text{RE}_h(x) - \text{IM}_g(x)\text{IM}_h(x),\]
        \[\text{IM}(x) = \text{RE}_g(x)\text{IM}_h(x) + \text{IM}_g(x)\text{RE}_h(x).\]
        Because subtraction is subsumed by addition via negation ($v_i \in \mathbb{R}$ allows negative coefficients), the real and imaginary components of the product are exclusively formed by the sums of products of the components of $g(x)$ and $h(x)$. By induction, multiplying these structures simply increases the number of terms inside the product operators ($\prod_j$) of the final architecture.
        \end{itemize}
\end{itemize}

Consequently, an SR model based on or trained on the real-valued univariate skeleton dataset proposed in this work could potentially be applied to identify complex-valued expressions. 
While this constitutes future work, we argue that an SR framework could be designed such that the trained model targets $\text{RE}(x)$ and $\text{IM}(x)$ using real-valued primitives while taking into account potential domain-specific constraints. For example, for holomorphic systems, the discovery process could enforce structural consistency by requiring that the derivatives of the recovered real and imaginary expressions satisfy the Cauchy-Riemann equations.

\subsubsection{Differential Equations}

Dynamical systems governed by Ordinary Differential Equations (ODEs) or Partial Differential Equations (PDEs) model the relationship between independent variables, state variables, and their respective derivatives. 
We show that these equations can be accommodated to comply with the recursive structural decomposition of Proposition~\ref{sec:prop1}.
As such, an SR framework featuring an SR model trained on the univariate skeleton dataset proposed in this work could be naturally extended to discover these dynamics. 
This can be achieved by adopting a data-augmentation strategy similar to the Sparse Identification of Nonlinear Dynamics (SINDy) paradigm~\cite{Brunton2016SINDy,PDESINDy}.

In this strategy, the relevant time-derivatives or spatial partial derivatives (e.g., $\dot{\mathbf{x}}$, $y'$, $y''$) are first numerically or analytically estimated directly from the raw experimental data. 
Once computed, they are appended to the dataset as additional independent feature columns.
The justification for this is that, once the system has been augmented with these derivative variables, they represent real-valued variable expressions operated on by the same set of operators $\mathcal{U}$ and $\mathcal{B}$. 
Because the underlying equation is built from these standard algebraic combinations, the functional blocks mapping the derivatives to the system dynamics preserve the exact recursive sum-of-products architecture required by Proposition~\ref{sec:prop1}. 
Therefore, Proposition~\ref{sec:prop1} holds over the augmented space.

Specifically, by treating the pre-computed derivatives as standalone variables, the task of learning a differential equation reduces to an SR problem that can be solved in a decomposable manner. 
Consider an underlying differential equation $\mathcal{G}(x, y, y', y'') = x (y'')^2 + y y' y'' + \cos(x) y' + 7 = 0$.
Let the derivatives be evaluated from the data and mapped to explicit variable tokens, such as $u_1 = y'$ and $u_2 = y''$, yielding the augmented input vector $\mathbf{x}_{\text{aug}} = \{x, y, u_1, u_2\}$. 
Because Proposition~\ref{sec:prop1} remains valid, the skeleton function $\kappa(\cdot)$ can isolate the variable-response relationships for each derivative track independently. 
For instance, $\mathbf{e}(u_1) = \kappa(\mathcal{G}; u_1) = c_1 u_1 + c_2,$ and $\mathbf{e}(u_2) = \kappa(\mathcal{G}; u_2) = c_3 u_2^2 + c_4 u_2 + c_5$.
The subsequent progressive merging process then recombines these isolated derivative skeletons with the state variable skeletons (e.g., those found for $x$ and $y$) to reconstruct the complete multivariate non-linear differential equation. 


\subsection{Additional Structural Filtering Rules} \label{app:rules}

To mitigate unnecessary computational overhead and prevent the expression space from expanding into functionally redundant or physically unrealistic structures, we enforce a set of additional filtering rules. 
These rules govern operator definitions, composition nesting, and algebraic simplifications for the skeleton generation process.

\subsubsection{Inverse and Power Operations}

\begin{itemize}
    \item The Inversion Operator: The \texttt{inv} operator is treated strictly as an algebraic rational operator rather than a transcendental transformation. Consequently, its presence does not contribute to the transcendental unary count $N_U$ or the maximum structural depth metric $N_D$.
    \item Power Domain: While $\texttt{Pow}(f(x))=f(x)^h$ is defined as a generic power function where $h \in \mathbb{Q}$, the operators $\sqrt{\cdot}$, $\texttt{square}$, and $\texttt{cube}$ are retained as independent primitives in $\mathcal{U}$ due to their high prevalence in physical systems. 
    To prevent structural overlap, the generic $\texttt{Pow}$ operator avoids values $h \notin \{0.5, 1, 2, 3\}$.
    \item Power Atom Cost: Because the atom $P$ explicitly involves the exponentiation of a variable, it is penalized by counting directly toward the unary operator budget $N_U$.
\end{itemize}

\subsubsection{Structural Closure and Nesting Boundaries}

\begin{itemize}
    \item Power Application to Atoms: The generic power functions cannot be applied directly to atoms $L$ or $P$. 
    This restriction is derived from the structural closure rules (Definition~\ref{def:equiv}) where $L^h \to P$ and $P^h \to P$.
    
    \item Prohibited Products with Inversion: Based on closure rules, multiplying an atom by the inverse of $L$ yields structural equivalences such as $L\, \texttt{inv}(L) \to C + C \,\texttt{inv}(L)$ and $P \,\texttt{inv}(L) \to P + C \,\texttt{inv}(L)$, already accounted for by our grammar. 
    Hence, structural configurations $L \,\texttt{inv}(L)$ and $P  \,\texttt{inv}(L)$ are forbidden.
    
    \item Additive Absorptions of $P \,\texttt{inv}(P)$: Combinations involving $P  \,\texttt{inv}(P)$ cannot be additively appended to canonical atoms. This prunes the redundant extensions $P \,\texttt{inv}(P) + L \to P  \,\texttt{inv}(P)$ and $P  \,\texttt{inv}(P) + P \to P$.

    \item Nesting and Operations of \texttt{inv}: The inversion operator cannot be nested within another inversion (e.g., $\texttt{inv}(C + C \, \texttt{inv}(\cdot))$) since such topologies compress algebraically into a single inversion layer. 
    Furthermore, the sum or product of independent inverse terms is prohibited. 
    For instance, the expression $\texttt{inv}(A) + \texttt{inv}(B)$ is blocked because its equivalent form $(A+B)\,\texttt{inv}(A\cdot B)$ is already expressible within the primary grammar.

    \item Nesting of $\texttt{Pow}$: $\texttt{Pow}$ cannot be nested within another $\texttt{Pow}$ operator or any similar power operation (i.e., $\exp$, $\sinh$, $\cosh$), and viceversa. 
    This prevents numerical explosion during evaluation, given that $\texttt{Pow}$ is restricted to exponents higher than 3. 
    In addition, compositions such as $\texttt{Pow}(P)$ are prohibited by canonical closure rules.
    
\end{itemize}

\subsubsection{Exponential Algebraic Simplifications}

\begin{itemize}
    \item Constant Exclusion: The argument of an exponential operator does not receive the general additive constant placeholder. 
    Under standard grammar expansions, a unary operator follows the rule $\nu(\tau) \to \nu(C + \Omega)$. 
    For the $\exp$ operator, this is modified to $\nu(\tau) \to \nu(\Omega)$ as an additive constant in the exponent is absorbed linearly into the outer constant placeholder via the algebraic identity: $\exp(C + f(x)) = \exp(C)\exp(f(x)) = C\exp(f(x))$.

    \item Interaction Restraints: To prevent redundant multi-term tracks, an $\exp(\cdot)$ operator cannot be multiplied by another $\exp(\cdot)$ operator, as the identity $\exp(A)\exp(B) = \exp(A+B)$ allows the architecture to represent the relationship within a single unified argument. 
    Finally, an $\exp(\cdot)$ block cannot be nested inside an $\texttt{inv}(\cdot)$ operator, as this transformation is functionally identical to mapping an exponential with a negated internal coefficient.
\end{itemize}

\section{Avoiding Invalid Operations} \label{app:invalidops}

For completeness, we detail the recursive tree-traversal pipeline for avoiding invalid operations (outlined in Algorithm~\ref{alg:avoidNaNs}), adapted from \citet{MultiSetSR} and extended with our derivative-driven singularity avoidance mechanism, as described in Sec.~\ref{sec:samplingdata}.

To systematically resolve domain violations and numerical overflows without relying on single-pass inspection, symbolic expressions are analyzed hierarchically as trees. 
Each tree node contains an operator ($f.\texttt{func}$) and a vector of arguments ($f.\texttt{args}$). 
The algorithm evaluates each argument independently (Line~\ref{alg:arg}), treating it as a nested sub-expression tree called $\texttt{arg}$. 
The function $\texttt{containSpecialF}(\text{arg})$ (Line~\ref{alg:contains}) traverses the subtree to determine whether it contains bounded or singular operators (e.g., $\ln$, $\sqrt{\cdot}$, $\arcsin$, $\tan$, or $\texttt{inv}$).
If special operators are detected, $\texttt{isSpecialF}(\text{arg}.\texttt{func})$ checks whether the top operator of the subtree itself is special (Line~\ref{alg:isspecial}). 
If the root is a standard operator (e.g., addition), the algorithm recursively descends deeper into the tree structure (Line~\ref{alg:deeper}). 
If the top operator is special, any deeper nested dependencies must be resolved first. 
Line~\ref{alg:anyspecial} checks if another special function resides within the inner argument; if so, a recursive call is triggered (Line~\ref{alg:deepfirst}) to clear deeper constraints before modifying the parent node.

To illustrate this bottom-up traversal, consider $f(x) = \arcsin\left(\log(x)\right) + x^2 + 3$. 
The top-level operator is an addition node with three arguments: $\left[\arcsin\left(\log(x)\right), x^2, 3\right]$. 
The recursive pipeline descends into the first argument, identifying $\arcsin$ as a special double-bounded operator. 
Before evaluating it, the algorithm parses its inner argument, discovering the nested single-bounded $\log(x)$ sub-tree. 
This bottom-up resolution ensures that the inner domain constraints (e.g., enforcing $x > 0$ for $\log(x)$) are satisfied and modified before the parent operator evaluates its incoming values (e.g., scaling the log values to fit within the $[-1, 1]$ domain of $\arcsin$).

Furthermore, because the special operators are always unary, the inner argument is extracted as $\text{arg}.\texttt{args}[0]$ (Line~\ref{alg:innArgext}). 
The subtree is evaluated on the current support points $\mathbf{x}$ to produce output array $\texttt{vals}$ (Line~\ref{alg:vals}). 
If $\texttt{containNaNs}(\texttt{vals})$ flags undefined outputs or numerical values exceeding $10^5$ (Line~\ref{alg:nan}), the inner function values are computed (Line~\ref{alg:innArg}) and adjusted according to the operator class:
\begin{itemize}[leftmargin=1.5em, noitemsep, topsep=2pt]
    \item \textbf{Single-Bounded Operators} (Line~\ref{alg:singleB}): Modified via $\texttt{modSBounded}$ by dynamically shifting or negating the inner domain to satisfy constraints (e.g., injecting horizontal offsets if $\sqrt{\cdot}$ receives negative inputs).
    \item \textbf{Double-Bounded Operators} (Line~\ref{alg:doubleB}): Modified via $\texttt{modDBounded}$ by linearly scaling internal values to map strictly within stable boundaries (e.g., $[-1, 1]$).
    \item \textbf{Operations with Singularities} (Line~\ref{alg:singularities}): Handled via $\texttt{getAdaptiveIntervals}$, which computes local derivative-driven exclusion windows based on the slope threshold $\alpha$ and appends them to a global list of forbidden subdomains ($\texttt{Iexcl}$).
\end{itemize}

Finally, the inner argument of the subtree node is updated with its modified, valid form $\text{innArg}$ (Line~\ref{alg:update}), and the altered arguments are recompiled back into the root expression tree (Line~\ref{alg:replace}), yielding a globally safe function alongside a set of forbidden sampling intervals.
Fig.~\ref{fig:singularity} illustrates a data generation example using $f(x) = \frac{1}{(x+4)^3} + e^{\frac{1}{x-4}}$ and $\alpha=0.9995$, demonstrating how $\texttt{getAdaptiveIntervals}$ assigns a tightly compressed exclusion window around the sharp asymptote at $x = 4$ while expanding the window to insulate the broader asymptote at $x = -4$. 
Fig.~\ref{fig:sets} depicts representative functions sampled from the same canonical skeleton, demonstrating consistent and reliable singularity avoidance across diverse parameter realizations.

\begin{figure}[!t]
    \centering
    \includegraphics[width=.85\columnwidth]{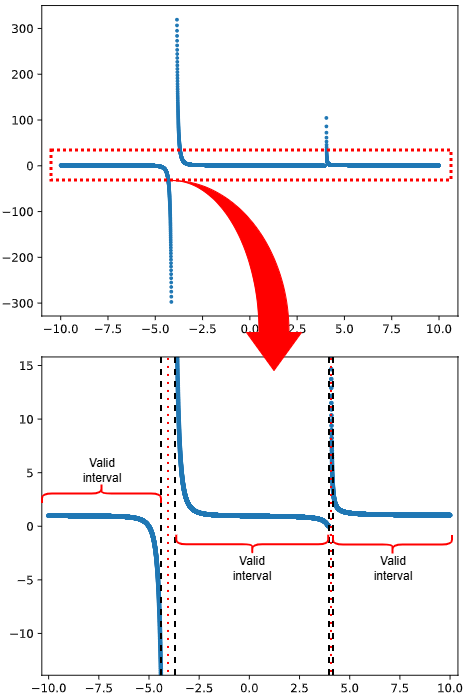}
    \caption{A generation example using $f(x) = \frac{1}{(x+4)^3}+e^{\frac{1}{x-4}}$. (Top) Generated data on the entire domain $[-10, 10]$. (Bottom) Detailed view of how singularities are avoided.}
    \label{fig:singularity}
\end{figure}

\begin{figure}[!t]
    \centering
    \includegraphics[width=1.\columnwidth]{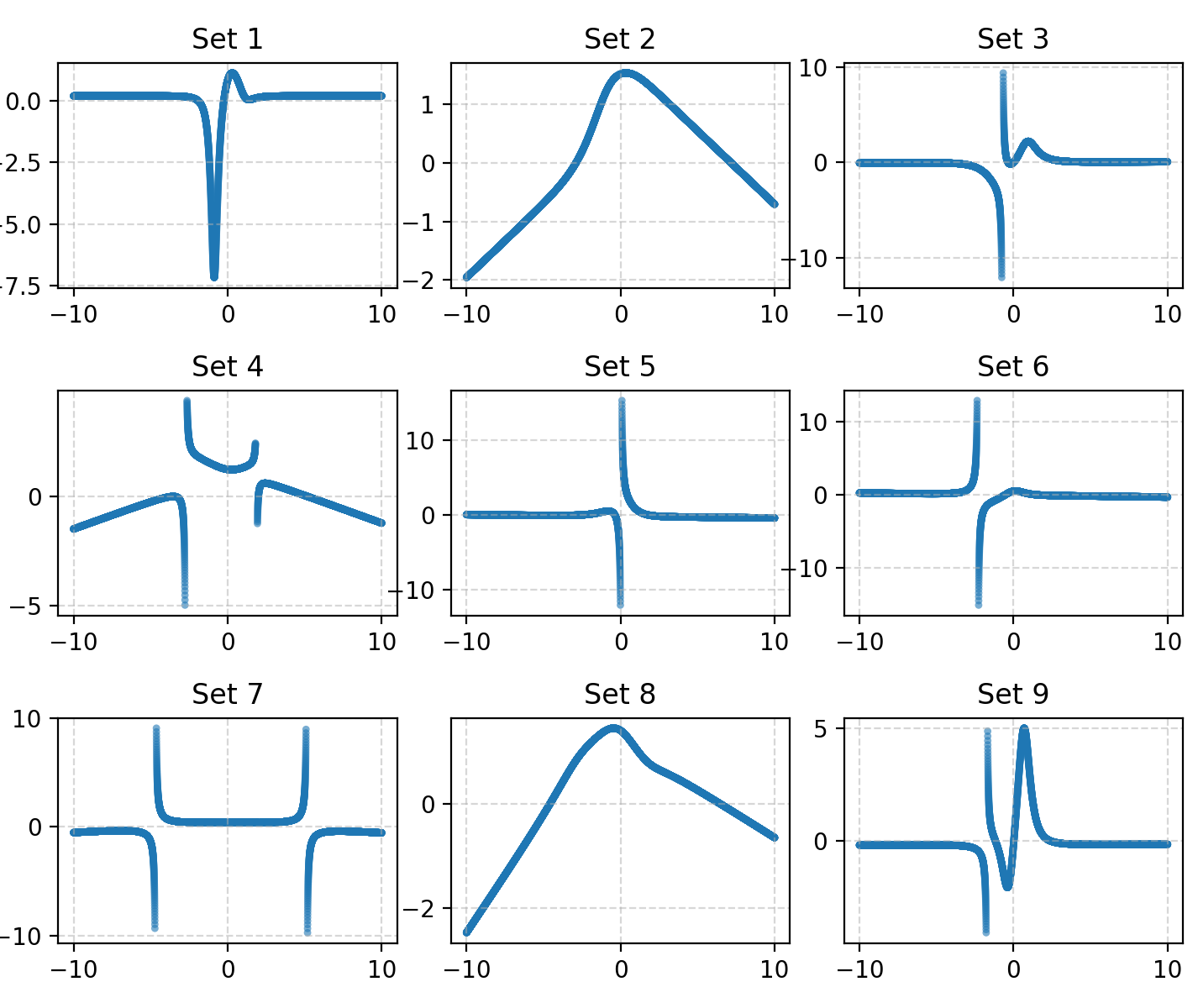}
    \caption{Various functions generated using the canonical skeleton $\mathbf{e}(x) = C + \frac{L}{C\, \arctan(P) + C} \tanh(L)$}
    \label{fig:sets}
    \vspace{-1ex}
\end{figure}

\begin{algorithm} [ht]
\scriptsize
\renewcommand{\algorithmicrequire}{\textbf{Input:}}  
\renewcommand{\algorithmicensure}{\textbf{Output:}} 
\begin{algorithmic}[1]
\Require Support points $\mathbf{x}$; function $f$; initial list of forbidden intervals \texttt{Iexcl}
\Ensure Updated $\mathbf{x}$, $f$, and accumulated \texttt{Iexcl}
    \vspace{-1.5ex}
    \Statex \makebox[0pt][l]{\hspace*{-\algorithmicindent}\rule{\dimexpr\linewidth+\algorithmicindent}{0.4pt}}
    
\Function{avoidNaNs}{$\mathbf{x}, f, \texttt{Iexcl} = [\,]$}
    \State $\text{args} \leftarrow f.\texttt{args}$ 
    \State $\text{newArgs} \leftarrow [\,]$  
    \For { each $\text{arg} \in \text{args}$}  \Comment{Analyze each sub-tree (argument) of $f$} \label{alg:arg}
            \If {$\texttt{containSpecialF}(\text{arg})$} \label{alg:contains}            
                    \If {$! \texttt{isSpecialF}(\text{arg}.\texttt{func})$} \label{alg:isspecial}
                            \State $\mathbf{x}, \text{arg}, \texttt{Iexcl} \leftarrow \texttt{avoidNaNs}(\mathbf{x}, \text{arg}, \texttt{Iexcl})$ \label{alg:deeper}
                    \Else 

                    \Statex \hspace*{5.8em} \# Is there another special func. inside the current arg function?
                            \If {$\texttt{containSpecialF}(\text{arg}.\texttt{args}[0])$}  \label{alg:anyspecial}
                                    \State $\mathbf{x}, \text{arg}, \texttt{Iexcl}  \leftarrow \texttt{avoidNaNs}(\mathbf{x}, \text{arg}, \texttt{Iexcl} )$  \label{alg:deepfirst}                                       
                            \EndIf
                            \State $\text{innArg} \leftarrow \text{arg}.\texttt{args}[0]$ \label{alg:innArgext}  
                            \State $\text{vals} \leftarrow \texttt{\text{arg}}(\mathbf{x})$  \label{alg:vals}  
                            \If{$\texttt{containNaNs}(\text{vals})$} \label{alg:nan}
                                    \State $\text{innVals} \leftarrow \texttt{innArg}(\mathbf{x})$ \label{alg:innArg}
                                    \If{$\texttt{isSingleBounded}(\text{arg})$}
                                            \State $\text{innArg} \leftarrow \texttt{modSBounded}(\text{arg}, \text{vals}, \text{innArg}, \text{innVals})$ \label{alg:singleB}
                                    \ElsIf{$\texttt{isDoubleBounded}(\text{arg})$}
                                            \State $\text{innArg} \leftarrow \texttt{modDBounded}(\text{arg}, \text{innArg}, \text{innVals})$ \label{alg:doubleB}
                                    \Else  \; \Comment{Operations with singularities}
                                            \State $\texttt{Iexcl} \leftarrow \texttt{getAdaptiveIntervals}(\text{arg}, \texttt{Iexcl})$ \label{alg:singularities}
                                    \EndIf
                                \EndIf
                            \State $\text{arg}.\texttt{args}[0] \leftarrow \text{innArg}$ \Comment{Update function} \label{alg:update}
                    \EndIf
            \EndIf
            \State $\text{newArgs}.\texttt{append}(\text{arg})$
    \EndFor
    \State $f.\texttt{args} \leftarrow \text{newArgs}$ \label{alg:replace}
    \State \Return $\mathbf{x}, f, \texttt{Iexcl}$
\EndFunction
\end{algorithmic}
\caption{Avoiding Invalid Operations}
\label{alg:avoidNaNs}
\end{algorithm}

\section{Benchmark Datasets} \label{app:benchmark}

This appendix provides the ExCaUS-Bench equations and the corresponding sampled datasets used to evaluate the SR methods listed in Sec.~\ref{sec:resbench}.
The benchmark is organized into three difficulty tiers: Easy (E1--E50), Medium (M1--M50), and Hard (H1--H50), with the generated equations reported in Tables~\ref{tab:eq_e}--\ref{tab:eq_h}. 
The corresponding plots of the sampled datasets are provided in Figs.~\ref{fig:eq_e1}--\ref{fig:eq_h3}, with each figure displaying the benchmark equations within its respective tier.

\begin{table*}[htbp]
\centering
\caption{Generated Equations for the E tier}
\label{tab:eq_e}
\vspace{-2ex}
\scriptsize
\resizebox{\textwidth}{!}{%
\begin{tabular}{cll}
\hline\hline
\textbf{ID} & \textbf{Skeleton} & \textbf{Equation} \\
\hline
E1 & $C |{L}| + L$ & $-1.31 + 5.79\text{E-2} |{4.54 x + 4.61}| + 7.67\text{E-2} x$ \\
E2 & $C e^{L} \sinh(P ) + C$ & $-0.30 + 2.17\text{E-2} e^{0.16 x} \sinh(-0.48 x + 0.11 + 3.11\text{E-2} x^{2} )$ \\
E3 & $C \operatorname{inv}(L ) \cosh(P ) + C \operatorname{asinh}(L ) + C$ & $-4.97\text{E-2} + 0.13 \operatorname{asinh}(x ) + \frac{5.06\text{E-2} \cosh(-5.54\text{E-3} + 3.05\text{E-2} x + 7.70\text{E-2} x^{2} )}{1.35 x + 1.63}$ \\
E4 & $C \operatorname{inv}(C \sin(L ) \operatorname{atan}(P ) + L ) + C$ & $\frac{-1.26}{0.85 x + 1.40 \sin(2.91 x ) \operatorname{atan}(-3.21 + 0.32 x^{3} + 0.64 x - x^{2} )} + -4.89\text{E-2}$ \\
E5 & $C + L \log(P )$ & $-0.42 + (0.13 + 3.54\text{E-2} x) \log(-2.06 x + 0.45 + 2.49 x^{2} )$ \\
E6 & $C \operatorname{Pow}(C + L \operatorname{asinh}(L ) ) \tanh(P ) + C$ & $\begin{aligned}[t] & -0.14 + -1.02\text{E-13} ((-3.12 + -4.11 x) \operatorname{asinh}(-0.65 + x ))^{\frac{13}{2}} \tanh(-8.37\text{E-3} x + 1.85\text{E-3} \\ & \quad + 4.91\text{E-3} x^{2} + 7.59\text{E-3} x^{3} ) \end{aligned}$ \\
E7 & $C + L \sin(C \sin(P ) + C )$ & $-6.32\text{E-2} + (-0.22 x + 0.20) \sin(-1.21 + 3.48 \sin(-8.48\text{E-2} x^{3} + 0.18 x^{2} + 0.90 + 0.98 x ) )$ \\
E8 & $C \tan(C \operatorname{inv}(L ) + C |{P}| + C ) + L$ & $-5.99\text{E-2} + 0.22 x + 0.49 \tan(\frac{-0.10}{x} + 1.26\text{E-2} |{-1.60 x + 0.48 x^{3} + 1.63 x^{2} + 4.55 + x^{4}}| )$ \\
E9 & $C + L \sin(C \operatorname{inv}(P ) + C ) \tanh(L )$ & $\begin{aligned}[t]  -7.79 \,x &\sin(\frac{0.32}{-0.16 x^{6} + -1.85 + -3.81 x^{3} + 2.17 x^{2} + 3.02 x + 4.60\text{E-2} x^{5} + x^{4}} ) \\ & \tanh(0.26 + 0  .77 x ) + 4.45\text{E-2} \end{aligned}$ \\
E10 & $L \tan(P ) + L$ & $-2.89\text{E-3} + -9.30\text{E-2} x - (-0.20 + 0.41 x) \tan(-0.36 x + 1.48 x^{2} )$ \\
E11 & $L |{P}| + L$ & $-0.44 + -3.80\text{E-3} x + (1.18\text{E-2} + 6.18\text{E-3} x) |{-3.64 + 0.24 x^{2} + 2.97 x}|$ \\
E12 & $C \operatorname{inv}(L ) + P$ & $-1.15\text{E-2} x + -5.67\text{E-3} x^{2} + \frac{-9.98\text{E-3}}{x} + 0.19 + 2.75\text{E-3} x^{3}$ \\
E13 & $C e^{C \operatorname{inv}(P )} \operatorname{asinh}(L ) + C$ & $0.11 + 0.45 e^{\frac{1.24\text{E-2}}{-2.77 x^{3} + 4.94 x + x^{2}}} \operatorname{asinh}(-0.85 + x )$ \\
E14 & $L + P |{L}|$ & $-0.80 + 1.79\text{E-3} x + (-3.26\text{E-3} + 1.48\text{E-3} x^{2} + 3.61\text{E-3} x) |{-4.41 + 2.53 x}|$ \\
E15 & $L \sin(L ) + P$ & $-1.08 + 1.48\text{E-2} x + 3.23\text{E-2} x^{2} - (-3.91\text{E-2} + 4.52\text{E-2} x) \sin(-2.97 + 4.60 x )$ \\
E16 & $C \sin(L ) + L$ & $-8.01\text{E-3} + 0.17 x + 0.17 \sin(2.83 x )$ \\
E17 & $C \cosh(P ) + C + L \operatorname{atanh}(L )$ & $\begin{aligned}[t] & -0.26 + 5.85\text{E-3} \cosh(-6.99\text{E-3} + 7.15\text{E-2} x^{2} + 8.65\text{E-2} x ) + (-2.20\text{E-2} x + 2.22\text{E-2}) \operatorname{atanh} \\ & \quad (2.68\text{E-4} + 9.99\text{E-2} x ) \end{aligned}$ \\
E18 & $C \operatorname{inv}(P ) + C + L \log(C \tan(L ) + C )$ & $\begin{aligned}[t] & \frac{-7.89\text{E-2}}{2.53 x^{2} + 3.30 x^{3} + x^{4} + x} + 0.27 + (-0.13 + -7.25\text{E-2} x) \log(-1.83 \tan(2.72 + x ) + 10.24 ) \end{aligned}$ \\
E19 & $C \sinh(C \sqrt{L} \operatorname{inv}(P ) + C ) + C$ & $0.65 + 3.76 \sinh(-0.32 + \frac{1.11 \sqrt{0.68 x + 6.89}}{2.66 + 3.42 x^{2} + x} )$ \\
E20 & $C \log(L ) + C$ & $-1.07 \log(-1.17 x + 11.78 ) + 2.32$ \\
E21 & $C \sin(L ) + P$ & $\begin{aligned}[t] & -0.58 + -4.49\text{E-6} x^{2} + -6.16\text{E-6} x^{5} + 1.08\text{E-6} x^{4} + 1.25\text{E-6} x + 2.29\text{E-6} x^{3} + 4.13\text{E-6} \\ & \quad x^{6} + 4.40\text{E-6} \sin(0.10 + 4.70 x ) \end{aligned}$ \\
E22 & $C \sin(L ) + C$ & $-1.40 \sin(1.23 x ) + 5.52\text{E-4}$ \\
E23 & $C \tan(L ) + P$ & $-2.34\text{E-2} \tan(2.21 x ) + -2.83\text{E-2} x^{2} + -9.12\text{E-2} x + 0.95$ \\
E24 & $C + P \operatorname{inv}(C \operatorname{atan}(L ) + C )$ & $-1.12 + \frac{-3.51 + 0.14 x^{2} + 1.60 x}{-0.66 + 4.85 \operatorname{atan}(3.40 x )}$ \\
E25 & $C \tanh(C e^{P} + C ) + C$ & $-4.13 \tanh(-1.67\text{E-4} + 3.63\text{E-4} e^{0.22 x + 7.78\text{E-2} x^{2}} ) + 0.38$ \\
E26 & $P$ & $\begin{aligned}[t] & -1.19\text{E-6} x + -2.01\text{E-6} x^{2} + -2.25\text{E-6} x^{5} + -4.21\text{E-6} x^{6} + 0.60 + 4.02\text{E-6} x^{4} + 4.50\text{E-6} x^{3} \end{aligned}$ \\
E27 & $C \sinh(L ) \operatorname{atan}(C \operatorname{inv}(L ) + C ) + P$ & $\begin{aligned}[t]  -0.30 + 2.44\text{E-3} x^{3} + 2.92\text{E-3} x + 3.85\text{E-3} x^{2} + &6.15\text{E-3} \sinh(0.12 + 0.79 x ) \\ & \operatorname{atan} (\frac{4.31}{3.15 x + 4.41} ) \end{aligned}$ \\
E28 & $C \sinh(C + P \operatorname{asin}(L ) ) + C$ & $\begin{aligned}[t]  3.06\text{E-2} + 8.04\text{E-3} \sinh(-1.51\text{E-2} + (-8.07\text{E-3} x + 5.16\text{E-2} x^{2}) \operatorname{asin}(&-1.00\text{E-4}  + 9.99\text{E-2} x ) ) \end{aligned}$ \\
E29 & $C + P e^{L} \operatorname{asin}(L )$ & $\begin{aligned}[t]  0.48 + (-1.30\text{E-3} x^{3} + -2.14\text{E-3} x + -2.44\text{E-3} + -6.14\text{E-4} x^{2} + &3.74\text{E-5} x^{4}) e^{-8.80\text{E-2} x} \\ & \operatorname{asin}(-4.45\text{E-5} + 9.99\text{E-2} x ) \end{aligned}$ \\
E30 & $L + P e^{L}$ & $\begin{aligned}[t] & -0.26 + 2.67\text{E-9} x + (1.73\text{E-9} x^{5} + 1.73\text{E-9} x^{4} + 4.24\text{E-9} x + 5.18\text{E-10} x^{6} + 5.76\text{E-9} \\ & \quad + 6.55\text{E-9} x^{2} + 8.44\text{E-9} x^{3}) e^{0.93 x} \end{aligned}$ \\
E31 & $C \operatorname{Pow}(C \sin(L ) + C ) \operatorname{atanh}(P ) + L$ & $\begin{aligned}[t] & -0.13 (-0.21 + 2.31 \sin(-0.60 + 3.55 x ))^{\frac{13}{5}} \operatorname{atanh}(-1 + 1.59\text{E-2} x + 1.84\text{E-2} x^{2} ) + 0.15 + 6.13\text{E-2} x \end{aligned}$ \\
E32 & $C + P \tan(L )$ & $2.47\text{E-2} + (-2.25\text{E-3} + -2.60\text{E-3} x^{3} + 5.97\text{E-4} x^{2} + 5.97\text{E-4} x) \tan(-0.42 + 4.88 x )$ \\
E33 & $C \sqrt{P} + C + L \sinh(L )$ & $\begin{aligned}[t] & -0.52 + -4.96\text{E-5} \sqrt{-7.68\text{E-4} x^{4} + 1.93\text{E-2} x^{2} + 10.12 + 4.20\text{E-3} x^{3} + 7.32\text{E-3} x} + ( \\ & \quad 1.65\text{E-4} + 3.28\text{E-4} x) \sinh(0.80 x ) \end{aligned}$ \\
E34 & $L$ & $-0.17 x + -2.07\text{E-2}$ \\
E35 & $C |{P}| + L$ & $-5.51\text{E-3} |{-3.33 x^{2} + -4.10 + 0.35 x^{3} + 3.99 x}| + 0.65 + 5.01\text{E-3} x$ \\
E36 & $C \log(L ) + P$ & $-0.19 x + -0.23 \log(-1.22 x + 12.29 ) + 0.11 + 1.20\text{E-2} x^{2}$ \\
E37 & $C \sqrt{L} \operatorname{inv}(C + L \cosh(P ) ) + C$ & $-2.15\text{E-3} + \frac{0.64 \sqrt{-1.13 x + 11.43}}{(-1.25 + -1.96 x) \cosh(-1.53\text{E-2} x^{2} + 1.78\text{E-2} x + 6.29\text{E-3} x^{3} )}$ \\
E38 & $C \operatorname{Pow}(C \tanh(L ) \operatorname{acosh}(L ) + C ) + P$ & $\begin{aligned}[t]  -0.90 + -1.87\text{E-4} x + 1.04\text{E-4} x^{2} + 1.05\text{E-4} (0.92 + &\tanh(0.80 x ) \operatorname{acosh}(0.94 x + 10.53 ))^{6} + \\ &   2.86\text{E-4} x^{3} + 2.87\text{E-4} x^{4} \end{aligned}$ \\
E39 & $C \operatorname{inv}(P ) + C$ & $-0.33 + \frac{-7.63}{-0.15 x + -4.08 + -4.12 x^{3} + -7.40\text{E-2} x^{2} + 2.74 x^{5} + x^{4}}$ \\
E40 & $C \operatorname{inv}(C \log(P ) + L ) \cosh(L ) + C$ & $0.29 + \frac{4.59\text{E-2} \cosh(-3.12\text{E-3} + 0.80 x )}{-2.53 x + -9.60\text{E-2} + 2.04 \log(0.19 + 2.82 x^{2} - x )}$ \\
E41 & $L e^{L} + P$ & $\begin{aligned}[t] & -1.68\text{E-5} x^{3} + 0.41 + 2.90\text{E-5} x^{4} + 2.90\text{E-5} x + 3.08\text{E-5} x^{5} + 7.94\text{E-6} x^{2} + (-4.92\text{E-5} x + 1.49\text{E-5}) e^{x} \end{aligned}$ \\
E42 & $C + L e^{P}$ & $-0.51 + (-1.49\text{E-2} + 8.71\text{E-3} x) e^{-0.29 x^{2} + 2.41 x}$ \\
E43 & $C \log(P ) + L$ & $-0.99 + 0.15 x + 0.15 \log(-1.10 x + -3.63 x^{2} + 3.69 x^{4} + 4.72 + 4.76 x^{3} )$ \\
E44 & $C \operatorname{acosh}(C \operatorname{atan}(L ) + P ) + L$ & $-0.89 + 0.16 x + 0.19 \operatorname{acosh}(-2.46 \operatorname{atan}(2.97 x + 4.63 ) + 0.82 x + 2.81 x^{2} + 4.46 )$ \\
E45 & $C \sin(P ) + L$ & $-2.39\text{E-3} + 0.17 x + 4.90\text{E-2} \sin(0.11 x + 7.40\text{E-2} x^{3} + 9.36\text{E-3} x^{2} )$ \\
E46 & $C + L \log(L )$ & $-0.69 + (0.16 + 7.89\text{E-2} x) \log(10.10 + x )$ \\
E47 & $C \operatorname{Pow}(C \tanh(P ) + L ) \cosh(L ) + C$ & $\begin{aligned}[t] & -0.33 + 3.20\text{E-6} (-0.33 + -0.48 x + -1.16 \tanh(3.12\text{E-2} + 5.60\text{E-2} x + 7.41\text{E-2} x^{2} ))^{4} \cosh(0.80x ) \end{aligned}$ \\
E48 & $C e^{L} + C$ & $-0.39 + 2.94\text{E-3} e^{0.76 x}$ \\
E49 & $C \cosh(C + L \operatorname{inv}(P ) ) + L$ & $-1.13 + 0.17 x + 0.78 \cosh(-0.90 + \frac{-0.27 + -0.93 x}{-0.29 x^{3} + 1.98 + 4.67 x^{2} + x} )$ \\
E50 & $C \operatorname{inv}(C + L \tanh(L ) ) + P$ & $-1.12 + 3.15\text{E-2} x + 3.29\text{E-2} x^{2} + \frac{3.94\text{E-2}}{-3.42 + (1.02 + 3.23 x) \tanh(-0.90 + 0.53 x )}$ \\
\hline\hline
\end{tabular}%
}
\end{table*}

\begin{table*}[htbp]
\centering
\caption{Generated Equations for the M tier}
\label{tab:eq_m}
\vspace{-2ex}
\scriptsize
\resizebox{\textwidth}{!}{%
\begin{tabular}{@{}cll@{}}
\hline\hline
\textbf{ID} & \textbf{Skeleton} & \textbf{Equation} \\
\hline
M1 & $C \operatorname{inv}(C |{L}| + L ) + C \sin(P ) + C$ & $-2.51\text{E-2} + 0.78 \sin(0.60 + 0.98 x + 1.40 x^{2} ) + \frac{0.78}{-4.23 x + 4.05 + |{-1.90 + 3.63 x}|}$ \\
M2 & $C \operatorname{inv}(C \log(P ) + C ) + C |{L}| + C$ & $\begin{aligned}[t] & -1.66 + \frac{0.31}{3.09 + 4.66 \log(-0.30 x^{5} + -4.93 x^{2} + 0.59 x + 0.93 x^{4} + 1.87 + 4.33 x^{6} - x^{3} )} + 0.32 |{1.92 + x}| \end{aligned}$ \\
M3 & $\begin{aligned}[t] & C \operatorname{cube}(C \operatorname{inv}(L ) \operatorname{asin}(P ) + C ) + C \operatorname{inv}( L ) + C \end{aligned}$ & $\begin{aligned}[t] & -0.42 + 1.26 (\frac{-1.51 \operatorname{asin}(-2.00\text{E-3} x^{2} + 0.20 + 4.10\text{E-3} x + 9.58\text{E-4} x^{3} )}{-2.70 x + -2.81}  + 0.66)^{3} + \frac{3.09}{-2.95 + 3.17 x} \end{aligned}$ \\
M4 & $C \operatorname{Pow}(C + L \operatorname{asinh}(L ) ) + C \cosh(P ) + C$ & $\begin{aligned}[t] & -0.45 + 1.54\text{E-5} (-4.16 + (-3.34 + 0.58 x) \operatorname{asinh}(-0.65 + x ))^{4} + 1.54\text{E-5}  \cosh(4.61\text{E-2} x + 6.66\text{E-2} + 7.47\text{E-2} x^{2} ) \end{aligned}$ \\

M5 & $C \operatorname{acosh}(L ) \operatorname{acosh}(C \operatorname{inv}(P ) + L ) + C$ & $\begin{aligned}[t] & 
-2.23 + 0.43 \operatorname{acosh}{\left(-0.97 x + 10.75 \right)} \operatorname{acosh}{\left(-0.21 x + 3.19 - \frac{1}{-1.94 x^{2} + -4.34 x^{5} + 2.77 + 4.17 x^{4} + x^{3} + x} \right)}
\end{aligned}$ \\
M6 & $C \sqrt{C \operatorname{inv}(L ) + C} \sinh(C + L \operatorname{inv}(P ) ) + C$ & $\begin{aligned}[t] & 1.35 \sqrt{10.11 + \frac{7.08\text{E-4}}{-2.28 + -3.14 x}} \sinh(-7.20\text{E-2} + \frac{0.39 + 2.98 x}{-3.16 x + 3.17 + 4.67 x^{2}}  ) + 5.44\text{E-2} \end{aligned}$ \\
M7 & $C \log(L ) + C + L \sinh(C \operatorname{inv}(P ) + C )$ & $\begin{aligned}[t] & -0.19 x \sinh(0.75 + \frac{2.13}{-1.83 x^{2} + -2.04 + -3.27 x^{4} + -4.71 x^{3} + 0.63 x} ) + -0.25 + 0.13 \log(-1.10 x + 11.06 ) \end{aligned}$ \\
M8 & $C \sqrt{C \operatorname{inv}(L ) + C} \operatorname{asin}(C \operatorname{inv}(L ) + C ) + C$ & $\begin{aligned}[t] & 
-5.05 \sqrt{\frac{-1.34}{-2.03 + -3.19 x} + 4.32} \operatorname{asin}{\left(\frac{-0.35}{-3.80 x + -4.49} + 2.23\text{E-3} \right)} + 1.63\text{E-2}
\end{aligned}$ \\
M9 & $C \operatorname{Pow}(C \operatorname{atanh}(L ) + C ) + C + L \operatorname{atan}(P )$ & $\begin{aligned}[t] & 
-0.60 + 7.07\text{E-2} \operatorname{atanh}^{4}{(1.09\text{E-4} + 9.99\text{E-2} x )} + (-0.19 + 6.17\text{E-2} x) \operatorname{atan}{(-0.47 + -4.25 x + x^{3} + x^{2} )}
\end{aligned}$ \\
M10 & $C \tan(L ) + C \tan(C \operatorname{inv}(P ) + C ) + C$ & $\begin{aligned}[t] & 0.83 \tan(2.93 x ) + 1.23 \tan(\frac{1.01}{-2.06 + -2.92 x^{3} + -4.10 x^{5} + 0.17 x^{2} + 3.53 x + 4.44 x^{4}} ) + 4.97\text{E-2} \end{aligned}$ \\
M11 & $C \sqrt{P} \operatorname{inv}(L ) e^{C \operatorname{inv}(L )} + C$ & $\frac{-0.35 \sqrt{0.54 + 1.31 x + x^{2}} e^{\frac{1}{-2.86 x + -4.36}}}{3.67 + x} + -8.54\text{E-2}$ \\
M12 & $C \operatorname{inv}(L ) + C \tan(C \tan(L ) + C ) + C$ & $\begin{aligned}[t] & \frac{-3.99}{3.20\text{E-2} + 4.96 x} + -6.03 \tan(1.52\text{E-2} \tan(-3.14 + 2.51 x ) + 4.52\text{E-3}  ) + 2.71\text{E-2} \end{aligned}$ \\
M13 & $C \operatorname{Pow}(C \operatorname{acosh}(P ) + L ) + C \sin(L ) + C$ & $\begin{aligned}[t]  -0.11 \sin(-1.49 + 3.59 x ) + -0.56 + 2.42\text{E-2} (-0.94 \operatorname{acosh}(&-1.88\text{E-2} x + \\ &  -7.23\text{E-3} x^{2} + -9.09\text{E-3} x^{3} + 11.10 ) + 1.62 + 5.90\text{E-2} x)^{6} \end{aligned}$ \\
M14 & $C |{C \operatorname{inv}(L ) + C \sin(L ) + C}| + P$ & $\begin{aligned}[t] & -0.66 + -8.19\text{E-2} |{\frac{2.38}{0.49 x + 4.02} + 3.48 - \sin(3.69 x )}| + 3.13\text{E-2} x^{2} + 4.26\text{E-2} x \end{aligned}$ \\
M15 & $C |{P}| + C + L \operatorname{inv}(C \log(L ) + C )$ & $-1.46\text{E-2} |{-2.28 x^{2} + 2.48 x + 4.55}| + 1.03 + \frac{8.24\text{E-3} x}{\log(10.10 - x )}$ \\
M16 & $C \sin(C \log(L ) + C + L \operatorname{inv}(P ) ) + C$ & $\begin{aligned}[t]  0.20 + 1.40 \sin(1.62 + &3.12 \log(-1.06 x + 10.66 ) + \frac{x}{-1.41 x^{6} + -1.70 x + -3.62 x^{5} + 1.87 x^{3} + 2.66 x^{2} + 4.54 + 4.94 x^{4}}  ) \end{aligned}$ \\
M17 & $C \operatorname{inv}(L ) + C + L \log(C \operatorname{acosh}(P ) + C )$ & $\begin{aligned}[t] & -0.12 + \frac{-2.59\text{E-2}}{-3.03 + x} + (-0.10 x + 7.44\text{E-2}) \log(-0.39 + 1.10 \operatorname{acosh}  (-1.30 x + 1.31 + 2 x^{2} ) ) \end{aligned}$ \\
M18 & $C \operatorname{inv}(C \tan(L ) + C ) \operatorname{Pow}(C \tanh(L ) + C ) + C$ & $1.21 + \frac{5.07 \tanh^{4}{(0.80 x )}}{-3.19 + \tan(1.27 + 4.67 x )}$ \\
M19 & $C \operatorname{inv}(C |{L}| + P ) \operatorname{Pow}(C |{L}| + C ) + C$ & $-0.19 + \frac{0.23 (-2.80 + 1.31 |{x}|)^{4}}{-0.82 x^{2} + -4.24 x + -4.85 x^{4} + 0.32 x^{3} + 1.90 |{x}| + 2.53}$ \\
M20 & $C \operatorname{asinh}(L ) \operatorname{asinh}(C \operatorname{inv}(L ) + P ) + C$ & $\begin{aligned}[t] & -0.11 \operatorname{asinh}(0.48 + 1.42 x ) \operatorname{asinh}(-1.72 x^{3} + \frac{-1.80}{-1.70 x + -2.41}  + 0.92 + 1.65 x^{2} + 1.65 x ) + -1.61 \end{aligned}$ \\
M21 & $C \sin(P ) + C + L \operatorname{inv}(C |{L}| + C )$ & $\begin{aligned}[t] & -0.37 \sin(-7.52\text{E-2} x^{3} + 1.55\text{E-2} x + 2.70\text{E-2} + 6.59\text{E-2} x^{2} ) + 0.40  + \frac{-0.26 + 0.42 x}{|{2.28 + x}|} \end{aligned}$ \\
M22 & $C \log(C \operatorname{inv}(L ) + C + P e^{L} ) + C$ & $\begin{aligned}[t]  -0.14 \log(\frac{-1.53}{-0.17 x + -4.35} + 4.38 + (&-0.48 x^{3} + -0.48 x^{2} + -0.77 x^{5} + -1.54 + \\ &  -1.80 x^{4} + -6.79\text{E-2} x + 0.80 x^{6}) e^{-1 x} ) + 1.14 \end{aligned}$ \\
M23 & $C \operatorname{inv}(L ) + C \sin(C \sin(P ) + L ) + C$ & $\begin{aligned}[t] & \frac{-0.30}{2.35 + x} + 1.47 \sin(-1.63 + 1.65 x + \sin(3.33\text{E-3} x^{2} + 5.53\text{E-2} + 7.2 9\text{E-2} x + 7.43\text{E-2} x^{3} ) ) + 3.16\text{E-2} \end{aligned}$ \\
M24 & $C |{C \operatorname{asin}(L ) + C}| + C + L \operatorname{inv}(P )$ & $\begin{aligned}[t] & -0.49 |{2.94 \operatorname{asin}(9.99\text{E-2} x ) + 4.70}| + \frac{-0.56 x}{-3.75 x^{5} + -4.68 + 2.93 x^{4} + 5 x^{3} + x^{2} + x}  + 2.32 \end{aligned}$ \\
M25 & $C \operatorname{inv}(L ) + C \log(C \operatorname{asin}(P ) + L ) + C$ & $\begin{aligned}[t] & \frac{-0.24}{-0.88 + -4.62 x} + -1.09 \log(-0.13 \operatorname{asin}(-0.99 + 1.72\text{E-2} x + 1.82  \text{E-2} x^{2} ) + 0.99 x + 10.10 ) + 2.23 \end{aligned}$ \\
M26 & $C \operatorname{Pow}(C \operatorname{atanh}(L ) + L ) \cosh(P ) + C$ & $\begin{aligned}[t] & -1.90\text{E-9} (-2.18 + -3.45 x + 3.55 \operatorname{atanh}(1.38\text{E-4} + 9.99\text{E-2} x ))^{\frac{14}{3}} \cosh(-7.98\text{E-2} x + 7.13\text{E-2} x^{2} + 7.92\text{E-2} ) + 0.12 \end{aligned}$ \\
M27 & $C \operatorname{inv}(C |{L}| + C ) \operatorname{Pow}(C \log(P ) + C ) + C$ & $-1.69 + \frac{6.10\text{E-4} (-2.26 + -2.48 \log(-0.85 x + 2.54 + x^{2} ))^{5}}{-2.20 |{3.16 + 4.45 x}| + 8.80\text{E-2}}$ \\
M28 & $C \operatorname{inv}(L ) + C \operatorname{acosh}(C + P \sin(L ) ) + C$ & $\begin{aligned}[t] & -2.63 \operatorname{acosh}(11.10 + (-1.48\text{E-2} x^{2} + -9.43\text{E-3} x^{3} + 3.05\text{E-2} x ) \sin(4.22 x + 4.22 ) ) + \frac{0.61}{x} + 8.05 \end{aligned}$ \\
M29 & $C \operatorname{inv}(C \operatorname{acosh}(L ) \operatorname{asinh}(P ) + L ) + C$ & $\begin{aligned}[t] & \frac{-3.58}{x - \operatorname{acosh}(0.86 x + 9.69 ) \operatorname{asinh}(-0.89 + -2.19 x^{2} + 1.41 x + 2.26 x^{3} )}  + 0.17 \end{aligned}$ \\
M30 & $C \operatorname{asinh}(L ) \operatorname{asinh}(C + L \operatorname{inv}(P ) ) + C$ & $\begin{aligned}[t] & -0.96 \operatorname{asinh}(-0.44 + \frac{-0.71 + -1.48 x}{-3.02 x^{4} + -4.88 x^{5} + 1.74 + 2.09 x^{3} + 3.58 x^{2} + x}  ) \operatorname{asinh}(-0.65 + 1.27 x ) + 6.87\text{E-2} \end{aligned}$ \\
M31 & $C \operatorname{inv}(L ) \operatorname{atanh}(C \operatorname{inv}(P ) + C \tan(L ) + C ) + C$ & $-1.95\text{E-2} + \frac{-6.68 \operatorname{atanh}(-0.19 + 3.36\text{E-4} \tan(2.16 + 3.33 x ) + \frac{8.17\text{E-4}}{-0.29 x^{2} + -1.07 x^{3} + 2.38 + x^{5} + x^{4} + x} )}{1.63 + x}$ \\
M32 & $C \sin(P ) + C \tan(C \operatorname{inv}(L ) + L ) + C$ & $-0.41 \tan(\frac{-3.84}{-4.26 x + 0.17} + x ) + -1.02 \sin(-1.49 x^{2} + 0.10 x + 1.86 ) + 6.24\text{E-2}$ \\
M33 & $C + L \tan(C \operatorname{inv}(L ) + C \operatorname{asinh}(P ) + C )$ & $\begin{aligned}[t] & 0.26 + (-0.18 + -8.60\text{E-2} x) \tan(\frac{-0.69}{-0.21 x + 2.51} + 2.54 \operatorname{asinh}(-0.12 + 1.92 x^{2} + 2.49 x ) ) \end{aligned}$ \\
M34 & $C \operatorname{inv}(P ) e^{C \operatorname{inv}(L )} \operatorname{atan}(L ) + C$ & $-0.35 + \frac{-2.81 e^{\frac{-2.50\text{E-4}}{-0.58 + 1.08 x}} \operatorname{atan}(0.80 + 2.50 x )}{-1.02 x + -1.19 + -3.05 x^{3} + 0.72 x^{4} + 3.46 x^{2}}$ \\
M35 & $C \operatorname{Pow}(C \operatorname{asin}(L ) \operatorname{atanh}(P ) + L ) + C$ & $\begin{aligned}[t] & -1.48\text{E-5} (2.29 \operatorname{asin}(5.50\text{E-5} + 9.99\text{E-2} x ) \operatorname{atanh}(- 1.86\text{E-2} x^{2} + 1 + 1.36\text{E-2} x ) + 2.46 x)^{4} + 0.89 \end{aligned}$ \\
M36 & $C \sin(C \operatorname{inv}(L ) + C ) + C + P e^{L}$ & $\begin{aligned}[t] & -0.39 + -1.68\text{E-4} \sin(\frac{-0.11}{1.13 x + 3.11} + 1.60\text{E-2} ) + (-3.78\text{E-4} x + 3.11\text{E-4} x^{2} + 5.93\text{E-4}) e^{-0.51 x} \end{aligned}$ \\
M37 & $C \operatorname{inv}(L ) \cosh(C \operatorname{inv}(L ) + C ) + C |{P}| + C$ & $-5.85\text{E-3} |{3.38 x^{2} + 4.78 - x}| + 0.70 + \frac{3.12\text{E-2} \cosh(\frac{0.96}{x} + 2.80\text{E-2} )}{x}$ \\
M38 & $C \operatorname{Pow}(C |{L}| + C ) \tanh(C + L \operatorname{inv}(P ) ) + C$ & $\begin{aligned}[t] & 
-0.12 + 6.20\text{E-5} (0.14 + 9.79 \left|{x}\right|)^{5} \tanh{\left(\frac{1.46\text{E-2} x + 4.38\text{E-3}}{-0.83 + -2.89 x^{3} + 0.27 x^{2} + 1.41 x + 3.72 x^{4}} \right)}
\end{aligned}$ \\
M39 & $C \log(C \operatorname{inv}(L ) + C ) + C + P \tan(L )$ & $\begin{aligned}[t] & -4.15\text{E-3} + 2.21\text{E-3} \log(\frac{-1.64\text{E-4}}{-0.49 x + -2.53} + 10.10 ) + (1.55\text{E-2}  + 2.07\text{E-2} x^{2} + 3.40\text{E-3} x) \tan(-4.39 + 4.23 x ) \end{aligned}$ \\
M40 & $C \operatorname{inv}(L ) \cosh(C \operatorname{inv}(L ) |{P}| + C ) + C$ & $0.18 + \frac{0.23 \cosh(\frac{3.10\text{E-4} |{-1.70 x^{3} + 2.27 x^{5} + 3.28 x^{4} + 4.18 + 4.32 x - x^{2}}|}{x} )}{-3.72 x + 4.84}$ \\
M41 & $C \operatorname{Pow}(C \tanh(P ) \operatorname{asinh}(L ) + C ) + C$ & $\begin{aligned}[t] & -1.11\text{E-2} (-2.72 \tanh(-9.93\text{E-3} + 1.46\text{E-2} x^{2} + 6.49\text{E-3} x^{3} + 6.83\text{E-3}  x ) \operatorname{asinh}(-0.88 + 0.54 x ) + 3.89)^{4} + 0.91 \end{aligned}$ \\
M42 & $C \operatorname{inv}(L ) + C + P |{C \log(L ) + C}|$ & $\begin{aligned}[t] & -0.61 + \frac{-1.39\text{E-3}}{4.36 x + 4.92} + (1.14\text{E-3} x + 1.22\text{E-3} x^{3} + 1.63\text{E-4} \\ & \quad + 3.12\text{E-4} x^{4} + 3.12\text{E-4} x^{2}) |{-0.28 + 0.73 \log(-1.21 x + 12.17 )}| \end{aligned}$ \\
M43 & $C \operatorname{square}(C + L \operatorname{inv}(P ) ) \operatorname{atan}(L ) + C$ & $\begin{aligned}[t] & -0.10 + -6.45\text{E-2} (\frac{1.23 x}{-3.88 + 1.12 x^{2} + x} + 3.21)^{2} \operatorname{atan}(-1.69 + x ) \end{aligned}$ \\
M44 & $C \cosh(P ) + C + L \operatorname{inv}(C \operatorname{atan}(L ) + C )$ & $-5.24\text{E-3} \cosh(-0.13 x + 6.54\text{E-2} + 6.65\text{E-2} x^{2} ) + 0.31 + \frac{-9.21\text{E-3} x + 4.49\text{E-3}}{-3.90 \operatorname{atan}(2.82 + x ) + 2.31}$ \\
M45 & $C + L \log(C \operatorname{inv}(L ) + C \sin(P ) + C )$ & $\begin{aligned}[t] & 0.78 + (-0.34 + 7.39\text{E-2} x) \log(\frac{-6.39\text{E-2}}{3.04 + 4.63 x} + 1.18 \sin(0.31 x + 3.79\text{E-2} x^{2} + 7.20\text{E-2} x^{3} + 7.76\text{E-2} ) + 10.46 ) \end{aligned}$ \\
M46 & $C \operatorname{inv}(C \tan(P ) + C |{L}| + C ) + C$ & $-2.07 + \frac{22.43}{-0.35 \tan(-3.26\text{E-2} x^{2} + 2.30\text{E-3} + 7.43\text{E-2} x^{3} + 7.43\text{E-2} x ) + 4.71 + |{-1.02 + 1.62 x}|}$ \\
M47 & $\begin{aligned}[t] & C \operatorname{inv}(L ) + C \operatorname{asin}(P ) \operatorname{atan}(C \operatorname{inv}(L ) +  C ) + C \end{aligned}$ & $\begin{aligned}[t] & 0.16 + \frac{0.54}{-1.56 + -4.55 x} + 1.88 \operatorname{asin}(-0.11 + -4.33\text{E-4} x + 1.00\text{E-3}  x^{3} + 1.06\text{E-3} x^{2} ) \operatorname{atan}(2.48 + \frac{4.69}{-1.70 x + -3.73} ) \end{aligned}$ \\
M48 & $C \operatorname{inv}(L ) \operatorname{atanh}(L ) + C \operatorname{atanh}(P ) + C$ & $-0.18 + 0.68 \operatorname{atanh}(-1 + 1.96\text{E-2} x^{2} + 4.24\text{E-3} x ) + \frac{5.13 \operatorname{atanh}(9.99\text{E-2} x )}{-0.69 + x}$ \\
M49 & $C \operatorname{Pow}(C \operatorname{atanh}(P ) + C ) \operatorname{acosh}(L ) + C$ & $\begin{aligned}[t] & -0.51 + 1.51\text{E-4} (-2.08 \operatorname{atanh}(-1.57\text{E-2} x^{2} + 0.97 + 4.00\text{E-2} x ) + -2.46)^{4} \operatorname{acosh}(0.96 x + 10.74 ) \end{aligned}$ \\
M50 & $C \tanh(C \operatorname{inv}(P ) + C + L \tan(L ) ) + C$ & $\begin{aligned}[t] & -2.38\text{E-2} + 17.92 \tanh(1.62\text{E-4} + \frac{1.63\text{E-4}}{-3.09 x + 1.99 + x^{2}} + (-2.55\text{E-5} + -4.79\text{E-4} x) \tan(-2.77 + 3.97 x ) ) \end{aligned}$ \\
\hline\hline
\end{tabular}%
}
\end{table*}

\begin{table*}[htbp]
\centering
\caption{Generated Equations for the H tier}
\label{tab:eq_h}
\scriptsize
\resizebox{\textwidth}{!}{
\begin{tabular}{cll}
\hline\hline
\textbf{ID} & \textbf{Skeleton} & \textbf{Equation} \\
\hline
H1 & $C \operatorname{inv}{\left(L \right)} \operatorname{Pow}{\left(C \log{\left(L \right)} + C \operatorname{asin}{\left(P \right)} + C \right)} + C$ & $-0.53 + \frac{8.41\text{E-2} \left(-1.16 + -2.64 \operatorname{asin}{\left(-0.97 + 1.52\text{E-2} x^{2} + 4.42\text{E-2} x \right)} + 0.29 \log{\left(1.14 x + 11.49 \right)}\right)^{4}}{-1.02 x + 3.34}$ \\
H2 & $C \operatorname{Pow}{\left(C \tanh{\left(L \right)} + C \right)} \operatorname{atan}{\left(C \operatorname{inv}{\left(L \right)} + C \right)} + P$ & 
$\begin{aligned}[t]-0.53 +& -9.11\text{E-6} x^{5} + 1.19\text{E-5} x + 1.39\text{E-5} x^{4} + 2.99\text{E-6} x^{6} + 2.99\text{E-6} x^{3} + \\ &3.28\text{E-6} x^{2} + 6.97\text{E-6} \left(-4.81 \tanh{\left(0.80 x \right)} + 1.39\right)^{6} \operatorname{atan}{(2.03 + \frac{3.71}{-0.51 x + -2.78} )} \end{aligned}$ \\
H3 & $C \operatorname{inv}{\left(L \right)} \operatorname{Pow}{\left(C \sin{\left(L \right)} + C \right)} + C \operatorname{Pow}{\left(C \operatorname{atan}{\left(P \right)} + C \right)} + C$ & $-0.18 \operatorname{atan}^{6}{\left(-3.57 x + 0.46 + x^{2} \right)} + 1.94 + \frac{5.09 \sin^{\frac{11}{5}}{\left(-1.63 + 2.17 x \right)}}{1.54 + 3.67 x}$ \\
H4 & $C \operatorname{inv}{\left(C + L \log{\left(L \right)} \right)} + C \operatorname{atanh}{\left(C \operatorname{inv}{\left(P \right)} + C \right)} + C$ & $-2.99 \operatorname{atanh}{\left(\frac{0.57}{-2.19 + -4.58 x^{3} + 0.45 x^{2} + 0.91 x + x^{4}} + 7.95\text{E-2} \right)} + 0.20 + \frac{2.01}{2.13 + \left(-1.12 x + 2.93\right) \log{\left(-1.31 x + 13.25 \right)}}$ \\
H5 & $C \sqrt{C \operatorname{inv}{\left(P \right)} \operatorname{atan}{\left(L \right)} + C} + C \operatorname{inv}{\left(L \right)} + C$ & $-8.82 + \frac{1.18}{x} + 2.74 \sqrt{\frac{-3.81 \operatorname{atan}{\left(1.83 x + 4.25 \right)}}{-4.57 x^{2} + 3.04 x + 4.18} + 10.37}$ \\
H6 & $C \operatorname{inv}{\left(C \operatorname{atanh}{\left(L \right)} + C \right)} + C \cosh{\left(C \operatorname{inv}{\left(P \right)} + L \right)} + C$ & $\frac{-1.33\text{E-4}}{-1.92 + 3.47 \operatorname{atanh}{\left(-2.45\text{E-3} + 0.10 x \right)}} + -1.76\text{E-2} \cosh{\left(0.60 x + \frac{1.48}{-0.21 + -1.17 x + x^{2}} \right)} + 0.70$ \\
H7 & $C \operatorname{Pow}{\left(C \operatorname{atan}{\left(L \right)} + C \right)} \operatorname{asin}{\left(C \operatorname{inv}{\left(L \right)} + C \right)} + C$ & $0.27 \left(-1.28 \operatorname{atan}{\left(-4.29 + 0.79 x \right)} + -3.98\right)^{4} \operatorname{asin}{\left(\frac{0.37}{-2.69 + -4.35 x} + 1.63\text{E-3} \right)} + 0.44$ \\
H8 & $C \operatorname{inv}{\left(C \cosh{\left(P \right)} + C \right)} + C \operatorname{Pow}{\left(C + L \log{\left(L \right)} \right)} + C$ & $-6.82\text{E-2} + \frac{0.40}{-1.45 + -3.08 \cosh{\left(-1.51\text{E-3} x + 1.65\text{E-2} x^{2} + 5.71\text{E-3} + 6.36\text{E-3} x^{3} \right)}} + 0.70 \sqrt[6]{-1.38 + \left(0.60 + x\right) \log{\left(1.11 x + 11.22 \right)}}$ \\
H9 & $C \operatorname{Pow}{\left(C \operatorname{inv}{\left(L \right)} + L \right)} + C \log{\left(C \operatorname{atanh}{\left(P \right)} + C \right)} + C$ & $-2.76 + 0.47 \sqrt[3]{-1.28 x + \frac{3.09}{3.32 + x}} + 1.27 \log{\left(-2.27 \operatorname{atanh}{\left(-0.11 + 1.04\text{E-3} x + 1.14\text{E-3} x^{2} + 9.89\text{E-4} x^{3} \right)} + 8.72 \right)}$ \\
H10 & $C \operatorname{inv}{\left(C \tan{\left(L \right)} + P \right)} \operatorname{acosh}{\left(C \operatorname{inv}{\left(L \right)} + C \right)} + C$ & $-0.28 + \frac{2.19 \operatorname{acosh}{\left(\frac{0.79}{x} + 5.53 \right)}}{-1.38 x + -4.48 \tan{\left(3.73 + 4.34 x \right)} + x^{2}}$ \\
H11 & $C + L \operatorname{inv}{\left(C \operatorname{acosh}{\left(P \right)} + C \right)} \tan{\left(C \operatorname{inv}{\left(L \right)} + C \right)}$ & $4.86\text{E-2} + \frac{\left(49.63 + 8.71 x\right) \tan{\left(\frac{-2.57\text{E-2}}{x} + 1.49\text{E-2} \right)}}{-3.19 \operatorname{acosh}{\left(-0.25 x + 1.11 + 2.63 x^{2} \right)} + 4.37}$ \\
H12 & $C \operatorname{inv}{\left(C \cosh{\left(P \right)} + C \right)} \sinh{\left(C \operatorname{inv}{\left(L \right)} + C \right)} + L$ & $-0.16 x + \frac{0.39 \sinh{\left(\frac{0.81}{1.43 x + 4.25} + 5.00\text{E-3} \right)}}{-1.74 \cosh{\left(-6.64\text{E-2} + 7.18\text{E-2} x^{2} + 8.82\text{E-2} x \right)} + 3.37} + 3.97\text{E-3}$ \\
H13 & $C \operatorname{atanh}{\left(C \operatorname{inv}{\left(L \right)} + C \right)} + C + P \operatorname{inv}{\left(C \operatorname{asinh}{\left(L \right)} + C \right)}$ & $-0.23 \operatorname{atanh}{\left(\frac{0.28}{-2.48 x + -3.83} + 3.04\text{E-3} \right)} + 0.32 + \frac{0.14 x^{2} + 0.21 + 2.07\text{E-3} x}{-2.12 \operatorname{asinh}{\left(-0.74 + 2.57 x \right)} + 2.81}$ \\
H14 & $C \operatorname{inv}{\left(P \right)} \tanh{\left(C \operatorname{inv}{\left(L \right)} + C \tanh{\left(L \right)} + C \right)} + C$ & $-7.28\text{E-2} + \frac{149.58 \tanh{\left(-5.35\text{E-3} + \frac{7.14\text{E-3}}{-2.57 + x} + 9.20\text{E-3} \tanh{\left(0.19 + 0.78 x \right)} \right)}}{-0.79 x + -1.78 x^{3} + -4.46 x^{2} + 2.33 x^{4}}$ \\
H15 & $C \operatorname{inv}{\left(C \sin{\left(L \right)} + C \right)} \operatorname{Pow}{\left(C \operatorname{atan}{\left(L \right)} + C \right)} + C$ & $0.23 + \frac{9.95\text{E-6} \left(-3.73 \operatorname{atan}{\left(0.82 + 4.59 x \right)} + 0.32\right)^{\frac{11}{2}}}{-0.25 + 0.19 \sin{\left(1.91 + x \right)}}$ \\
H16 & $C \operatorname{Pow}{\left(C \tanh{\left(P \right)} + C \right)} \left|{C \operatorname{inv}{\left(L \right)} + C}\right| + L$ & $-0.57 + 2.83\text{E-3} x + 2.83\text{E-3} \left(-4.26 + 1.37 \tanh{\left(7.27\text{E-2} x^{2} + 7.27\text{E-2} x \right)}\right)^{4} \left|{1.21 + \frac{3.76}{-2.70 + 1.70 x}}\right|$ \\
H17 & $C \operatorname{Pow}{\left(C + L \sin{\left(P \right)} \right)} + C \operatorname{asin}{\left(C \operatorname{inv}{\left(L \right)} + C \right)} + C$ & 
$\begin{aligned}[t]-6.76\text{E-5} \operatorname{asin}(&\frac{0.36}{-4.24 + 4 x} + 7.56\text{E-4} ) + \\ &1.03\text{E-4} \left(\left(7.00\text{E-2} + x\right) \sin{\left(-2.34\text{E-2} + 7.43\text{E-2} x^{3} + 7.43\text{E-2} x^{2} + 7.43\text{E-2} x \right)}\right)^{\frac{24}{5}} + 8.58\text{E-3}\end{aligned}$ \\
H18 & $C \operatorname{Pow}{\left(C \operatorname{atan}{\left(L \right)} + C \right)} \operatorname{atanh}{\left(C \operatorname{inv}{\left(L \right)} + P \right)} + C$ & $-5.41\text{E-2} + 0.75 \operatorname{atan}^{\frac{5}{3}}{\left(-1.35 + 1.41 x \right)} \operatorname{atanh}{\left(-0.94 + -1.08\text{E-2} x + \frac{1.46\text{E-2}}{0.44 + 1.28 x} + 1.83\text{E-2} x^{2} \right)}$ \\
H19 & $C \operatorname{Pow}{\left(C \operatorname{acosh}{\left(L \right)} + C \right)} + C + L \cosh{\left(C \operatorname{inv}{\left(P \right)} + C \right)}$ & 
$\begin{aligned}[t]-0.11 (-2.82 + 0.84 &\operatorname{acosh}(0.95 x + 10.64 ))^{4} + \\ &0.30 + \left(-9.11\text{E-2} + 7.94\text{E-2} x\right) \cosh{\left(0.66 + \frac{2.07}{-1.01 x + -2.37 x^{2} + 0.92 x^{4} + 1.78 + 2.77 x^{3}} \right)}\end{aligned}$ \\
H20 & $C e^{C \operatorname{inv}{\left(P \right)}} + C + L \operatorname{inv}{\left(C \operatorname{atanh}{\left(L \right)} + C \right)}$ & $-0.32 + \frac{-0.44 x}{-3.56 \operatorname{atanh}{\left(9.99\text{E-2} x \right)} + 0.29} + -0.56 e^{\frac{1.06}{-1.41 + 8.40\text{E-2} x + x^{2}}}$ \\
H21 & $C \operatorname{inv}{\left(C \operatorname{atan}{\left(L \right)} + C \right)} \cosh{\left(C \operatorname{inv}{\left(P \right)} + C \right)} + L$ & $0.12 x + \frac{0.51 \cosh{\left(0.14 + \frac{1}{-0.44 x^{3} + -2.62 x^{4} + -3.24 x^{5} + 1.69 x^{2} + 1.86 + 1.94 x} \right)}}{\operatorname{atan}{\left(-1.94 + 4.15 x \right)}} + 2.48\text{E-2}$ \\
H22 & $C \log{\left(C \operatorname{inv}{\left(P \right)} + C \right)} \operatorname{acosh}{\left(C \operatorname{inv}{\left(L \right)} + C \right)} + L$ & $-0.16 x + -1.06 \log{\left(\frac{2.23}{-2.55 x + -2.68 x^{2} + x^{4} + x^{3}} + 9.62 \right)} \operatorname{acosh}{\left(\frac{0.37}{-2.15 x + 1.05} + 2.53 \right)} + 3.74$ \\
H23 & $C \operatorname{inv}{\left(C \sin{\left(P \right)} + C \right)} \operatorname{Pow}{\left(C \operatorname{atan}{\left(L \right)} + L \right)} + C$ & $-1.74\text{E-2} + \frac{2.30\text{E-4} \left(-2.54 + -3.70 \operatorname{atan}{\left(-1.34 + 3.57 x \right)} + x\right)^{5}}{\sin{\left(-1.45 x^{2} + 0.48 x + 1.31 \right)}}$ \\
H24 & $C \operatorname{inv}{\left(P \right)} + C \operatorname{Pow}{\left(C \tanh{\left(L \right)} + C \right)} \operatorname{atan}{\left(L \right)} + C$ & $\frac{0.57}{0.60 x^{5} + 1.25 x^{3} + 1.41 x + 2.37 + 4.17 x^{2} + x^{4}} + 0.99 + 1.10 \sqrt[6]{0.13 - \tanh{\left(-0.15 + 0.79 x \right)}} \operatorname{atan}{\left(-4.50 + 1.12 x \right)}$ \\
H25 & $C \operatorname{Pow}{\left(C + L \operatorname{asin}{\left(L \right)} \right)} + C \operatorname{asin}{\left(C \operatorname{inv}{\left(P \right)} + C \right)} + C$ & $-1.38 + 0.36 \left(0.33 + \left(0.84 + 3.50\text{E-2} x\right) \operatorname{asin}{\left(9.99\text{E-2} x \right)}\right)^{4} + 1.11 \operatorname{asin}{\left(\frac{-6.05\text{E-2}}{0.18 x + 2.09 x^{2} + 3.71 x^{3}} + 0.77 \right)}$ \\
H26 & $C \sqrt{P} + C \operatorname{Pow}{\left(C \operatorname{inv}{\left(L \right)} \sin{\left(L \right)} + C \right)} + C$ & 
$\begin{aligned}[t]-4.13 + 0.69 (&-1.07 + \frac{3.49 \sin{\left(0.45 + x \right)}}{1.58 + 4.21 x})^{6} + \\ &1.01 \sqrt{-1.86\text{E-5} x + -3.61\text{E-5} x^{3} + 1.09\text{E-4} x^{5} + 10.11 + 6.55\text{E-5} x^{2} + 8.11\text{E-5} x^{4}}\end{aligned}$ \\
H27 & $C \operatorname{inv}{\left(C \operatorname{asinh}{\left(P \right)} + C \right)} + C \operatorname{Pow}{\left(C \left|{L}\right| + C \right)} + C$ & $-0.57 + \frac{0.24}{-1.12 \operatorname{asinh}{\left(-0.68 + 1.98 x + x^{3} + x^{2} \right)} + 4.03} + 0.78 \left(-0.22 \left|{x}\right| + 0.72\right)^{4}$ \\
H28 & $C \operatorname{inv}{\left(L \right)} + C \operatorname{Pow}{\left(C \tan{\left(L \right)} + C \right)} \log{\left(C \operatorname{inv}{\left(P \right)} + C \right)} + C$ & $-0.73 + \frac{-3.96\text{E-3}}{-4.89 x + 1.92} + 8.92\text{E-4} \left(-0.85 \tan{\left(1.53 x + 4.46 \right)} + 3.66\right)^{4} \log{\left(\frac{-2.69}{x^{2} + x} + 13.07 \right)}$ \\
H29 & $C \operatorname{inv}{\left(C + P \left|{L}\right| \right)} \operatorname{Pow}{\left(C \operatorname{atanh}{\left(L \right)} + C \right)} + C$ & $\frac{-1.97 \left(-2.44 + 0.51 \operatorname{atanh}{\left(9.99\text{E-2} x \right)}\right)^{\frac{3}{2}}}{-1.17 + \left(-0.65 x^{2} + -1.66 + 0.41 x^{3} + 1.10 x + x^{4}\right) \left|{-0.25 + 3.69 x}\right|} + 0.23$ \\
H30 & $C \operatorname{inv}{\left(C \left|{P}\right| + C \right)} \operatorname{Pow}{\left(C \operatorname{atan}{\left(L \right)} + C \right)} + L$ & $\frac{-0.43 \operatorname{atan}^{5}{\left(1.67 + 4.11 x \right)}}{\left|{-3.75 x + -4.07 x^{2} + 3.10 + 3.75 x^{3}}\right|} + 0.13 x + 0.15$ \\
H31 & $C \operatorname{Pow}{\left(C + L \operatorname{inv}{\left(P \right)} \right)} + C \tanh{\left(C \sin{\left(L \right)} + C \right)} + C$ & $-1.22 + 2.68\text{E-2} \tanh{\left(0.32 + 1.26 \sin{\left(-4.02 + 1.52 x \right)} \right)} + 6.30\text{E-2} \left(-1.64 + \frac{-1.02 + 2.90 x}{-2.16 x^{2} + -3.83 x + -4.01}\right)^{6}$ \\
H32 & $C \operatorname{inv}{\left(L \right)} \operatorname{asinh}{\left(C \operatorname{inv}{\left(L \right)} + C \right)} + C \log{\left(P \right)} + C$ & $\frac{-0.70 \operatorname{asinh}{\left(0.47 + \frac{1.93}{x} \right)}}{-2.08 + x} + -1.94 + 0.20 \log{\left(-2.83 x^{4} + 0.76 x^{5} + 3.74 x^{6} + 3.96 x^{3} + 4.21 - x^{2} - x \right)}$ \\
H33 & $C \operatorname{inv}{\left(C \operatorname{acosh}{\left(L \right)} + C \right)} + C \operatorname{Pow}{\left(C \operatorname{inv}{\left(L \right)} \operatorname{asin}{\left(P \right)} + C \right)} + C$ & $6.71 \left(-4.65 - \frac{\operatorname{asin}{\left(-1.45\text{E-3} x + -8.71\text{E-2} + 1.01\text{E-3} x^{3} + 8.71\text{E-4} x^{2} \right)}}{-3.34 + 0.99 x}\right)^{\frac{3}{2}} + \frac{6.71}{-4.39 + 0.41 \operatorname{acosh}{\left(-0.58 x + 6.90 \right)}} + 70.35$ \\
H34 & $C \operatorname{Pow}{\left(C + L \operatorname{inv}{\left(P \right)} \right)} \sin{\left(C \sqrt{L} + C \right)} + C$ & $6.53\text{E-3} \left(3.44 + \frac{2.40 x + 2.85}{0.56 x + 2.08 + 2.47 x^{3} + 3.02 x^{2}}\right)^{4} \sin{\left(3.29 \sqrt{1 x + 10.10} \right)} + 8.29\text{E-2}$ \\
H35 & $C \operatorname{asin}{\left(C \operatorname{inv}{\left(P \right)} + C \right)} \operatorname{atan}{\left(C \operatorname{inv}{\left(L \right)} + C \right)} + C$ & $-2.25 + 1.34 \operatorname{asin}{\left(\frac{-0.61}{-0.72 x^{4} + -3.02 x + 1.35 + 2.77 x^{6} + 3.05 x^{5} + x^{3} + x^{2}} + 1 \right)} \operatorname{atan}{\left(\frac{3.55}{-1.43 x + 1.89} + 3.60 \right)}$ \\
H36 & $C \operatorname{inv}{\left(C \sin{\left(L \right)} + C \right)} \operatorname{Pow}{\left(C \operatorname{atan}{\left(P \right)} + C \right)} + L$ & $-1.84\text{E-4} x + \frac{1.23\text{E-2} \left(-0.96 + -1.38 \operatorname{atan}{\left(-2.88 + -2.98 x^{2} + 1.43 x + x^{3} \right)}\right)^{4}}{\sin{\left(3.88 x + 4.80 \right)}} + 3.97\text{E-2}$ \\
H37 & $C \operatorname{square}{\left(C \operatorname{inv}{\left(L \right)} + L \right)} \operatorname{asinh}{\left(C \operatorname{inv}{\left(P \right)} + C \right)} + C$ & $-1.19 + 9.22\text{E-3} \left(-2 x + -2.50 + \frac{3.98}{-0.46 x + -1.94}\right)^{2} \operatorname{asinh}{\left(\frac{-0.46}{-0.19 x^{4} + -3.46 + -3.85 x + 0.65 x^{3} + 3.57 x^{2}} + 0.67 \right)}$ \\
H38 & $C \operatorname{Pow}{\left(C \operatorname{atanh}{\left(L \right)} + C \right)} + C \operatorname{Pow}{\left(C \operatorname{inv}{\left(L \right)} \operatorname{acosh}{\left(P \right)} + C \right)} + C$ & $-0.86 + 1.04\text{E-3} \left(-2.66 \operatorname{atanh}{\left(9.99\text{E-2} x \right)} + 1.27\right)^{4} + 6.79\text{E-4} \left(\frac{-2.50 \operatorname{acosh}{\left(1.50 + 2.01 x + 2.54 x^{2} \right)}}{0.62 + x} + -3.80\right)^{4}$ \\
H39 & $C \operatorname{Pow}{\left(C \operatorname{inv}{\left(L \right)} + C \tan{\left(L \right)} + C \right)} \sinh{\left(P \right)} + C$ & $0.24 + 2.49\text{E-7} \left(-4.41 - \tan{\left(-2.80 + 2.69 x \right)} + \frac{1}{-0.85 + 3.82 x}\right)^{6} \sinh{\left(-7.29\text{E-2} x^{2} + 5.52\text{E-2} + 8.27\text{E-2} x \right)}$ \\
H40 & $C \operatorname{Pow}{\left(C \operatorname{inv}{\left(L \right)} + C \operatorname{atan}{\left(P \right)} + C \right)} + C \tan{\left(L \right)} + C$ & $1.44\text{E-2} \left(-3.12 + \operatorname{atan}{\left(-4.46 x^{2} + 2.05 + x \right)} + \frac{1}{-0.50 + -2.05 x}\right)^{\frac{11}{3}} + 2.50\text{E-2} \tan{\left(-0.74 + 4.82 x \right)} + 3.76$ \\
H41 & $C \operatorname{inv}{\left(L \right)} \operatorname{asin}{\left(C \operatorname{inv}{\left(P \right)} + C \cosh{\left(L \right)} + C \right)} + C$ & $1.30\text{E-3} + \frac{1.35 \operatorname{asin}{\left(-1.33\text{E-3} \cosh{\left(0.80 x \right)} + 0.99 + \frac{1.60\text{E-3}}{-0.38 x + -0.84 x^{3} + 1.73 x^{2}} \right)}}{-1.57 x + 1.83}$ \\
H42 & $C \operatorname{inv}{\left(P \right)} + C \operatorname{Pow}{\left(C \tan{\left(L \right)} + C \right)} \operatorname{atan}{\left(L \right)} + C$ & $\frac{-8.20\text{E-3}}{-4.58 x^{5} + 2.35 x^{3} + 4.75 x + x^{4} + x^{2}} + 1.76\text{E-2} \tan^{6}{\left(4.18 x \right)} \operatorname{atan}{\left(-1.61 + x \right)} + 5.75\text{E-2}$ \\
H43 & $C \operatorname{inv}{\left(L \right)} + C \cosh{\left(C \operatorname{inv}{\left(L \right)} + C \right)} \operatorname{asin}{\left(P \right)} + C$ & $0.28 + \frac{0.76}{0.53 x + 4.91} + 0.89 \cosh{\left(\frac{0.24}{-4.11 + 3.18 x} + 0.28 \right)} \operatorname{asin}{\left(-0.98 + 1.65\text{E-2} x^{2} + 3.32\text{E-2} x \right)}$ \\
H44 & $C \operatorname{Pow}{\left(C + L \operatorname{inv}{\left(P \right)} \right)} + C \log{\left(C \sqrt{L} + C \right)} + C$ & $-5.35 + 0.14 \left(4.24 + \frac{-3.58 + 3.39 x}{-2.48 + -2.52 x + 0.57 x^{3} + 2.62 x^{2}}\right)^{\frac{6}{5}} + 1.86 \log{\left(4.90 \sqrt{0.72 x + 7.32} \right)}$ \\
H45 & $C \operatorname{atan}{\left(C \operatorname{inv}{\left(L \right)} + C \right)} \operatorname{atan}{\left(C \operatorname{inv}{\left(P \right)} + C \right)} + C$ & $-0.46 + 1.94 \operatorname{atan}{\left(0.23 + \frac{3.85}{-1.40\text{E-2} x^{3} + -2.57 x + -3.55 + 0.76 x^{5} + 1.60 x^{4} + x^{2}} \right)} \operatorname{atan}{\left(3.50 - \frac{1}{-1.88 + 4.64 x} \right)}$ \\
H46 & $C \operatorname{inv}{\left(C \operatorname{asin}{\left(L \right)} + C \right)} \operatorname{Pow}{\left(C \left|{L}\right| + C \right)} + P$ & $-0.24 + -3.27\text{E-4} x^{2} + -6.15\text{E-4} x^{3} + 2.61\text{E-4} x + \frac{6.42\text{E-4} \left(-3.10 + \left|{x}\right|\right)^{5}}{-0.81 + 2.01 \operatorname{asin}{\left(9.99\text{E-2} x \right)}}$ \\
H47 & $C \operatorname{inv}{\left(P \right)} \operatorname{Pow}{\left(C \log{\left(L \right)} + C \right)} + C \operatorname{Pow}{\left(C \left|{L}\right| + C \right)} + C$ & $-0.67 + \frac{-1.26\text{E-6} \left(-1.65 + 4.92 \log{\left(-1.21 x + 12.17 \right)}\right)^{5}}{-1.51 x + -3.28 x^{2} + -3.66 x^{3} + -3.87 x^{5} + 3.27 x^{4} + 4.40} + 2.40\text{E-7} \left(-1.35 \left|{x}\right| + -2.43\right)^{6}$ \\
H48 & $C \sqrt{C \operatorname{inv}{\left(P \right)} + C} \operatorname{inv}{\left(C \cosh{\left(L \right)} + C \right)} + L$ & $0.17 x + \frac{0.17 \sqrt{\frac{-4.24}{-0.28 x^{3} + -2.68 x + -4.30 x^{4} + -4.69 x^{6} + 2.46 x^{2} + 3.54 + x^{5}} + 6.68}}{-3.44 + \cosh{\left(-0.77 + 0.72 x \right)}} + 3.49\text{E-2}$ \\
H49 & $C \operatorname{Pow}{\left(C \operatorname{inv}{\left(L \right)} + C \tanh{\left(L \right)} + C \right)} \cosh{\left(P \right)} + C$ & $-4.48\text{E-4} \left(-1.98 \tanh{\left(0.80 x \right)} + \frac{0.51}{0.98 + 1.35 x}\right)^{4} \cosh{\left(-7.11\text{E-3} + 1.31\text{E-2} x^{2} + 1.31\text{E-2} x + 6.57\text{E-3} x^{3} \right)} + 0.26$ \\
H50 & $C \operatorname{inv}{\left(C \tan{\left(L \right)} + C \right)} + C \operatorname{Pow}{\left(C \operatorname{inv}{\left(P \right)} + C \operatorname{atan}{\left(L \right)} + C \right)} + C$ & $-1.90 + \frac{-5.48\text{E-5}}{-2.12 \tan{\left(3.16 + 4.99 x \right)} + 1.98} + 2.63\text{E-5} \left(-0.41 + 4.61 \operatorname{atan}{\left(1.56 x + 3.48 \right)} + \frac{1}{-2.85 x + 1.51 x^{2} + 3.50}\right)^{6}$ \\
\hline\hline
\end{tabular}
}
\end{table*}

\clearpage

\begin{figure*}
    \centering
    \includegraphics[width=\linewidth]{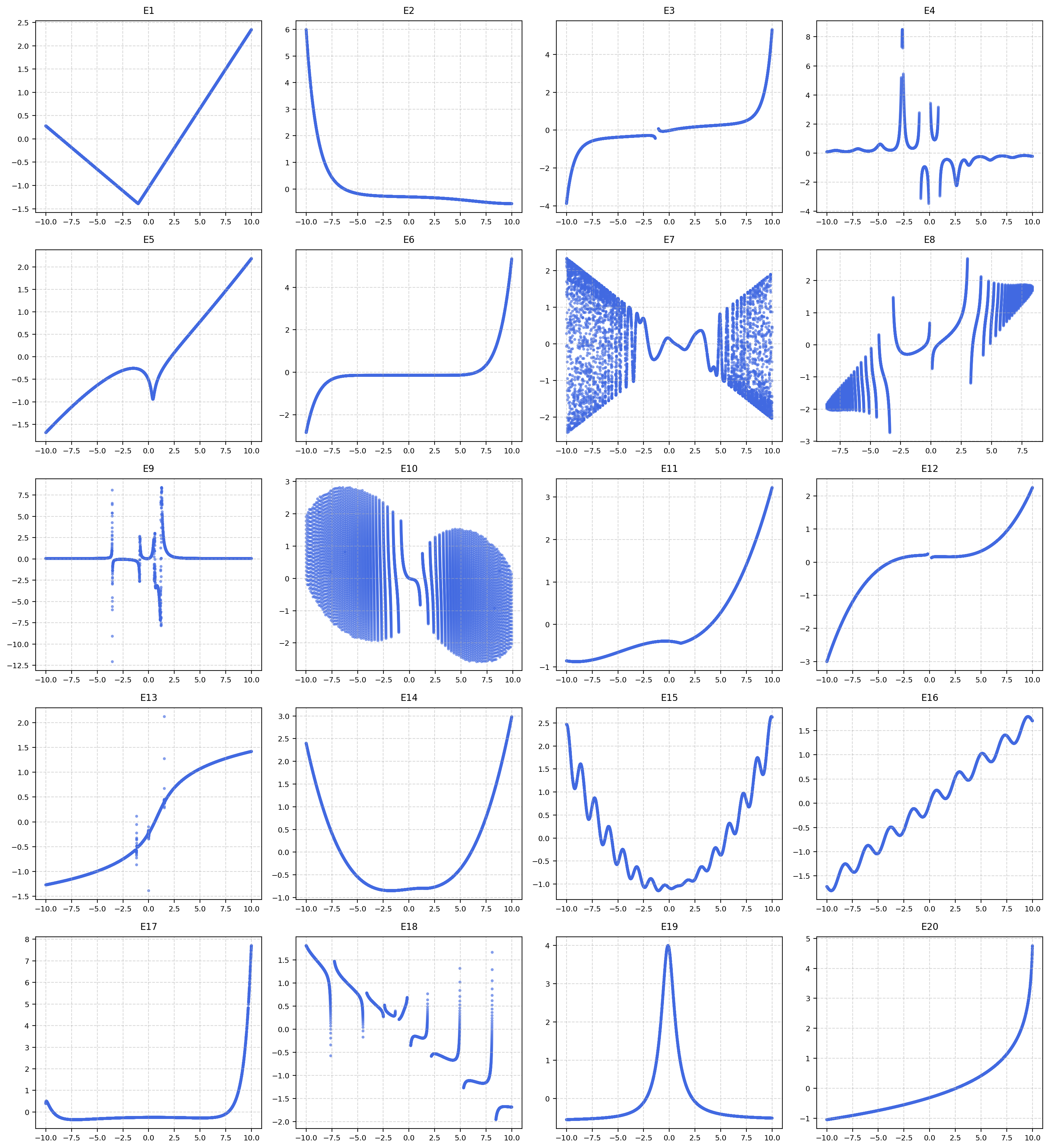}
    \caption{Plots for Benchmark Equations --- E Tier (E1--E20)}
    \label{fig:eq_e1}
\end{figure*}

\begin{figure*}
    \centering
    \includegraphics[width=\linewidth]{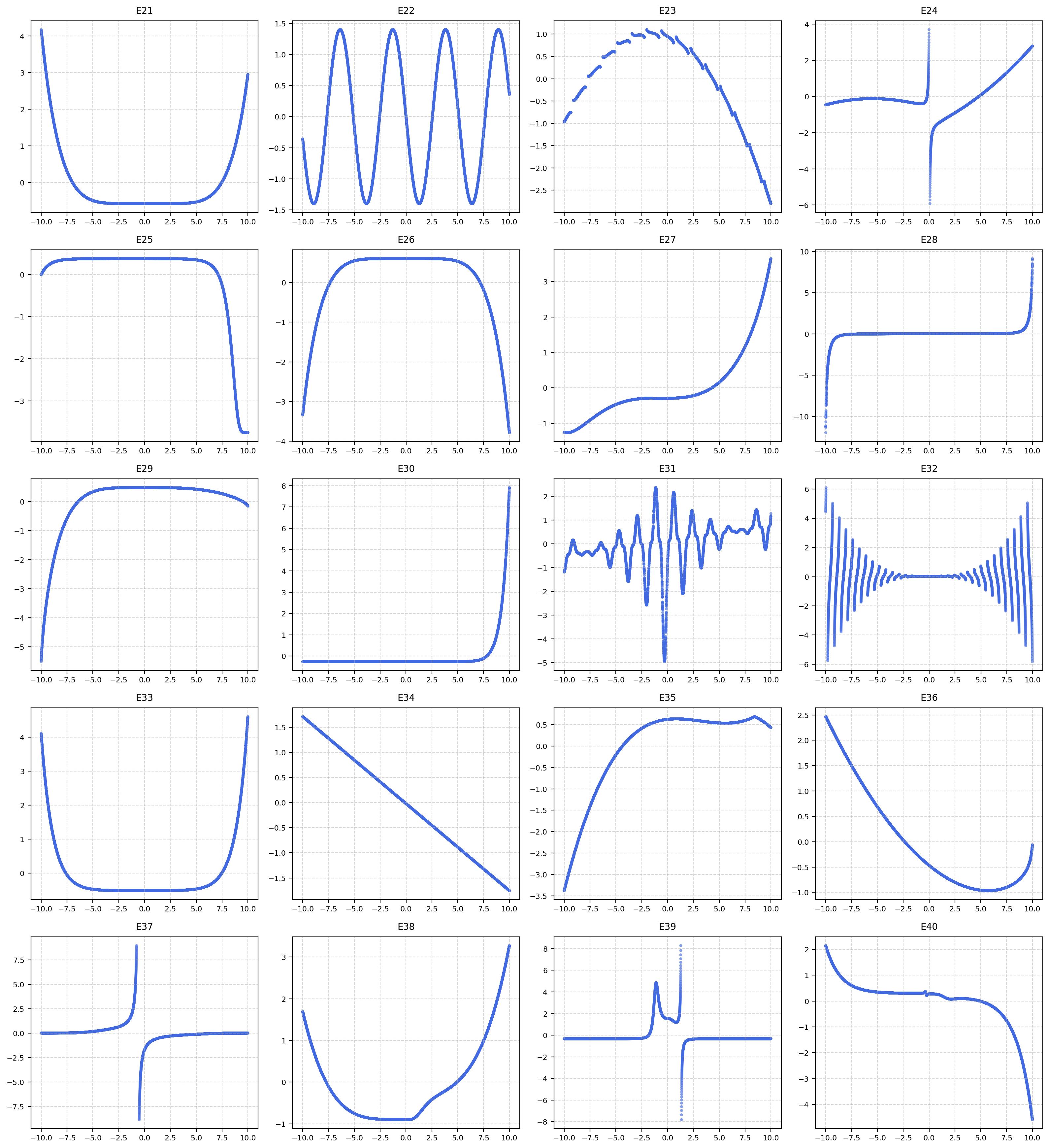}
    \caption{Plots for Benchmark Equations --- E Tier (E21--E40)}
    \label{fig:eq_e2}
\end{figure*}

\begin{figure*}
    \centering
    \includegraphics[width=\linewidth]{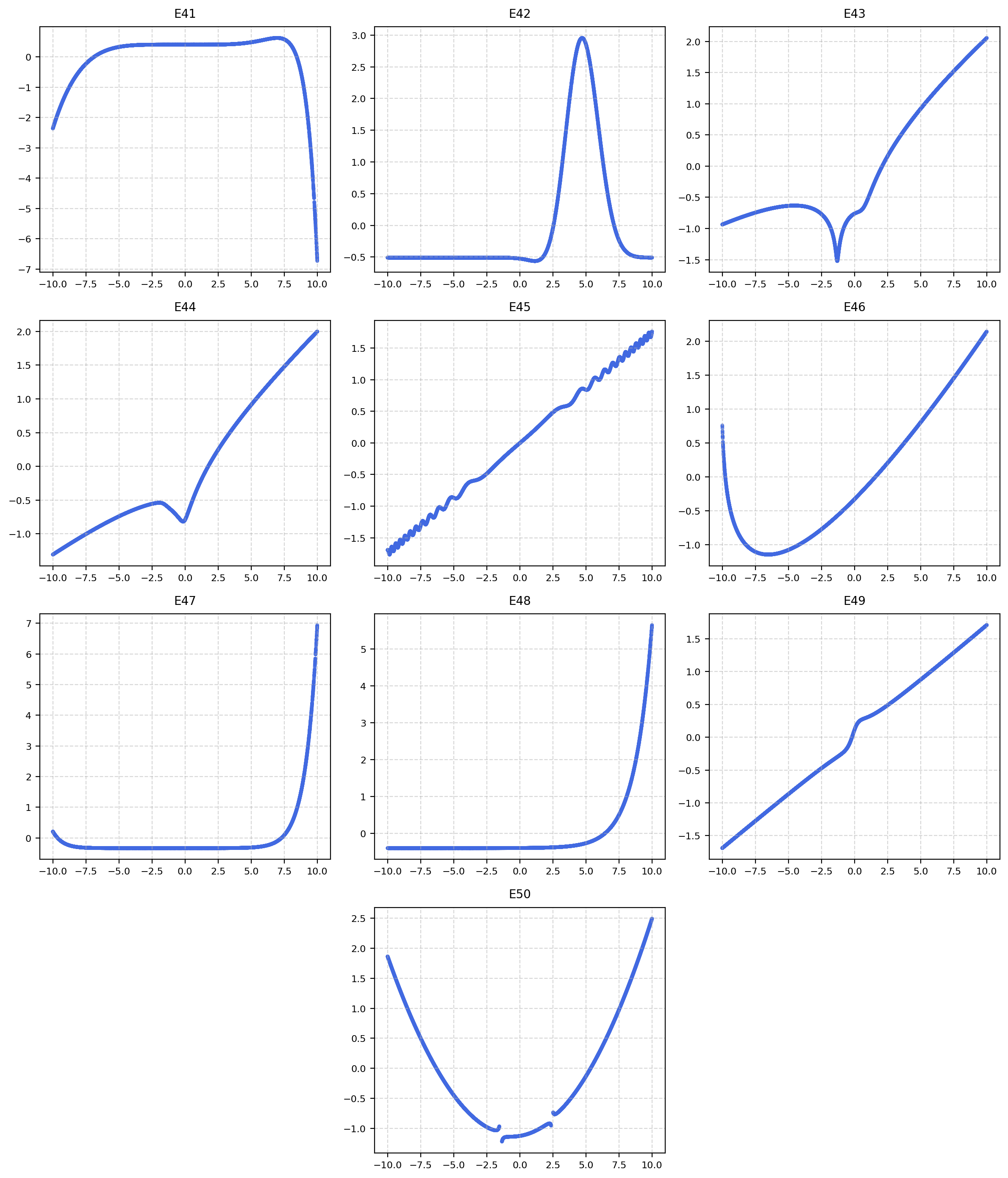}
    \caption{Plots for Benchmark Equations --- E Tier (E41--E50)}
    \label{fig:eq_e3}
\end{figure*}

\begin{figure*}
    \centering
    \includegraphics[width=\linewidth]{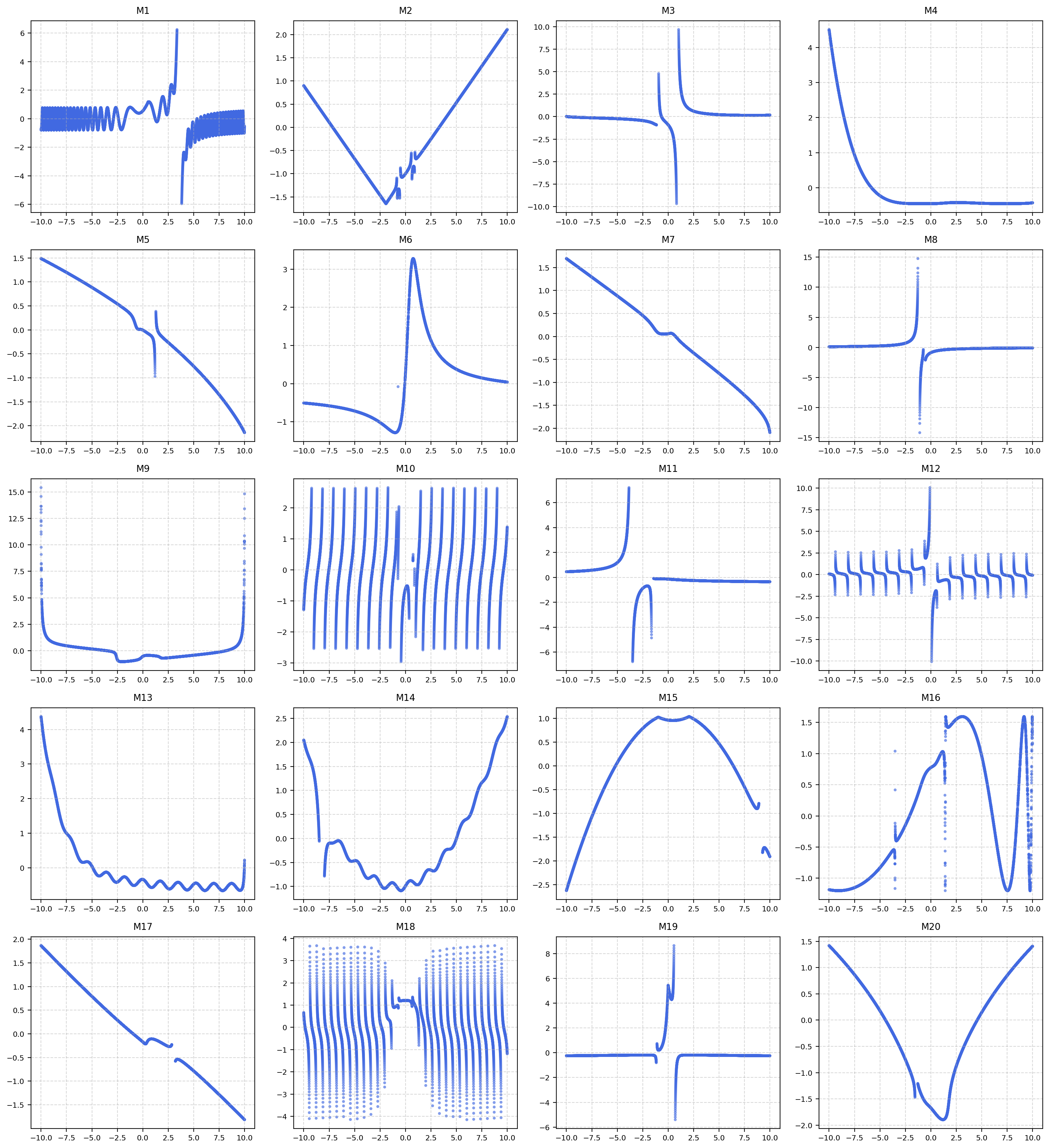}
    \caption{Plots for Benchmark Equations --- M Tier (M1--M20)}
    \label{fig:eq_m1}
\end{figure*}

\begin{figure*}
    \centering
    \includegraphics[width=\linewidth]{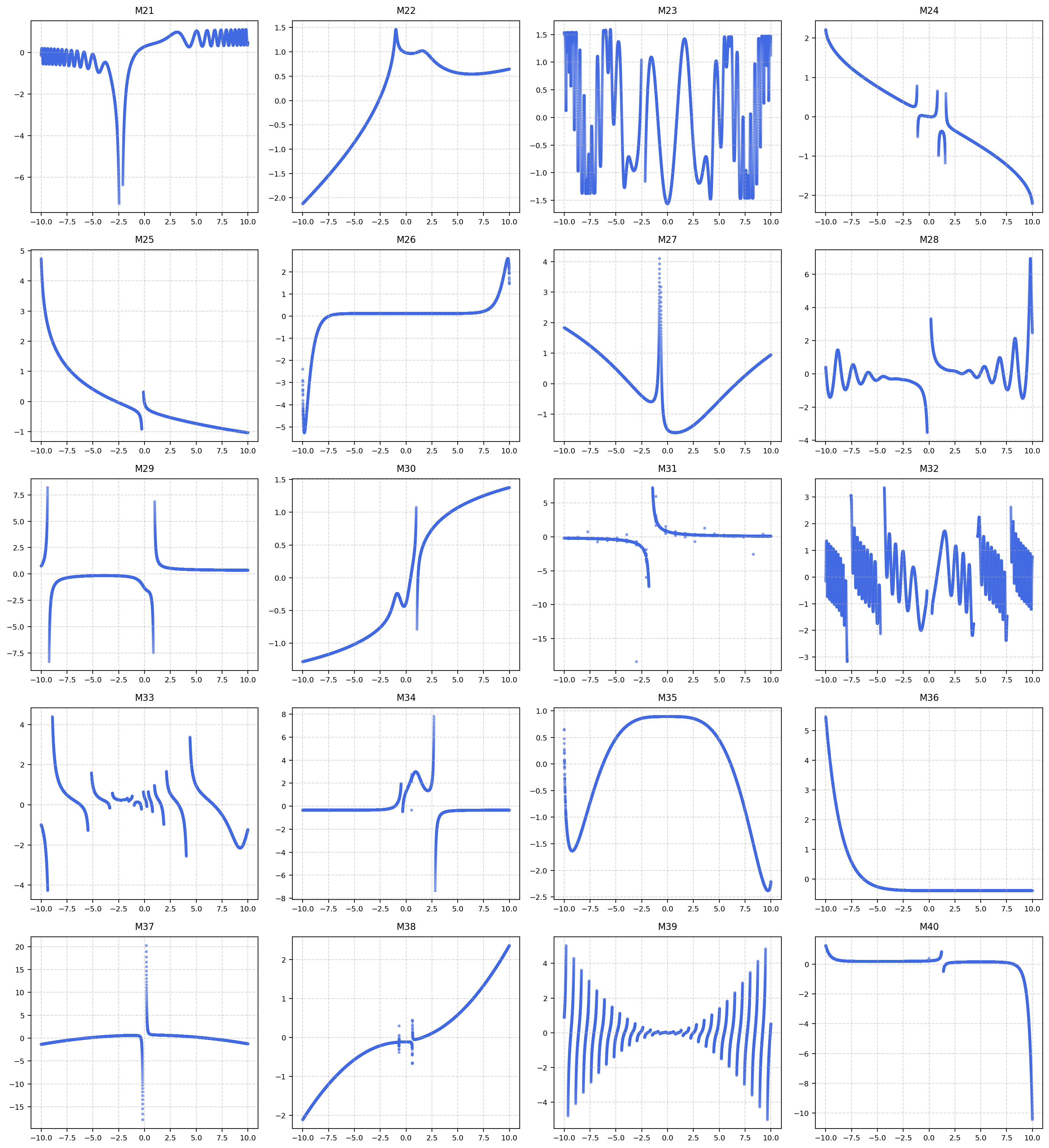}
    \caption{Plots for Benchmark Equations --- M Tier (M21--M40)}
    \label{fig:eq_m2}
\end{figure*}

\begin{figure*}
    \centering
    \includegraphics[width=\linewidth]{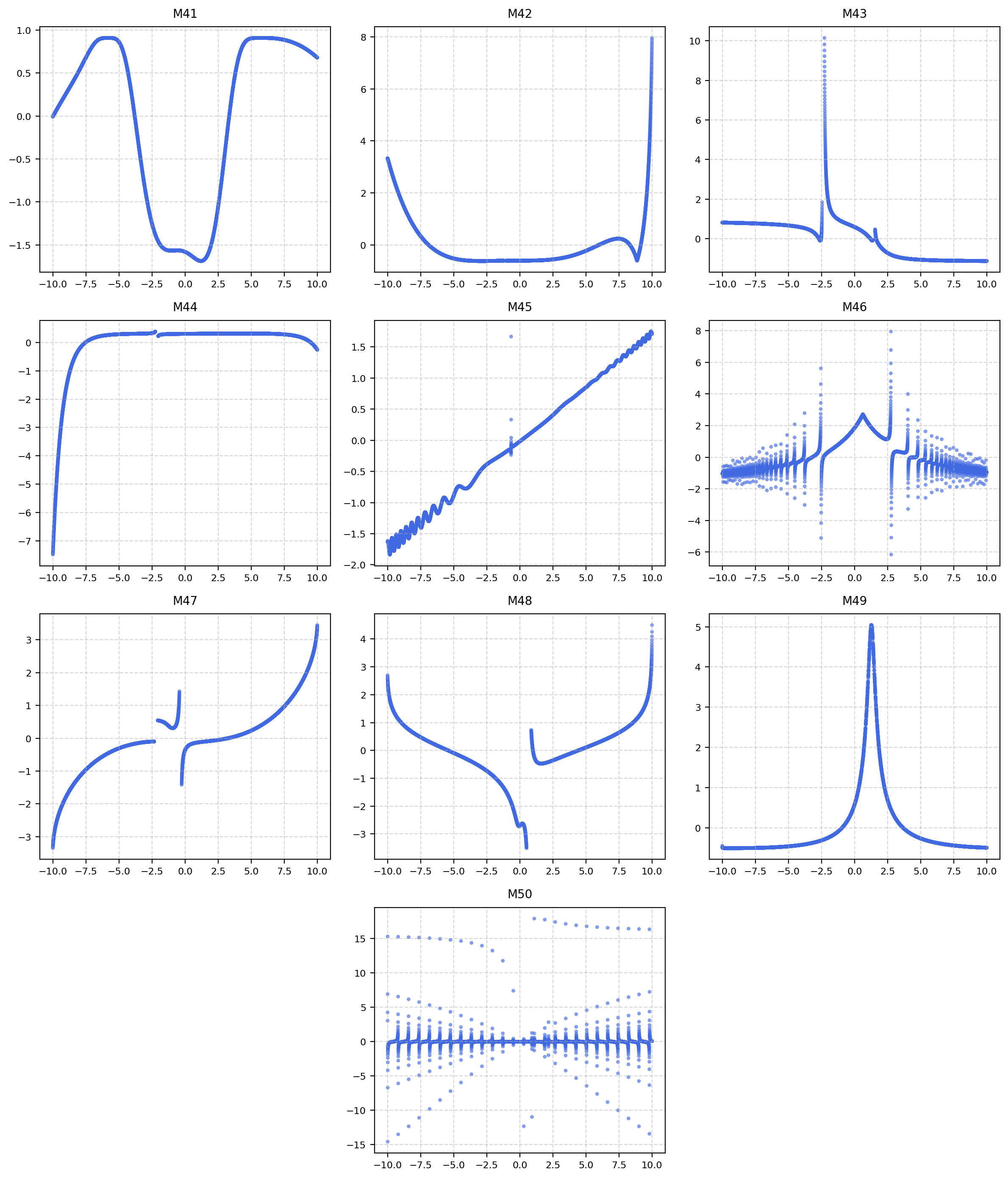}
    \caption{Plots for Benchmark Equations --- M Tier (M41--M50)}
    \label{fig:eq_m3}
\end{figure*}

\begin{figure*}
    \centering
    \includegraphics[width=\linewidth]{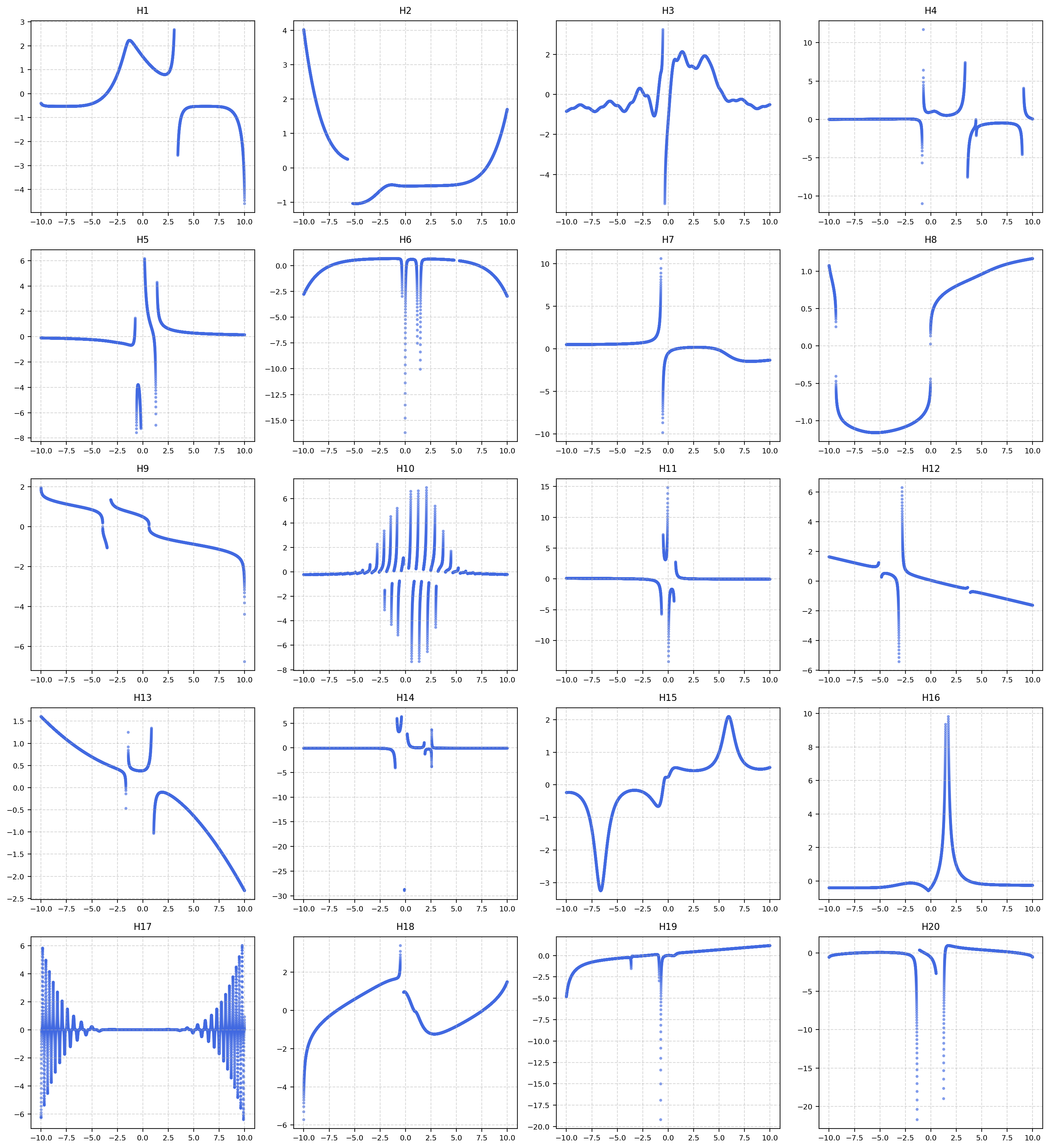}
    \caption{Plots for Benchmark Equations --- H Tier (H1--H20)}
    \label{fig:eq_h1}
\end{figure*}

\begin{figure*}
    \centering
    \includegraphics[width=\linewidth]{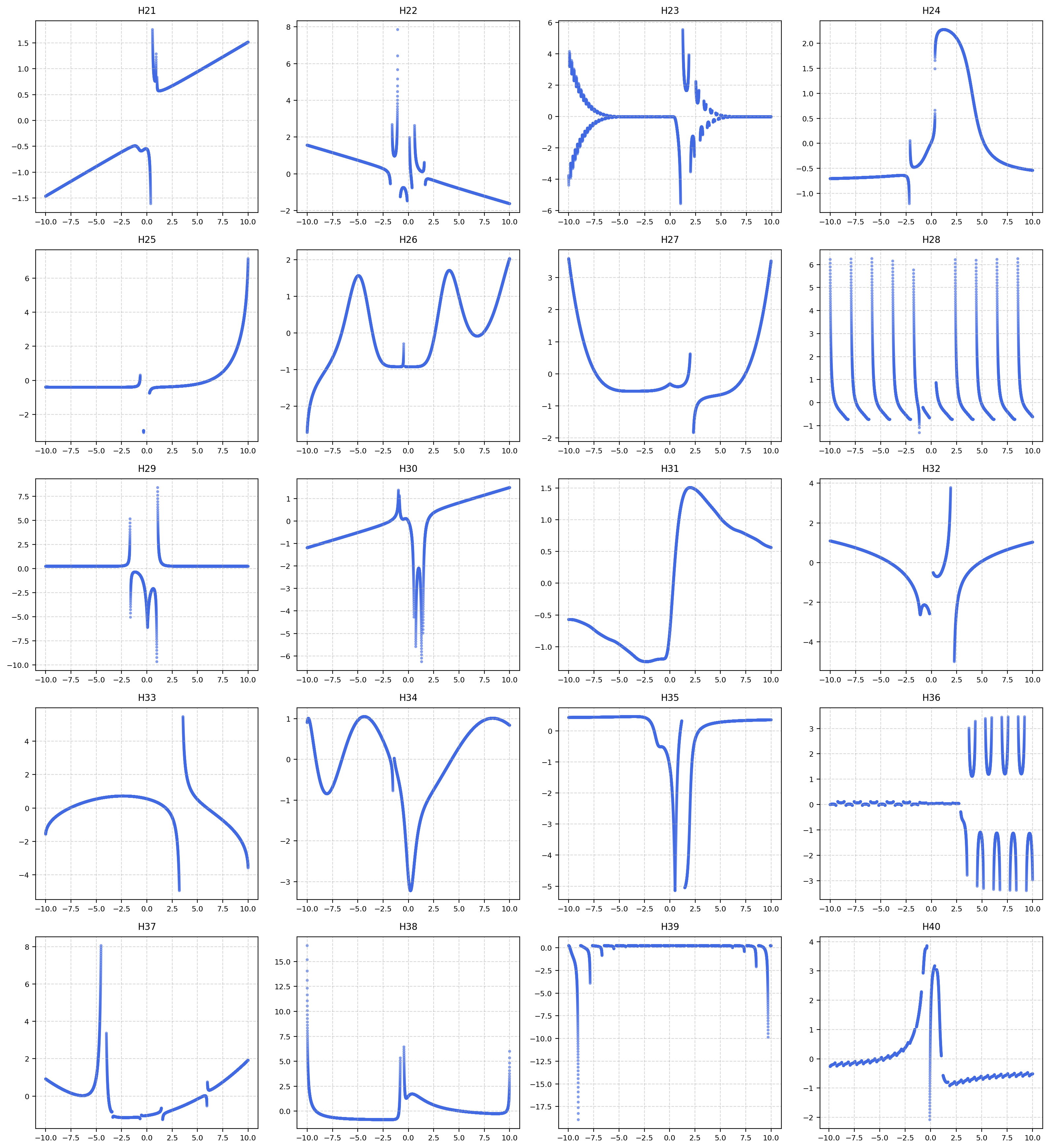}
    \caption{Plots for Benchmark Equations --- H Tier (H21--H40)}
    \label{fig:eq_h2}
\end{figure*}

\begin{figure*}
    \centering
    \includegraphics[width=\linewidth]{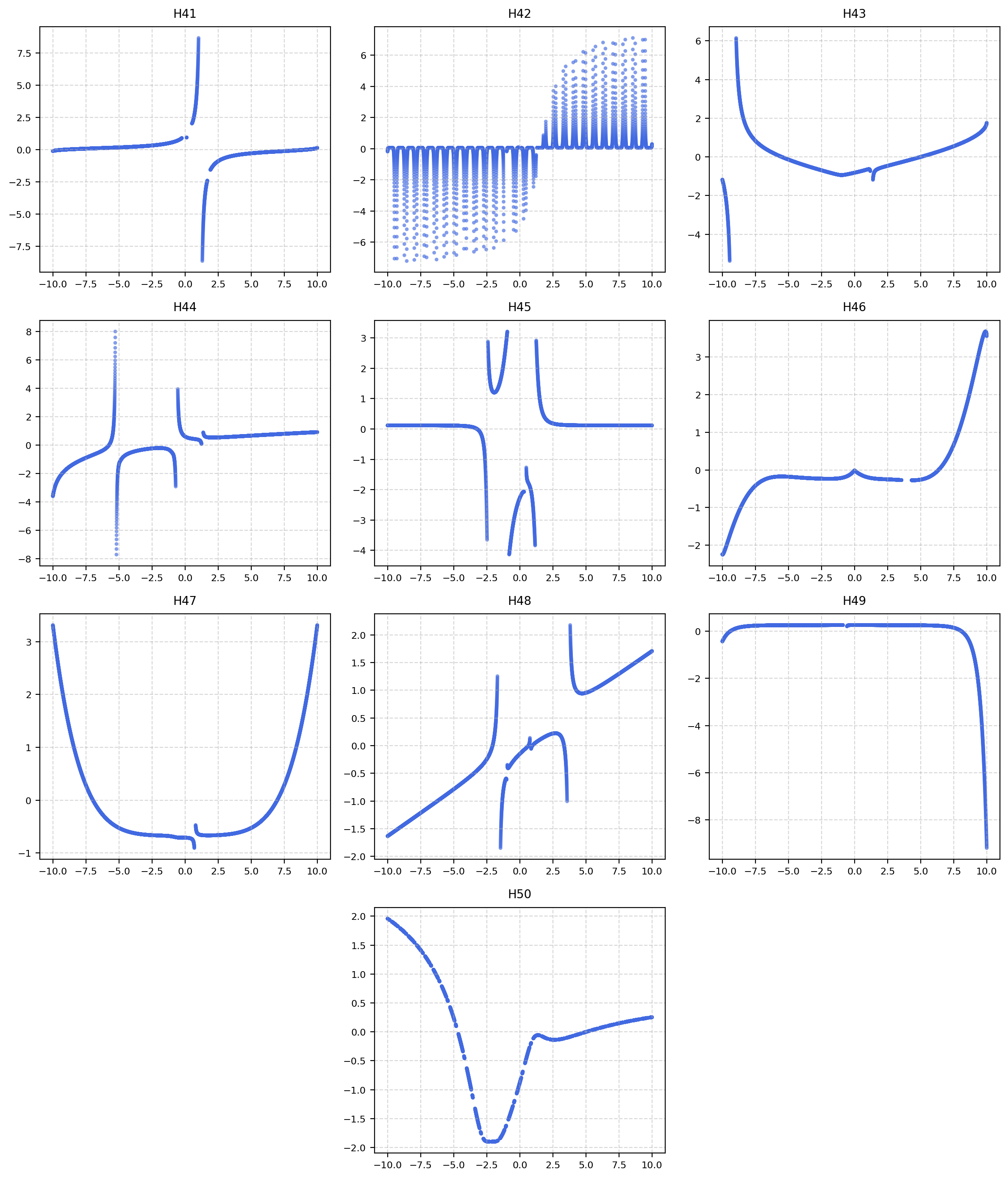}
    \caption{Plots for Benchmark Equations --- H Tier (H41--H50)}
    \label{fig:eq_h3}
\end{figure*}

\FloatBarrier

\stepcounter{section} 
\setcounter{table}{0}  

\begin{table}[!ht]
\centering
\caption{Tuning hyperparameter grids}
\label{tab:hyperparameter_grid}
\scriptsize
\setlength{\tabcolsep}{3pt}
\begin{tabularx}{\columnwidth}{l l X}
\toprule
\textbf{Method} & \textbf{Hyperparameter} & \textbf{Values} \\
\midrule

\multirow{2}{*}{AIFeynman}
& Polynomial fitting degree & 2, 3, 4, 5 \\
& Neural network training epochs & 500, 1000, 2000 \\
\midrule

\multirow{3}{*}{BSR}
& Number of trees & 2, 4, 8 \\
& Number of iterations & 50, 100, 200, 500 \\
& Regularization parameter & -0.1, -1, -2, -5 \\
\midrule

\multirow{3}{*}{DSR}
& Number of samples & 100,000, 500,000, 1,000,000 \\
& Learning rate & 0.0003, 0.0005, 0.001, 0.01 \\
& Risk parameter & 0.05, 0.10, 0.15, 0.5 \\
\midrule

\multirow{3}{*}{E2E}
& Number of bags & 50, 100, 200 \\
& Beam size & 1, 5, 10 \\
& Number of top candidates & 1, 5, 10 \\
\midrule

\multirow{3}{*}{GP-GOMEA}
& Population size & 100, 250, 500, 1000 \\
& Number of generations & 50, 100, 200, 500 \\
& Maximum tree size & 15, 31, 63, 127 \\
\midrule

\multirow{6}{*}{gplearn}
& Population size & 200, 500, 1000, 2000 \\
& Number of generations & 100, 200, 500 \\
& Tournament size & 10, 20, 50 \\
& Crossover probability & 0.6, 0.7, 0.8 \\
& Subtree mutation probability & 0.05, 0.1, 0.2 \\
& Parsimony coefficient & 0.0001, 0.001, 0.01, 0.1 \\
\midrule

\multirow{4}{*}{ITEA}
& Population size & 100, 250, 500, 1000 \\
& Number of generations & 50, 100, 200, 500 \\
& Maximum number of terms & 5, 10, 15, 20 \\
& Maximum exponent & 2, 3, 4 \\
\midrule

\multirow{2}{*}{NeSymReS}
& Beam size & 2, 5, 10 \\
& Number of BFGS restarts & 2, 4, 8, 16 \\
\midrule

\multirow{5}{*}{Operon}
& Population size & 250, 500, 1000, 2000 \\
& Evaluation budget & 50,000, 100,000, 250,000, 500,000 \\
& Crossover probability & 0.6, 0.7, 0.8 \\
& Mutation probability & 0.1, 0.15, 0.2, 0.25 \\
& Maximum tree length & 15, 25, 50 \\
\midrule

\multirow{6}{*}{PySR}
& Population size & 30, 100, 250, 500 \\
& Cycles per iteration & 500, 1000, 2000 \\
& Maximum expression depth & 5, 10, 15 \\
& Parsimony coefficient & 0.0001, 0.001, 0.01 \\
& Number of iterations & 100, 200, 500 \\
\midrule

\multirow{4}{*}{TaylorGP}
& Population size & 100, 200, 500 \\
& Number of generations & 50, 100, 200, 500 \\
& Crossover probability & 0.6, 0.7, 0.8 \\
& Mutation probability & 0.05, 0.1, 0.15, 0.2 \\
\midrule

\multirow{4}{*}{TPSR}
& Regularization strength & 0.01, 0.1, 1.0, 5.0 \\
& Network width & 2, 3, 4 \\
& Rollout length & 1, 3, 6, 9 \\
& Number of beams & 1, 3 \\
\midrule

\multirow{3}{*}{uDSR}
& Number of samples & 100,000, 500,000, 1,000,000 \\
& Learning rate & 0.0003, 0.0005, 0.001, 0.01 \\
& Risk parameter & 0.05, 0.10, 0.15, 0.5 \\

\bottomrule
\end{tabularx}
\end{table}

\addtocounter{section}{-1} 

\section{SR Benchmarking: Tuning and Results}  \label{app:hyperparameter}

This appendix provides the hyperparameter tuning details and per-problem results for the SR benchmarking experiments described in Sec.~\ref{sec:resbench}. 
Table~\ref{tab:hyperparameter_grid} reports the exhaustive hyperparameter grids evaluated for each SR method, while Table~\ref{tab:hyperparameter_selected} lists the corresponding optimal configurations selected through an exhaustive grid search. 
Tables~\ref{tab:easy_mse}--\ref{tab:hard_r2} provide the per-problem mean and standard deviation of the interpolation and extrapolation MSE and $R^2$ results.
Cells labeled with -- indicate cases where the algorithm produced invalid expressions or yielded NaN outputs during evaluation for all 10 repetitions.

\setcounter{table}{1} 
\begin{table}[!ht]
\centering
\caption{Selected optimal hyperparameter configurations}
\label{tab:hyperparameter_selected}
\scriptsize
\begin{tabularx}{\columnwidth}{l l c c c}
\toprule
\textbf{Method} & \textbf{Hyperparameter} & \textbf{Easy} & \textbf{Medium} & \textbf{Hard} \\
\midrule

\multirow{2}{*}{AIFeynman}
& Polynomial fitting degree & 3 & 2 & 3 \\
& Neural network training epochs & 2000 & 2000 & 1000 \\
\midrule

\multirow{3}{*}{BSR}
& Number of trees & 2 & 2 & 2\\
& Number of iterations & 50 & 200 & 100\\
& Regularization parameter & -0.1 & -0.1& -2\\
\midrule

\multirow{3}{*}{DSR}
& Number of samples & 1000000 & 1000000 & 100000\\
& Learning rate & 0.0003& 0.0003 & 0.0003\\
\bottomrule
\end{tabularx}
\end{table}

\begin{table*}[htbp]
\centering
\caption{Easy Tier — MSE Performance}
\label{tab:easy_mse}
\resizebox{\textwidth}{!}{%
\setlength{\tabcolsep}{3pt}
\begin{tabular}{@{}ccccccccccccccc@{}}
\toprule
\textbf{ID} & \textbf{Set} & \textbf{AIFeynman} & \textbf{BSR} & \textbf{DSR} & \textbf{E2E} & \textbf{GPGOMEA} & \textbf{Gplearn} & \textbf{ITEA} & \textbf{NeSymReS} & \textbf{Operon} & \textbf{PySR} & \textbf{TaylorGP} & \textbf{TPSR} & \textbf{UDSR} \\
\midrule
\multirow{2}{*}{E1} & Int. & 1.72 $\pm$ 0 & 0.846 $\pm$ 0 & 0.089 $\pm$ 0 & 3.4e-5 $\pm$ 0 & 2.2e-7 $\pm$ 1.9e-7 & 0.317 $\pm$ 0.374 & 1.6e-8 $\pm$ 7.6e-9 & 0.831 $\pm$ 9.8e-3 & 8.0e-4 $\pm$ 9.4e-4 & 1.1e-12 $\pm$ 2.8e-12 & 0.846 $\pm$ 0 & 0.846 $\pm$ 0 & 0.089 $\pm$ 0 \\
 & Ext. & 2.07 $\pm$ 0 & 2.38 $\pm$ 0 & 0.604 $\pm$ 0 & 2.0e-5 $\pm$ 0 & 7.3e-7 $\pm$ 7.5e-7 & 1.25 $\pm$ 0.814 & 5.6e-7 $\pm$ 1.1e-6 & 2.71 $\pm$ 0.024 & 3.0e-3 $\pm$ 4.3e-3 & 0 $\pm$ 1.5e-12 & 2.38 $\pm$ 0 & 2.38 $\pm$ 0 & 0.604 $\pm$ 0 \\
\cline{2-15}
\multirow{2}{*}{E2} & Int. & 1.36 $\pm$ 0 & 0.223 $\pm$ 0 & 0.115 $\pm$ 0 & 3.3e-4 $\pm$ 0 & 0.022 $\pm$ 0.067 & 0.164 $\pm$ 0.060 & 1.8e-10 $\pm$ 7.5e-11 & 0.207 $\pm$ 0.016 & 6.8e-5 $\pm$ 5.1e-5 & 3.5e-5 $\pm$ 1.3e-5 & 0.092 $\pm$ 0 & 0.223 $\pm$ 0 & 0.115 $\pm$ 0 \\
 & Ext. & 5.88 $\pm$ 0 & 8.00 $\pm$ 0 & 5.93 $\pm$ 0 & 0.087 $\pm$ 0 & 12.94 $\pm$ 36.21 & 7.96 $\pm$ 0.881 & 2.7e-5 $\pm$ 1.5e-6 & 8.31 $\pm$ 0.318 & 0.824 $\pm$ 2.31 & 0.090 $\pm$ 0.065 & 6.41 $\pm$ 0 & 8.00 $\pm$ 0 & 5.93 $\pm$ 0 \\
\cline{2-15}
\multirow{2}{*}{E3} & Int. & 1.09 $\pm$ 0 & 0.216 $\pm$ 0 & 0.020 $\pm$ 0 & 3.3e-3 $\pm$ 2.3e-4 & 9.9e-4 $\pm$ 1.1e-3 & 0.164 $\pm$ 0.090 & 2.6e-5 $\pm$ 5.7e-6 & 0.214 $\pm$ 5.9e-4 & 2.3e-3 $\pm$ 1.4e-3 & 0.022 $\pm$ 0.065 & 0.021 $\pm$ 0 & 0.216 $\pm$ 0 & 0.020 $\pm$ 0 \\
 & Ext. & 8.28 $\pm$ 0 & 8.06 $\pm$ 0 & 4.44 $\pm$ 0 & 1.69 $\pm$ 0.171 & 1.1e4 $\pm$ 1.8e4 & 8.06 $\pm$ 0 & 0.020 $\pm$ 7.7e-3 & 8.08 $\pm$ 8.5e-3 & 4.6e6 $\pm$ 1.2e7 & 3.3e4 $\pm$ 9.1e4 & 4.42 $\pm$ 0 & 8.06 $\pm$ 0 & 4.44 $\pm$ 0 \\
\bottomrule
\end{tabular}
}
\end{table*}

\begin{table*}[htbp]
\centering
\caption{Medium Tier — MSE Performance}
\label{tab:medium_mse}
\resizebox{\textwidth}{!}{%
\setlength{\tabcolsep}{3pt}
\begin{tabular}{@{}ccccccccccccccc@{}}
\toprule
\textbf{ID} & \textbf{Set} & \textbf{AIFeynman} & \textbf{BSR} & \textbf{DSR} & \textbf{E2E} & \textbf{GPGOMEA} & \textbf{Gplearn} & \textbf{ITEA} & \textbf{NeSymReS} & \textbf{Operon} & \textbf{PySR} & \textbf{TaylorGP} & \textbf{TPSR} & \textbf{UDSR} \\
\midrule
\multirow{2}{*}{M1} & Int. & 1.02 $\pm$ 0 & 177.3 $\pm$ 325.1 & 0.381 $\pm$ 0 & 0.181 $\pm$ 0.076 & 0.355 $\pm$ 0.240 & --- & 1.18 $\pm$ 0.329 & 1.07 $\pm$ 1.4e-3 & 0.382 $\pm$ 0.131 & 0.268 $\pm$ 0.087 & --- & 1.07 $\pm$ 0 & 0.381 $\pm$ 0 \\
 & Ext. & 0.351 $\pm$ 0 & 380.6 $\pm$ 770.2 & 0.321 $\pm$ 0 & 0.428 $\pm$ 0.265 & 0.317 $\pm$ 0.012 & 0.395 $\pm$ 0.075 & 360.8 $\pm$ 1025.7 & 0.365 $\pm$ 0.042 & 0.370 $\pm$ 0.044 & 0.286 $\pm$ 0.094 & --- & 0.351 $\pm$ 0 & 0.321 $\pm$ 0 \\
\cline{2-15}
\multirow{2}{*}{M2} & Int. & 9.2e-4 $\pm$ 0 & 0.916 $\pm$ 0.119 & 0.148 $\pm$ 0 & 8.8e-4 $\pm$ 1.4e-4 & 7.7e-4 $\pm$ 2.8e-4 & 0.281 $\pm$ 0.355 & 4.1e-4 $\pm$ 2.2e-5 & 1.11 $\pm$ 0.404 & 8.2e-4 $\pm$ 1.1e-4 & 2.7e-4 $\pm$ 8.8e-5 & --- & 0.161 $\pm$ 0 & 0.148 $\pm$ 0 \\
 & Ext. & 3.7e-5 $\pm$ 0 & 1.96 $\pm$ 0.572 & 0.588 $\pm$ 0 & 6.7e-4 $\pm$ 4.7e-4 & 2.5e-5 $\pm$ 2.3e-5 & 1.03 $\pm$ 1.15 & 1.3e-4 $\pm$ 1.7e-4 & 2.72 $\pm$ 0.171 & 2.8e-5 $\pm$ 2.1e-5 & 3.0e-5 $\pm$ 1.1e-5 & --- & 0.254 $\pm$ 0 & 0.588 $\pm$ 0 \\
\cline{2-15}
\multirow{2}{*}{M3} & Int. & 1.15 $\pm$ 0 & 2.1e4 $\pm$ 4.7e4 & 0.101 $\pm$ 0 & 0.024 $\pm$ 0.025 & 0.037 $\pm$ 0.035 & 0.079 $\pm$ 7.2e-5 & 7.4e-5 $\pm$ 1.9e-5 & 1.11 $\pm$ 1.7e-3 & 0.094 $\pm$ 0.070 & 5.0e-4 $\pm$ 2.1e-4 & 1.11 $\pm$ 0 & 0.923 $\pm$ 0 & 0.101 $\pm$ 0 \\
 & Ext. & 0.024 $\pm$ 0 & 16.77 $\pm$ 33.93 & 3.9e-3 $\pm$ 0 & 3.4e-3 $\pm$ 7.7e-3 & 7.6e-3 $\pm$ 5.2e-3 & 2.9e-3 $\pm$ 2.7e-4 & 4.9e-4 $\pm$ 1.4e-4 & 0.013 $\pm$ 6.4e-4 & 0.013 $\pm$ 7.7e-3 & 5.3e-3 $\pm$ 1.1e-3 & 0.013 $\pm$ 0 & 1.8e-3 $\pm$ 0 & 3.9e-3 $\pm$ 0 \\
\bottomrule
\end{tabular}
}
\end{table*}

\begin{table*}[htbp]
\centering
\caption{Hard Tier — MSE Performance}
\label{tab:hard_mse}
\resizebox{\textwidth}{!}{%
\setlength{\tabcolsep}{3pt}
\begin{tabular}{@{}ccccccccccccccc@{}}
\toprule
\textbf{ID} & \textbf{Set} & \textbf{AIFeynman} & \textbf{BSR} & \textbf{DSR} & \textbf{E2E} & \textbf{GPGOMEA} & \textbf{Gplearn} & \textbf{ITEA} & \textbf{NeSymReS} & \textbf{Operon} & \textbf{PySR} & \textbf{TaylorGP} & \textbf{TPSR} & \textbf{UDSR} \\
\midrule
\multirow{2}{*}{H1} & Int. & 1.63 $\pm$ 0 & 80.60 $\pm$ 183.1 & 0.308 $\pm$ 0 & 0.089 $\pm$ 0.024 & 0.024 $\pm$ 0.029 & 0.147 $\pm$ 0.060 & 0.064 $\pm$ 1.5e-3 & 0.897 $\pm$ 0.016 & 0.065 $\pm$ 0.025 & 0.100 $\pm$ 0.270 & --- & 0.910 $\pm$ 0 & 0.308 $\pm$ 0 \\
 & Ext. & 4.14 $\pm$ 0 & 15.47 $\pm$ 27.12 & 2.45 $\pm$ 0 & 1.13 $\pm$ 0.576 & 0.924 $\pm$ 0.081 & 1.28 $\pm$ 0.808 & 24.42 $\pm$ 19.54 & 2.13 $\pm$ 0.120 & 0.934 $\pm$ 0.032 & 0.981 $\pm$ 0.284 & --- & 1.81 $\pm$ 0 & 2.45 $\pm$ 0 \\
\cline{2-15}
\multirow{2}{*}{H2} & Int. & 2.02 $\pm$ 0 & 1.3e4 $\pm$ 2.9e4 & 0.217 $\pm$ 0 & 0.171 $\pm$ 0.215 & 0.014 $\pm$ 0.015 & 0.254 $\pm$ 0.181 & 5.1e-3 $\pm$ 1.1e-4 & 0.448 $\pm$ 2.9e-3 & 0.016 $\pm$ 9.2e-3 & 0.056 $\pm$ 0.147 & 0.389 $\pm$ 0 & 1.85 $\pm$ 0 & 0.217 $\pm$ 0 \\
 & Ext. & 6.44 $\pm$ 0 & 1.3e5 $\pm$ 2.8e5 & 2.86 $\pm$ 0 & 1.83 $\pm$ 2.41 & 0.235 $\pm$ 0.246 & 5.31 $\pm$ 1.47 & 23.70 $\pm$ 42.29 & 6.55 $\pm$ 0.040 & 1.06 $\pm$ 2.68 & 0.841 $\pm$ 1.57 & 5.86 $\pm$ 0 & 18.42 $\pm$ 0 & 2.86 $\pm$ 0 \\
\cline{2-15}
\multirow{2}{*}{H3} & Int. & 0.451 $\pm$ 0 & 9.0e+12 $\pm$ 2.2e+13 & 0.505 $\pm$ 0 & 0.771 $\pm$ 0.304 & 0.104 $\pm$ 0.025 & 0.599 $\pm$ 0.333 & 0.045 $\pm$ 5.6e-3 & 1.06 $\pm$ 7.5e-4 & 0.206 $\pm$ 0.053 & 0.094 $\pm$ 9.8e-3 & --- & 1.06 $\pm$ 0 & 0.505 $\pm$ 0 \\
 & Ext. & 0.015 $\pm$ 0 & 4.2e+15 $\pm$ 1.0e+16 & 0.644 $\pm$ 0 & 0.456 $\pm$ 0.090 & 0.022 $\pm$ 0.020 & 0.298 $\pm$ 0.111 & 29.23 $\pm$ 46.72 & 0.547 $\pm$ 0.016 & 0.056 $\pm$ 0.075 & 0.017 $\pm$ 0.012 & --- & 0.441 $\pm$ 0 & 0.644 $\pm$ 0 \\
\bottomrule
\end{tabular}
}
\end{table*}

\begin{table*}[htbp]
\centering
\caption{Easy Tier — $R^2$ Performance}
\label{tab:easy_r2}
\resizebox{\textwidth}{!}{%
\setlength{\tabcolsep}{3pt}
\begin{tabular}{@{}ccccccccccccccc@{}}
\toprule
\textbf{ID} & \textbf{Set} & \textbf{AIFeynman} & \textbf{BSR} & \textbf{DSR} & \textbf{E2E} & \textbf{GPGOMEA} & \textbf{Gplearn} & \textbf{ITEA} & \textbf{NeSymReS} & \textbf{Operon} & \textbf{PySR} & \textbf{TaylorGP} & \textbf{TPSR} & \textbf{UDSR} \\
\midrule
\multirow{2}{*}{E1} & Int. & -1.07 $\pm$ 0 & -0.021 $\pm$ 0 & 0.892 $\pm$ 0 & 1.000 $\pm$ 0 & 1.000 $\pm$ 2.3e-7 & 0.617 $\pm$ 0.451 & 1.000 $\pm$ 9.2e-9 & -2.4e-3 $\pm$ 0.012 & 0.999 $\pm$ 1.1e-3 & 1.000 $\pm$ 3.4e-12 & -0.021 $\pm$ 0 & -0.021 $\pm$ 0 & 0.892 $\pm$ 0 \\
 & Ext. & -1.09 $\pm$ 0 & -1.41 $\pm$ 0 & 0.388 $\pm$ 0 & 1.000 $\pm$ 0 & 1.000 $\pm$ 7.6e-7 & -0.263 $\pm$ 0.825 & 1.000 $\pm$ 1.1e-6 & -1.74 $\pm$ 0.024 & 0.997 $\pm$ 4.4e-3 & 1.000 $\pm$ 1.5e-12 & -1.41 $\pm$ 0 & -1.41 $\pm$ 0 & 0.388 $\pm$ 0 \\
\cline{2-15}
\multirow{2}{*}{E2} & Int. & -6.15 $\pm$ 0 & -0.173 $\pm$ 0 & 0.396 $\pm$ 0 & 0.998 $\pm$ 0 & 0.883 $\pm$ 0.352 & 0.134 $\pm$ 0.316 & 1.000 $\pm$ 3.9e-10 & -0.088 $\pm$ 0.085 & 1.000 $\pm$ 2.7e-4 & 1.000 $\pm$ 6.6e-5 & 0.518 $\pm$ 0 & -0.173 $\pm$ 0 & 0.396 $\pm$ 0 \\
 & Ext. & -0.100 $\pm$ 0 & -0.496 $\pm$ 0 & -0.109 $\pm$ 0 & 0.984 $\pm$ 0 & -1.42 $\pm$ 6.78 & -0.490 $\pm$ 0.165 & 1.000 $\pm$ 2.9e-7 & -0.556 $\pm$ 0.059 & 0.846 $\pm$ 0.433 & 0.983 $\pm$ 0.012 & -0.200 $\pm$ 0 & -0.496 $\pm$ 0 & -0.109 $\pm$ 0 \\
\cline{2-15}
\multirow{2}{*}{E3} & Int. & -4.09 $\pm$ 0 & -6.2e-3 $\pm$ 0 & 0.905 $\pm$ 0 & 0.985 $\pm$ 1.1e-3 & 0.995 $\pm$ 5.3e-3 & 0.237 $\pm$ 0.421 & 1.000 $\pm$ 2.7e-5 & -4.2e-4 $\pm$ 2.7e-3 & 0.989 $\pm$ 6.7e-3 & 0.897 $\pm$ 0.301 & 0.903 $\pm$ 0 & -6.2e-3 $\pm$ 0 & 0.905 $\pm$ 0 \\
 & Ext. & -0.041 $\pm$ 0 & -0.014 $\pm$ 0 & 0.442 $\pm$ 0 & 0.787 $\pm$ 0.021 & -1442.9 $\pm$ 2247.9 & -0.014 $\pm$ 0 & 0.997 $\pm$ 9.7e-4 & -0.016 $\pm$ 1.1e-3 & -5.8e5 $\pm$ 1.5e6 & -4093.2 $\pm$ 1.1e4 & 0.444 $\pm$ 0 & -0.014 $\pm$ 0 & 0.442 $\pm$ 0 \\
\bottomrule
\end{tabular}
}
\end{table*}

\begin{table*}[htbp]
\centering
\caption{Medium Tier — $R^2$ Performance}
\label{tab:medium_r2}
\resizebox{\textwidth}{!}{%
\setlength{\tabcolsep}{3pt}
\begin{tabular}{@{}ccccccccccccccc@{}}
\toprule
\textbf{ID} & \textbf{Set} & \textbf{AIFeynman} & \textbf{BSR} & \textbf{DSR} & \textbf{E2E} & \textbf{GPGOMEA} & \textbf{Gplearn} & \textbf{ITEA} & \textbf{NeSymReS} & \textbf{Operon} & \textbf{PySR} & \textbf{TaylorGP} & \textbf{TPSR} & \textbf{UDSR} \\
\midrule
\multirow{2}{*}{M1} & Int. & 0.048 $\pm$ 0 & -164.4 $\pm$ 303.3 & 0.645 $\pm$ 0 & 0.831 $\pm$ 0.071 & 0.669 $\pm$ 0.224 & --- & -0.103 $\pm$ 0.307 & -1.4e-3 $\pm$ 1.3e-3 & 0.644 $\pm$ 0.123 & 0.750 $\pm$ 0.081 & --- & -2.6e-4 $\pm$ 0 & 0.645 $\pm$ 0 \\
 & Ext. & -0.070 $\pm$ 0 & -1159.8 $\pm$ 2349.3 & 0.022 $\pm$ 0 & -0.307 $\pm$ 0.809 & 0.035 $\pm$ 0.037 & -0.206 $\pm$ 0.230 & -1099.6 $\pm$ 3128.8 & -0.114 $\pm$ 0.128 & -0.128 $\pm$ 0.135 & 0.128 $\pm$ 0.287 & --- & -0.070 $\pm$ 0 & 0.022 $\pm$ 0 \\
\cline{2-15}
\multirow{2}{*}{M2} & Int. & 0.999 $\pm$ 0 & -0.084 $\pm$ 0.141 & 0.824 $\pm$ 0 & 0.999 $\pm$ 1.7e-4 & 0.999 $\pm$ 3.3e-4 & 0.667 $\pm$ 0.420 & 1.000 $\pm$ 2.6e-5 & -0.312 $\pm$ 0.478 & 0.999 $\pm$ 1.3e-4 & 1.000 $\pm$ 1.0e-4 & --- & 0.809 $\pm$ 0 & 0.824 $\pm$ 0 \\
 & Ext. & 1.000 $\pm$ 0 & -4.23 $\pm$ 1.53 & -0.572 $\pm$ 0 & 0.998 $\pm$ 1.3e-3 & 1.000 $\pm$ 6.2e-5 & -1.76 $\pm$ 3.07 & 1.000 $\pm$ 4.6e-4 & -6.28 $\pm$ 0.457 & 1.000 $\pm$ 5.6e-5 & 1.000 $\pm$ 3.0e-5 & --- & 0.320 $\pm$ 0 & -0.572 $\pm$ 0 \\
\cline{2-15}
\multirow{2}{*}{M3} & Int. & -0.032 $\pm$ 0 & -1.9e4 $\pm$ 4.3e4 & 0.909 $\pm$ 0 & 0.978 $\pm$ 0.023 & 0.967 $\pm$ 0.032 & 0.929 $\pm$ 6.5e-5 & 1.000 $\pm$ 1.7e-5 & -4.8e-4 $\pm$ 1.5e-3 & 0.915 $\pm$ 0.063 & 1.000 $\pm$ 1.9e-4 & -3.8e-5 $\pm$ 0 & 0.168 $\pm$ 0 & 0.909 $\pm$ 0 \\
 & Ext. & -1.65 $\pm$ 0 & -1833.9 $\pm$ 3711.7 & 0.577 $\pm$ 0 & 0.630 $\pm$ 0.841 & 0.172 $\pm$ 0.565 & 0.686 $\pm$ 0.029 & 0.946 $\pm$ 0.015 & -0.424 $\pm$ 0.070 & -0.395 $\pm$ 0.845 & 0.418 $\pm$ 0.121 & -0.378 $\pm$ 0 & 0.805 $\pm$ 0 & 0.577 $\pm$ 0 \\
\bottomrule
\end{tabular}
}
\end{table*}

\begin{table*}[htbp]
\centering
\caption{Hard Tier — $R^2$ Performance}
\label{tab:hard_r2}
\resizebox{\textwidth}{!}{%
\setlength{\tabcolsep}{3pt}
\begin{tabular}{@{}ccccccccccccccc@{}}
\toprule
\textbf{ID} & \textbf{Set} & \textbf{AIFeynman} & \textbf{BSR} & \textbf{DSR} & \textbf{E2E} & \textbf{GPGOMEA} & \textbf{Gplearn} & \textbf{ITEA} & \textbf{NeSymReS} & \textbf{Operon} & \textbf{PySR} & \textbf{TaylorGP} & \textbf{TPSR} & \textbf{UDSR} \\
\midrule
\multirow{2}{*}{H1} & Int. & -0.821 $\pm$ 0 & -89.03 $\pm$ 204.6 & 0.656 $\pm$ 0 & 0.900 $\pm$ 0.027 & 0.973 $\pm$ 0.032 & 0.836 $\pm$ 0.067 & 0.929 $\pm$ 1.7e-3 & -2.0e-3 $\pm$ 0.018 & 0.927 $\pm$ 0.028 & 0.889 $\pm$ 0.302 & --- & -0.016 $\pm$ 0 & 0.656 $\pm$ 0 \\
 & Ext. & -5.45 $\pm$ 0 & -23.13 $\pm$ 42.29 & -2.82 $\pm$ 0 & -0.762 $\pm$ 0.897 & -0.442 $\pm$ 0.126 & -1.00 $\pm$ 1.26 & -37.08 $\pm$ 30.47 & -2.32 $\pm$ 0.188 & -0.456 $\pm$ 0.049 & -0.530 $\pm$ 0.443 & --- & -1.83 $\pm$ 0 & -2.82 $\pm$ 0 \\
\cline{2-15}
\multirow{2}{*}{H2} & Int. & -3.55 $\pm$ 0 & -3.0e4 $\pm$ 6.6e4 & 0.511 $\pm$ 0 & 0.616 $\pm$ 0.483 & 0.969 $\pm$ 0.034 & 0.428 $\pm$ 0.407 & 0.988 $\pm$ 2.5e-4 & -8.0e-3 $\pm$ 6.5e-3 & 0.965 $\pm$ 0.021 & 0.873 $\pm$ 0.332 & 0.126 $\pm$ 0 & -3.17 $\pm$ 0 & 0.511 $\pm$ 0 \\
 & Ext. & -4.77 $\pm$ 0 & -1.1e5 $\pm$ 2.5e5 & -1.56 $\pm$ 0 & -0.638 $\pm$ 2.16 & 0.789 $\pm$ 0.221 & -3.76 $\pm$ 1.32 & -20.24 $\pm$ 37.90 & -4.87 $\pm$ 0.035 & 0.054 $\pm$ 2.40 & 0.246 $\pm$ 1.41 & -4.25 $\pm$ 0 & -15.51 $\pm$ 0 & -1.56 $\pm$ 0 \\
\cline{2-15}
\multirow{2}{*}{H3} & Int. & 0.574 $\pm$ 0 & -8.5e+12 $\pm$ 2.1e+13 & 0.522 $\pm$ 0 & 0.270 $\pm$ 0.287 & 0.902 $\pm$ 0.024 & 0.433 $\pm$ 0.315 & 0.957 $\pm$ 5.3e-3 & -9.3e-4 $\pm$ 7.1e-4 & 0.805 $\pm$ 0.050 & 0.911 $\pm$ 9.3e-3 & --- & -5.1e-3 $\pm$ 0 & 0.522 $\pm$ 0 \\
 & Ext. & -1.13 $\pm$ 0 & -5.8e+17 $\pm$ 1.4e+18 & -88.33 $\pm$ 0 & -62.28 $\pm$ 12.43 & -1.98 $\pm$ 2.74 & -40.39 $\pm$ 15.41 & -4054.1 $\pm$ 6482.3 & -74.88 $\pm$ 2.25 & -6.83 $\pm$ 10.42 & -1.34 $\pm$ 1.73 & --- & -60.24 $\pm$ 0 & -88.33 $\pm$ 0 \\
\bottomrule
\end{tabular}
}
\end{table*}


\end{document}